%% file: main.tex
\documentclass[twoside]{article}
\usepackage[preprint]{aistats2027}

\usepackage[T1]{fontenc}
\usepackage{amsmath}
\usepackage{amssymb}
\usepackage{amsthm}
\usepackage{dsfont}
\usepackage{microtype}
\usepackage{xspace}
\usepackage[dvipsnames]{xcolor}
\usepackage{graphicx}
\usepackage{adjustbox}
\usepackage{booktabs}
\usepackage{array}
\usepackage{longtable}
\usepackage[colorlinks,linktoc=all]{hyperref}
\hypersetup{
  citecolor=MidnightBlue,
  linkcolor=MidnightBlue,
  urlcolor=MidnightBlue
}
\usepackage[noabbrev,capitalize]{cleveref}
\creflabelformat{equation}{#2\textup{#1}#3}
\usepackage{mathtools}
\usepackage{siunitx}
\usepackage[round]{natbib}

\usepackage{algorithm}
\usepackage{algpseudocode}

\crefname{assumption}{Assumption}{Assumptions}
\crefname{remark}{Remark}{Remarks}

\DeclareRobustCommand{\prstep}[1]{\text{\color{white!40!black}#1}}
\DeclareMathOperator{\KL}{KL}
\DeclareMathOperator{\Law}{Law}
\newcommand{\KLdiv}[2]{\KL\left(#1\,\|\,#2\right)}
\newcommand{\E}{\mathds{E}}
\newcommand{\Ep}[2]{\E_{#1}\left[#2\right]}
\newcommand{\Ex}[1]{\E\left[#1\right]}
\DeclareMathOperator{\Var}{var}
\newcommand{\Vrp}[2]{\Var_{#1}\left[#2\right]}
\DeclareMathOperator{\Cov}{Cov}
\newcommand{\sd}[1]{{\color{black!80}\scriptstyle#1}}
\newcommand{\Norm}{\mathcal{N}}
\DeclareMathOperator{\supp}{supp}
\newcommand{\veps}{\varepsilon}
\newcommand{\dd}{\mathop{}\!\mathrm{d}}
\newcommand{\T}{^\top}
\newcommand{\inv}{^{-1}}
\newcommand{\ind}{\mathds{1}}
\newcommand{\Id}{\mathbf{I}}
\newcommand{\defeq}{\triangleq}
\newcommand{\eg}{e.g.\@\xspace}
\newcommand{\calA}{\mathcal{A}}
\newcommand{\calE}{\mathcal{E}}
\newcommand{\calF}{\mathcal{F}}
\newcommand{\calG}{\mathcal{G}}
\newcommand{\calH}{\mathcal{H}}
\newcommand{\calI}{\mathcal{I}}
\newcommand{\calO}{\mathcal{O}}
\newcommand{\calP}{\mathcal{P}}
\newcommand{\calR}{\mathcal{R}}
\newcommand{\calS}{\mathcal{S}}
\newcommand{\calT}{\mathcal{T}}
\newcommand{\calX}{\mathcal{X}}
\newcommand{\dsE}{\mathds{E}}
\newcommand{\dsF}{\mathds{F}}
\newcommand{\dsP}{\mathds{P}}
\newcommand{\dsR}{\mathds{R}}

\newcommand{\W}[1]{w_{#1}}
\newcommand{\Wbar}[1]{\bar{w}_{#1}}
\newcommand{\dbo}{\calT^{\pi}}
\newcommand{\heat}[1]{\mathsf{H}_{#1}}
\newcommand{\sm}[2]{s_{#1,#2}}
\newcommand{\cvar}[1]{\mathrm{CVaR}_{#1}}
\newcommand{\tv}{\mathrm{TV}}
\newcommand{\stopgrad}[1]{\mathrm{sg}\!\left[#1\right]}

\theoremstyle{plain}
\newtheorem{theorem}{Theorem}
\newtheorem{proposition}{Proposition}
\newtheorem{lemma}{Lemma}
\newtheorem{corollary}{Corollary}
\theoremstyle{definition}
\newtheorem{definition}{Definition}
\newtheorem{assumption}{Assumption}
\theoremstyle{remark}
\newtheorem{remark}{Remark}

\title{Stein Variational Distributional Reinforcement Learning: Score-Guided Particles for Multivariate Return Distributions}

\begin{document}
\runningtitle{Stein Variational Distributional Reinforcement Learning}
\runningauthor{Author 1, Author 2, Author 3}

\twocolumn[
\aistatstitle{Stein Variational Distributional Reinforcement Learning: Score-Guided Particles for Multivariate Return Distributions}
\aistatsauthor{Manuel Haußmann \And Melih Kandemir}
\aistatsaddress{University of Southern Denmark}
]

\begin{abstract}
Decisions that weigh several reward signals against each other require the joint distribution of the return vector. Categorical critics scale exponentially in the reward dimension, and quantiles have no canonical multivariate analogue. Multivariate distributional critics therefore represent each return distribution by $M$ equal-weight particles trained with a discrepancy loss.
Beyond their atoms, the particles provide no sampler and cannot resolve events rarer than $1/M$. We propose Stein variational distributional RL (SVDRL). Its critic outputs Stein particles, which Stein variational gradient descent transports along a learned score field of the distributional Bellman target. We learn this field by bootstrapped denoising score matching across noise levels and tie it across the Bellman recursion by a consistency loss, which we prove is a non-expansive temporal-difference scheme in score space. The same field certifies the particles. Its reverse-diffusion sampler satisfies a supremal Wasserstein policy-evaluation bound that holds for bounded, possibly multimodal returns. The Stein particles inherit this bound up to a gap that we measure without ground truth, and the sampler answers tail-risk queries below the particle resolution at any sample size. Empirically, SVDRL is more accurate on rare events than particle critics with the same number of particles, and on par on MuJoCo continuous control.
\end{abstract}

\input{paper}

\bibliography{references}
\bibliographystyle{apalike}

\clearpage
\appendix
\thispagestyle{empty}
\onecolumn
\input{appendix}

\end{document}

%% file: paper.tex
\section{Introduction}\label{sec:intro}

Distributional reinforcement learning models the distribution of the return rather than only its expectation \citep{bellemare2017distributional}. With multivariate rewards, risk- and correlation-sensitive decisions depend on the joint distribution of the return components \citep{zhang2021distributional,hayes2022practical}. The univariate representations do not carry over: categorical grids \citep{bellemare2017distributional} need $\calO(M^d)$ atoms, and quantiles \citep{dabney2018implicit,dabney2018distributional} have no canonical analogue for $d>1$. Existing multivariate methods therefore represent the return distribution by $M$ particles trained with a discrepancy loss, \eg, MMD \citep{zhang2021distributional}, Sinkhorn divergences \citep{sun2024distributional}, or sliced divergences \citep{debes2026multivariate}.

The theory of these methods covers the Bellman operator, which contracts in the respective supremal divergence, but not the learned particles. We show that this gap is structural.
Equal-weight particles carry no density and no sampler, and $M$ of them cannot represent an event of probability below $1/M$, whatever loss trains them (\cref{prop:resolution,fig:teaser}).

We instead learn the \emph{return score field}, the smoothed score of the return distribution at all noise levels. Smoothing commutes with the Bellman update up to a rescaling of the noise level (\cref{lem:commute}), so the field satisfies its own Bellman equation, and we learn it along a noise schedule \citep{song2019generative} by bootstrapped denoising score matching (\cref{lem:dsmU}) and by temporal differences in score space (\cref{thm:scoreTDinf}). A single noise level does not suffice, as it cannot certify mode masses (\cref{rem:laddernecessity}). We sample from the field by reverse diffusion. For fast action selection we additionally keep $M$ particles per state--action pair, the Stein particles, which we transport along the field by Stein variational gradient descent (SVGD) \citep{liu2016stein} and amortize into a critic.

The training risk controls the sampled return distribution in total variation and $\W1$ for bounded multimodal returns, where log-concavity and related assumptions fail (\cref{thm:gen}). This yields a policy-evaluation bound in $\Wbar1$ (\cref{thm:e2e}), certified mode masses (\cref{cor:modemass,cor:modemasse2e}), and a CVaR estimate whose error does not depend on $M$ (\cref{thm:cvarrev}).

No single critic is accurate on every return law, and we show which error limits each family.
Equal-weight particles suffer from a \emph{resolution floor}: they round event probabilities and mixture weights to multiples of $1/M$ (\cref{prop:resolution}).
A score field suffers from \emph{smoothing bias} and is hard to learn where the return law is atomic or its modes lie far apart relative to their widths (\cref{prop:smoothbias,lem:scorelip}).
SVDRL keeps both objects, and a \emph{hybrid rule} reports, per state--action pair and without ground truth, the Stein particles on separated or atomic clusters and the sampler draws elsewhere (\cref{def:hybrid,tab:capabilities}).

SVDRL is, to our knowledge, the first certified multivariate score-based critic for possibly multimodal returns, and its Stein particles inherit the certificate up to a gap measured without ground truth (\cref{thm:e2e}).
At equal particle count, SVDRL is more accurate than the best particle critic in every rare-event comparison (\cref{sec:experiments}).
On all eight synthetic MDPs the hybrid wins within-mode shape against the competitors as specified, although at matched compute they win the unimodal ones.
The competitors' own benchmarks place the boundary at atomic laws, and on MuJoCo SAC-SVDRL is on par in return and has the most accurate critic distribution on three of four.
Its stratified Bellman target set and running-mean moment matching improve every particle critic we tested.

\begin{figure}[t]
  \centering
  \adjustbox{width=0.88\linewidth}{\includegraphics{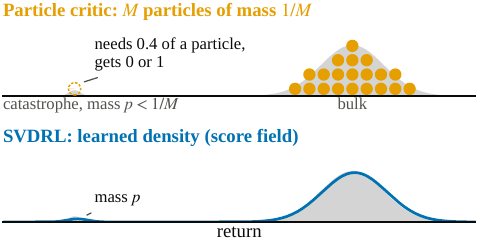}}
  \caption{\emph{A rare event below the particle resolution.} Schematic. A catastrophe of mass $p<1/M$ needs a fraction of a particle, so particle critics round its probability to a multiple of $1/M$ (\cref{prop:resolution}). SVDRL learns a density, which carries any mass; on the rare-event chain it estimates $p=0.005$ far better than particle critics with $M=128$ (\cref{sec:experiments,fig:rarelaw}).}
  \label{fig:teaser}
\end{figure}

\section{Background and Notation}\label{sec:background}

\paragraph{Markov Decision Processes with Multivariate Rewards.}
We consider a Markov Decision Process (MDP) defined by $(\calX,\calA,P,\calR,\gamma)$, where $\calX$ and $\calA$ denote state and action spaces, $\gamma \in (0,1)$ is a discount factor, and ${P:\calX\times\calA\to\calP(\calX)}$ is a transition kernel.
Unlike standard RL, where the reward is a scalar, we define the reward as a potentially random vector ${R(x,a) \sim \calR(\cdot|x,a)\in \calP(\dsR^d)}$. We define the number of  reward components by $d\geq 1$ and recover the classical univariate distributional RL setting for $d=1$.

Given a policy $\pi:\calX\to \calP(\calA)$, the \emph{random return} is
\begin{equation*}
  Z^\pi(x,a)\defeq\sum_{t\geq 0}\gamma^tR_t\in \dsR^d, \quad R_t\sim \calR(\cdot|X_t,A_t)
\end{equation*}
with $(X_0,A_0)=(x,a)$, $X_{t+1}\sim P(\cdot|X_t,A_t)$, and ${A_t\sim \pi(\cdot|X_t)}$ for $t\geq 1$. ${V^\pi(x,a)\defeq \Ex{Z^\pi(x,a)}\in\dsR^d}$ is its expected \emph{value function}.
We aim to infer the \emph{return distribution function}
${\eta^\pi(x,a) \defeq \Law(Z^\pi(x,a))\in\calP(\dsR^d)}$.

\begin{assumption}[Bounded returns]\label{asm:mdp}
The rewards are bounded per dimension, $|R_i|\leq r_{\max,i}$ almost surely for $i\leq d$, and $\gamma\in[0,1)$. Write ${z_{\max}\defeq r_{\max}/(1-\gamma)\in\dsR^d_+}$, let $B_R\defeq\prod_{i\leq d}[-z_{\max,i},z_{\max,i}]$ be the return box, $R_Z\defeq\|z_{\max}\|_2$ its circumradius, and $D\defeq2R_Z$. Then $|r_i+\gamma z_i|\leq z_{\max,i}$ whenever $z\in B_R$, so $\dbo$ maps laws supported in $B_R$ to laws supported in $B_R$, and $\supp\eta^\pi(x,a)\subseteq B_R\subseteq B(0,R_Z)$ for all $(x,a)$.
\end{assumption}

\paragraph{Distributional Bellman Operator.}
For ${r \in \dsR^d}$ and $\gamma \in (0,1)$, the
\emph{affine bootstrap map} ${b_{r,\gamma}: \dsR^d\to \dsR^d}$ is defined as ${b_{r,\gamma}(z) \defeq r+\gamma z}$.
The \emph{push-forward} of a distribution $\nu\in \calP(\dsR^d)$ under $b_{r,\gamma}$ is ${(b_{r,\gamma})_\#\nu \defeq \Law(r + \gamma Z)}$ for ${Z\sim \nu}$.
The \emph{distributional Bellman operator (DBO)} ${\calT^\pi: \calP(\dsR^d)^{\calX\times\calA}\to \calP(\dsR^d)^{\calX\times\calA}}$ is defined as
\begin{equation}\label{eq:dbo}
  (\calT^\pi\eta)(x,a) \defeq \Ep{\pi}{(b_{R,\gamma})_\#\eta(X',A')|X=x,A=a},
\end{equation}
where $R\sim \calR(\cdot|x,a)$, $X'\sim P(\cdot|x,a)$, and $A'\sim\pi(\cdot|X')$. The return distribution $\eta^\pi$ is the unique fixed point of this operator, i.e., $\calT^\pi\eta^\pi = \eta^\pi$.

DBO's convergence is naturally analyzed under the \emph{Wasserstein metric} $w_p$ and its supremal version ${\bar w_p(\eta,\eta')\defeq\sup_{(x,a)}w_p(\eta(x,a),\eta'(x,a))}$.

\begin{theorem}[Contraction of DBO \citep{bellemare2017distributional,zhang2021distributional}]\label{thm:dbo_contraction}\label{thm:contraction}
  For any $p\geq 1$ and ${\eta,\eta' \in \calP_p(\dsR^d)^{\calX\times \calA}}$,
    $\bar w_p(\calT^\pi\eta,\calT^\pi\eta')\leq \gamma \bar w_p(\eta,\eta').$
\end{theorem}
So $\eta_{k+1}=\dbo\eta_k$ converges to $\eta^\pi$ in $\bar w_p$ at rate $\gamma^k$ \citep{bellemare2023distributional}.
$\dbo$ also contracts in supremal $\ell_p$, in supremal MMD \citep{gretton2012kernel} for suitable kernels, and in sup-sliced divergences \citep{nguyentang2021distributional,wiltzer2024foundations,debes2026multivariate}, but its convergence in $\bar w_p$ motivates measuring approximation errors in $\bar w_p$.

\paragraph{Smoothed Scores, SVGD, and DSM.}
We define the smoothing operator $\heat{\sigma}$ as the convolution with a normal distribution, i.e., ${\heat{\sigma}\mu\defeq \mu*\Norm(0,\sigma^2\Id)}$.  The score and smoothed score of $\mu$ are given as ${s_\mu \defeq \nabla \log \mu}$, and $\sm{\mu}{\sigma} \defeq \nabla \log (\heat{\sigma}\mu)$.
Score matching is a minimum Stein discrepancy estimator \citep{barp2019minimum}, and denoising score matching (DSM), the regression of $s_\psi(y+\sigma\varepsilon,\sigma)$ on $-\varepsilon/\sigma$ for $y\sim\mu$, $\varepsilon\sim\Norm(0,\Id)$, is score matching on the smoothed law $\heat{\sigma}\mu$ \citep{vincent2011connection}.
SVGD moves particles along ${\hat\phi^*(z)=\frac1M\sum_j[k(z_j,z)\,s(z_j)+\nabla_{z_j}k(z_j,z)]}$, the steepest descent direction of $\KLdiv{\cdot}{p}$ in the unit ball of the RKHS of $k$ for score $s=\nabla\log p$ \citep{liu2016stein}.

\paragraph{Representations.}
We focus on particle representations, i.e., equal-weight Dirac mixtures of the form
${\calF_M\defeq\{\frac1M\sum_{i=1}^M\delta_{z_i}: z_i\in\dsR^d\}}$.
\Cref{tab:capabilities} in \cref{sec:anatomy} compares the critic families by their main error.

\section{Stein Variational Distributional RL}\label{sec:method}
We model the return distribution function through its smoothed score over all noise levels, the \emph{return score field}, and use the Stein particles only to represent it.
The smoothed score exists for every return law, and smoothing commutes with the affine bootstrap map up to a rescaling of the noise level (\cref{lem:commute}), which yields a Bellman equation in score space (\cref{cor:scorebellman}).

\paragraph{The Return Score Field.}
The return score field of $\eta$ is the map ${(z,x,a,\sigma)\mapsto\sm{\eta(x,a)}{\sigma}(z)}$, which we learn with a noise-conditioned network $s_\psi(z,x,a,\sigma)$.

\begin{lemma}[Bellman consistency]\label{lem:bellman-consistency}\label{lem:commute}
  For any ${\mu \in \calP(\dsR^d)}$, $r \in \dsR^d$, $\gamma \in (0,1)$, $\sigma > 0$, we have
  \begin{equation}\label{eq:commute}
    \begin{split}
    &\text{(i)}\hspace{0.5em}{(b_{r,\gamma})}_{\#}(\heat{\sigma/\gamma}\mu) = \heat{\sigma}\left({(b_{r,\gamma})}_{\#}\mu\right),
  \hspace{1em}\text{and}\\
    &\text{(ii)}\hspace{0.5em}
  \sm{{(b_{r,\gamma})}_{\#}\mu}{\sigma}(y) = \frac1\gamma \sm{\mu}{\sigma/\gamma}\left(\frac{y-r}{\gamma}\right).
    \end{split}
\end{equation}
\end{lemma}

\begin{corollary}[Score Bellman equation; general branch mixture]
\label{cor:scorebellman}
Assume \cref{asm:mdp}. Let ${J=(R,X',A')\sim q(\cdot\mid x,a)}$ index Bellman branches and
$p_{j,\sigma}$ be the density of
${(b_{r_j,\gamma})}_{\#}\heat{\sigma/\gamma}\,\eta(x'_j,a'_j)$.
Then, for all $y$ and $\sigma>0$,
\begin{align}\label{eq:scorebellman}
  \sm{\dbo\eta(x,a)}{\sigma}(y)
  &= \Ep{j\sim q_\sigma(\cdot\mid y)}{\tfrac{1}{\gamma}\,
    \sm{\eta(x'_j,a'_j)}{\sigma/\gamma}\!\Big(\tfrac{y-r_j}{\gamma}\Big)}\notag\\
    q_\sigma(j\mid y)&\propto q(j|x,a)\,p_{j,\sigma}(y).
\end{align}
With a single deterministic branch, $q_\sigma(\cdot\mid y)$ is a point mass and \eqref{eq:scorebellman} reduces to \eqref{eq:commute}(ii), our consistency target \eqref{eq:cons}. Stochastic rewards are the case $J=R$.
\end{corollary}

We learn $s_\psi$ by DSM on the bootstrap targets ${y=r+\gamma z'}$, with $a'\sim\pi(\cdot|x')$ and $z'$ a particle at $(x',a')$, and tie the noise levels together by the score Bellman equation \eqref{eq:scorebellman}, with one sampled branch and one shared noise draw:
\begin{equation}\label{eq:dsm}
  L_{\mathrm{DSM}}(\psi)=\Ep{y,\sigma,\varepsilon}{\lambda(\sigma)\big\|s_\psi(y+\sigma\varepsilon,x,a,\sigma)+\varepsilon/\sigma\big\|^2},
\end{equation}
\begin{equation}\label{eq:cons}
  \begin{split}
    L_{\mathrm{cons}}(\psi)=\dsE\Big[&\lambda(\sigma)\big\|s_\psi(r+\gamma z+\sigma\varepsilon,x,a,\sigma)\\
  &-\tfrac1\gamma\,\stopgrad{s_{\bar\psi}(z+\tfrac\sigma\gamma\varepsilon,x',a',\tfrac\sigma\gamma)}\big\|^2\Big],
  \end{split}
\end{equation}
with $\bar\psi$ an exponential moving average (EMA) of $\psi$.
The pairing of $\sigma$ with $\sigma/\gamma$ is \eqref{eq:commute}, and it makes score TD non-expansive (\cref{thm:scoreTDinf}).
Both losses stay unbiased under one-sample bootstrapping, the property (U) of \citet{debes2026multivariate} (\cref{lem:dsmU,lem:consU}).
Tabular runs use the EDM preconditioning \citep{NEURIPS2022_a98846e9} in coordinates that standardize the targets at $(x,a)$, with $\lambda(\sigma)=\sigma^2$ (\cref{app:method}).

\paragraph{The Resample-Move Update.}
In the tabular experiments SVDRL keeps $M$ particles per state--action pair and updates them by a resample-move step of sequential Monte Carlo \citep{gilks2001following} in four steps (\cref{alg:update}).
The \emph{(i) stratified Bellman target set} draws the $M$ successors by systematic resampling, so a branch of probability $p$ receives $\lfloor Mp\rfloor$ or $\lceil Mp\rceil$ targets, and the reward noise is stratified within each successor block (\cref{app:method}).
\emph{(ii) Resampling} replaces a fraction $\rho$ of the particles by targets; $\rho=1$ carries the branch masses exactly up to $1/M$.
The \emph{(iii) Stein move} runs SVGD along $s_\psi$ down the levels $\sigma_\ell\leq\sigma_{\mathrm{start}}$ with relaxation $\alpha$, and \emph{(iv) moment matching} maps the cloud affinely onto the target set's moments.
Resampling carries the mass, the Stein move the shape.

\begin{algorithm}[t]
\caption{One tabular SVDRL update at $(x,a)$}\label{alg:update}
\begin{algorithmic}[1]
\Require particles $z_{1:M}$; score network $s_\psi$; hyperparameters $\rho$, $\sigma_{\mathrm{start}}$, $\alpha$
\State $y_{1:M}\leftarrow$ stratified Bellman targets $r_j+\gamma z'_j$
\State DSM and consistency steps \eqref{eq:dsm}--\eqref{eq:cons}; replace $\mathrm{round}(\rho M)$ particles by targets
\State $\tilde z\leftarrow$ SVGD along $s_\psi(\cdot,\sigma_\ell)$ from $z$, for $\sigma_\ell\leq\sigma_{\mathrm{start}}$
\State $z\leftarrow z+\alpha(\tilde z-z)$, then map $z_{1:M}$ onto the moments of $y_{1:M}$
\end{algorithmic}
\end{algorithm}

\paragraph{The Sampler and the Hybrid Rule.}
The sampler draws from the field by reverse diffusion of the EMA networks (\cref{def:ladderrisk}); one set of $N\gg M$ draws answers CVaR of any scalarization, event probabilities and conditional queries (\cref{cor:frontcert,cor:condquery}).
\Cref{lem:scorelip} says the field is learnable when the return law has a density whose components are not far apart relative to their widths; we turn it into a rule that the agent evaluates on its own particles, without ground truth, with a threshold fixed in advance (\cref{app:method}).
\begin{definition}[Hybrid rule]\label{def:hybrid}
Let $r$ be the ratio of the smallest gap between single-linkage clusters of the particle cloud at $(x,a)$ to the largest within-cluster standard deviation, after merging particles closer than the field's smallest noise level.
The hybrid rule reports the particles if $r>6$, or if a cluster of at least two particles is narrower than the field's smallest noise level; otherwise it reports the sampler draws.
\end{definition}

\paragraph{SAC-SVDRL.} In continuous control SVDRL is the critic of SAC \citep{haarnoja2018soft}: an amortized critic outputs $M$ particles regressed onto the transported cloud, i.e., amortized SVGD \citep{feng2017learning}, and a moment head sets the score network's standardization and moment-matching target (\cref{app:method}).

\section{Convergence Analysis}\label{sec:theory}
All results concern policy evaluation for a fixed policy $\pi$ under \cref{asm:mdp}. We certify the reverse-diffusion sampler, for bounded and possibly multimodal returns, with no isoperimetric, dissipativity or mode-separation assumption; the Stein particles enter through one measurable term. Proofs are in \cref{app:proofs}, and \cref{sec:experiments} checks \cref{thm:scoreTDinf,thm:gen} on trained networks.

\begin{definition}[Noise grid, multi-level risk, reverse-diffusion sampler]\label{def:ladderrisk}
  Let ${\sigma_0>\sigma_1>\dots>\sigma_{N_{\mathrm g}}>0}$ be the noise grid on which the score network is trained, with variance steps $h_n\defeq\sigma_n^2-\sigma_{n+1}^2$, and let ${\calI(\nu)\defeq\Ep{\nu}{\|\nabla\log\nu\|^2}}$ denote the Fisher information of a law with a smooth positive density. For probability measure $\mu$ with bounded support and score model $s_\psi(\cdot,\sigma)$ that is bounded on bounded sets, define the DSM excess risk at level $\sigma$ \citep{vincent2011connection} and the \emph{multi-level risk} by
\begin{align*}
  \calE_\sigma(\psi;\mu)&\defeq\Ep{\heat{\sigma}\mu}{\|s_\psi(\cdot,\sigma)-\sm{\mu}{\sigma}\|^2},\\
  \calE_{\mathrm{ml}}(\psi;\mu)&\defeq\sum_{n=0}^{N_{\mathrm g}-1}h_n\,\calE_{\sigma_n}(\psi;\mu).
\end{align*}
The reverse-diffusion sampler of $s_\psi$ starts from ${\hat X_0\sim\Norm(z_c,\sigma_0^2\Id)}$, for a center $z_c\in B_R$ fixed before sampling (the standardizing mean at $(x,a)$ in \cref{alg:revdiff}), and iterates, for $n=0,\dots,N_{\mathrm g}-1$ and i.i.d.\ ${\xi_n\sim\Norm(0,\Id)}$,
${\hat X_{n+1}=\hat X_n+h_n\,s_\psi(\hat X_n,\sigma_n)+\sqrt{h_n}\,\xi_n.}$
We write $\hat Q_\psi$ for the law of $\hat X_{N_{\mathrm g}}$ and ${Q_\psi\defeq(\mathrm{proj}_{B_R})_\#\hat Q_\psi}$ for the law of its projection onto the return box (\cref{alg:revdiff}). At a state--action pair we take $\mu=\mu(x,a)$ and $s_\psi(\cdot,\sigma)=s_\psi(\cdot,x,a,\sigma)$.
\end{definition}

\subsection{Score-Space Temporal Differences}\label{sec:scoreTDmain}
\begin{theorem}[Score-space TD, informal]\label{thm:scoreTDinf}
Assume a geometric grid with $\kappa_{\mathrm g}=\gamma^{1/m}$ for an integer $m$ and the Gaussian prior score $-(z-c)/\sigma^2$ above $\sigma_0$, with $c$ the center of $B_R$.
Write $\tilde e_n$ and $\tilde r_n$ for the score error and the consistency residual at level $n$, each normalized by $\sigma_n$ and uniform over $(x,a)$.
\emph{(i)} The score Bellman operator is non-expansive in the multi-level risk: the risk of its image is at most that of its input on the noise levels shifted up by $m$.
\emph{(ii)} At a fixed point of training, $\tilde e_n\leq\sum_{i=0}^{\lfloor n/m\rfloor}\tilde r_{n-im}+D/(2\sigma_0)$ for every level $n$.
\end{theorem}
By \emph{(ii)} the consistency residuals, which need no ground truth, bound the error per level; exact sweeps converge in finitely many steps, with the noise schedule as discount factor (\cref{thm:scoreTD,cor:consTD,cor:scoreTDconv}).

\subsection{A Bound on the Sampled Return Law}\label{sec:gen}
\begin{theorem}[The multi-level risk controls the sampled return law]\label{thm:gen}
Let $\mu$ be supported in $B_R$ and $s_\psi(\cdot,\sigma)$ be bounded on $\dsR^d$ for every grid level. Then
\begin{align}\label{eq:gen-kl}
  &\KLdiv{\heat{\sigma_{N_{\mathrm g}}}\mu}{\hat Q_\psi}\leq\frac{m_2(\mu;z_c)}{2\sigma_0^2}+\calE_{\mathrm{ml}}(\psi;\mu)+\veps_{\mathrm{disc}},\\
  &\text{where }\veps_{\mathrm{disc}}\defeq\sum_{n=0}^{N_{\mathrm g}-1}h_n\big(\calI(\heat{\sigma_{n+1}}\mu)-\calI(\heat{\sigma_n}\mu)\big)\geq0. \notag
\end{align}
Here $m_2(\mu;z_c)\defeq\Ep{Y\sim\mu}{\|Y-z_c\|^2}$ is the second moment of $\mu$ about the sampler's center, which is at most $\sup_{y\in B_R}\|y-z_c\|^2$, i.e., at most $R_Z^2$ for $z_c=0$. It is measurable from the bootstrap targets, which are draws of $\mu$.
  For a geometric grid, $\sigma_{n+1}=\kappa_{\mathrm g}\sigma_n$ with $\kappa_{\mathrm g}\in(0,1)$, $\veps_{\mathrm{disc}}\leq d\,\kappa_{\mathrm g}^{-2}(1-\kappa_{\mathrm g}^2)\big(1+2\log(\sigma_0/\sigma_{N_{\mathrm g}})\big)$. Consequently, with $\veps_{\mathrm{TV}}$ the square root of half the right side of \eqref{eq:gen-kl},
\begin{equation}\label{eq:gen-w1}
  \begin{split}
    \tv\big(\hat Q_\psi,\heat{\sigma_{N_{\mathrm g}}}\mu\big)&\leq\veps_{\mathrm{TV}}
  \qquad\text{and}\\
    \W1(Q_\psi,\mu)&\leq D\,\veps_{\mathrm{TV}}+\sigma_{N_{\mathrm g}}\sqrt d .
  \end{split}
\end{equation}
\end{theorem}
A clipping argument extends it to score networks of linear growth such as ours (\cref{rem:unbounded}).

\subsection{Policy Evaluation}\label{sec:e2e}
\begin{theorem}[Policy evaluation from the sampler and from Stein particles]\label{thm:e2e}\label{rem:e2emodel}
For $k=0,\dots,K-1$ let $\nu_k\in\calP(B_R)^{\calX\times\calA}$ be the bootstrap source of iteration $k$, let $\psi_{k+1}$ be a score model toward the Bellman targets ${\mu_k(x,a)\defeq\dbo\nu_k(x,a)}$ with $s_{\psi_{k+1}}(\cdot,\sigma)$ bounded on $\dsR^d$ at every grid level, and let $Q_{k+1}(x,a)$ be its sampler law. Let $\veps_{\mathrm{TV},k}(x,a)$ be the bound of \cref{thm:gen} for $\mu_k(x,a)$ and $\psi_{k+1}$, and ${\veps_{\mathrm{gen}}\defeq D\sup_{k,(x,a)}\veps_{\mathrm{TV},k}(x,a)+\sigma_{N_{\mathrm g}}\sqrt d}$.
\emph{(i) Sampler bootstrapping.} If $\nu_k=Q_k$ for every $k$, with $Q_0\in\calP(B_R)^{\calX\times\calA}$ arbitrary, then ${\Wbar1(Q_K,\eta^\pi)\leq\gamma^K\,\Wbar1(Q_0,\eta^\pi)+\veps_{\mathrm{gen}}/(1-\gamma)}$.
\emph{(ii) Stein particles.} If $\nu_k=\hat\eta_k$ are the Stein particles of iteration $k$, then, with the Stein-particle gap $\veps_{\mathrm{St}}\defeq\sup_{k,(x,a)}\W1\big(\hat\eta_{k+1}(x,a),Q_{k+1}(x,a)\big)$,
\begin{align*}
  \Wbar1(\hat\eta_K,\eta^\pi)&\leq\gamma^K\,\Wbar1(\hat\eta_0,\eta^\pi)+\frac{\veps_{\mathrm{gen}}+\veps_{\mathrm{St}}}{1-\gamma},\\
  \Wbar1(Q_K,\eta^\pi)&\leq\veps_{\mathrm{gen}}+\gamma\,\Wbar1(\hat\eta_{K-1},\eta^\pi).
\end{align*}
\end{theorem}
Every term is proved or measurable without ground truth, the multi-level risk through the consistency residuals at a fixed point of training (\cref{cor:residualcert}), and the analogues in $\Wbar2$ hold (\cref{cor:e2eW2}).
Sampler bootstrapping needs moment matching, since draws without it inflate the variance of the targets through the recursion, by $\gamma^2\sigma_{N_{\mathrm g}}^2/(1-\gamma^2)$ per coordinate on a deterministic self-loop (\cref{rem:e2ereadings}).
The Stein particles enter \emph{(ii)} only through the measured $\veps_{\mathrm{St}}$: the exact-score SVGD drift is not globally dissipative for any bounded even $C^1$ kernel with bounded gradient (\cref{prop:nodissip}).

\subsection{Three Sources of Error}\label{sec:anatomy}\label{sec:risk}
A critic's error has three sources: the $1/M$ resolution floor of a finite particle set, the sampler's smoothing bias, and the regularity of the field it must learn.
\begin{proposition}[What particles cannot answer, informal]\label{prop:resolution}
Let $p\in(0,\frac1{2M})$ and let $\lambda\defeq z_{\max}/\|z_{\max}\|_2$.
\emph{(i)} For every equal-weight $M$-atom law $Q$ and every event $A$ with target mass $p$, either $Q(A)=0$ or $Q(A)>2p$.
\emph{(ii)} There is a two-point return law $\eta$ in $B_R$ with a catastrophe of mass $p$ at depth $D$ along $\lambda$ such that \emph{every} equal-weight $M$-atom law $Q$, read by plug-in, errs by at least $\frac D2(1-2pM)$ in $\cvar\alpha(\lambda\T Z)$ at $\alpha=p$ or at $\alpha=\frac1M$, and takes the wrong chance-constrained decision for every threshold $\delta$ in an interval of length at least $p$.
\emph{(iii)} With probability at least $1-\beta$ over $N$ draws, the reverse-diffusion sampler estimates the probability of every half-space event, simultaneously for all scalarizations, with an error that does not depend on $M$: training risk $\veps_{\mathrm{TV}}$, plus a smoothing term that is at most $\frac12e^{-w^2/(2\sigma_{N_{\mathrm g}}^2)}$ when the slab of half-width $w$ around the event boundary carries no target mass, plus $\calO(\sqrt{(d\log N+\log(1/\beta))/N})$.
\end{proposition}
\emph{(i)} and \emph{(ii)} hold for every particle critic, the Stein particles included, and \emph{(i)} also rounds mixture weights to multiples of $1/M$.
At rare-event scale the certificate is qualitative, since $\veps_{\mathrm{TV}}$ can exceed $p$ (\cref{sec:experiments}); the claim is what controls the error: $1/M$ for atoms, the training risk and $N$ for the sampler.

The sampler instead has a smoothing bias, first order on point masses and second order on laws with densities.
\begin{proposition}[Smoothing bias]\label{prop:smoothbias}
Let $\eta\in\calP_2(\dsR^d)$ and $\sigma>0$.
\emph{(i)} $\W2(\heat{\sigma}\eta,\eta)\leq\sigma\sqrt d$.
\emph{(ii)} If ${\eta=\sum_kw_k\delta_{y_k}}$ with $\|y_j-y_k\|\geq\Delta$ for all ${j\neq k}$, then ${\W2^2(\heat{\sigma}\eta,\eta)\geq\sigma^2d\,\big(1-\Pr(\chi^2_{d+2}\geq\Delta^2/(4\sigma^2))\big)}$, where $\chi^2_{d+2}$ is a chi-squared variable with $d+2$ degrees of freedom; for a single atom, $\W2(\heat{\sigma}\delta_y,\delta_y)=\sigma\sqrt d$.
\emph{(iii)} If $\eta=\sum_kw_k\Norm(\mu_k,S_k)$ with $S_k\succeq s^2\Id$ for every $k$, then $\W2(\heat{\sigma}\eta,\eta)\leq\sqrt d\,\sigma^2/(2s)$.
\end{proposition}
Part \emph{(iii)} decides where the field can be learned.
\begin{lemma}[Regularity of the smoothed score]\label{lem:scorelip}
  Let $\eta\in\calP(\dsR^d)$, $\sigma>0$, $s_\sigma\defeq\sm{\eta}{\sigma}$ and ${\mathrm{Lip}(s_\sigma)\defeq\sup_z\|\nabla s_\sigma(z)\|_{\mathrm{op}}}$.
\emph{(i)} If $\supp\eta$ has diameter at most $D$, then $-\sigma^{-2}\Id\preceq\nabla s_\sigma(z)\preceq\sigma^{-4}(D^2/4-\sigma^2)\Id$ for every $z$.
\emph{(ii)} If $\eta=\sum_{k=1}^Kw_k\delta_{y_k}$ with $K\geq2$ atoms at pairwise distances at least $\delta$, attained by some pair, then $\liminf_{\sigma\to0}\sigma^4\,\mathrm{Lip}(s_\sigma)\geq\delta^2/4$.
\emph{(iii)} If $\eta=\sum_kw_k\Norm(\mu_k,S)$ with $S\succeq s^2\Id$ and the means in a set of diameter $\Delta$, then $\mathrm{Lip}(s_\sigma)\leq\max\{s^{-2},\Delta^2/(4s^4)\}$ for every $\sigma>0$.
\end{lemma}
So the field of a law with densities is equally regular on every level, while the Lipschitz constant of an atomic law's field grows like $\sigma^{-4}$ at the lowest noise levels, where a network of bounded regularity cannot fit it uniformly and Euler steps of the sampler become unstable (\cref{app:proofs}).
Particles need no regularity of the field. This is the split we observe in \cref{sec:experiments}, and the hybrid rule of \cref{def:hybrid} is its agent-side form.

\begin{table}[t]
  \caption{\emph{Main errors of the critic families.} The error that limits each family (\cref{prop:resolution,prop:smoothbias,lem:scorelip}). $d{>}1$: joint return law; ${<}1/M$: events rarer than one particle's mass. Particle critics with fresh or time-averaged targets differ only in sampling variance. Only SVDRL has both without a grid.}
  \label{tab:capabilities}
  \centering\small\setlength{\tabcolsep}{4pt}
  \begin{tabular}{lccp{1.95cm}}
    \toprule
    Critic & $d{>}1$ & ${<}1/M$ & Main error \\
    \midrule
    Categorical grid      & $m^d$ bins & on grid            & bin width \\
    Quantile, IQN         & --         & IQN, one $\lambda$ & $d{=}1$ only \\
    Particle critics      & \checkmark & --                 & $1/M$ floor \\
    Scalar generative     & --         & \checkmark         & $d{=}1$ only \\
    \textbf{SVDRL}        & \checkmark & \checkmark         & smoothing, regularity \\
    \bottomrule
  \end{tabular}
\end{table}

\section{Experiments}\label{sec:experiments}
\paragraph{Protocol.} In the tabular experiments a seed's error is its mean over the last $20\,\%$ of training, which averages out the fluctuation of single evaluations.
Configurations are selected on tuning seeds and reported on different seeds that the selection did not use, in the rare-event study additionally on a second set of fresh seeds.
We report the interquartile mean (IQM) or median over seeds with $95\,\%$ bootstrap intervals; on the synthetic MDPs and the chain a comparison is a win or a loss when the Welch $t$ statistic is at most $-2$ or at least $2$, and in the rare-event study it compares IQM errors (\cref{app:protocols}).
The particle competitors are MMD critics \citep{nguyentang2021distributional,zhang2021distributional} with RBF, MD3QN and energy kernels, sliced Cram\'er and sliced $\W2$ critics, the persistent-particle objectives of \citet{debes2026multivariate} and Sinkhorn critics \citep{sun2024distributional}, each with a tuned learning rate and each also run with our stratified target set (\cref{app:hparams}).
The MMD critic with the energy kernel, energy MMD, is our own variant; its loss is the energy distance, in which the multivariate distributional Bellman operator contracts \citep{wiltzer2024foundations}.

\begin{figure*}[t]
  \centering
  \adjustbox{width=0.9\textwidth}{\includegraphics{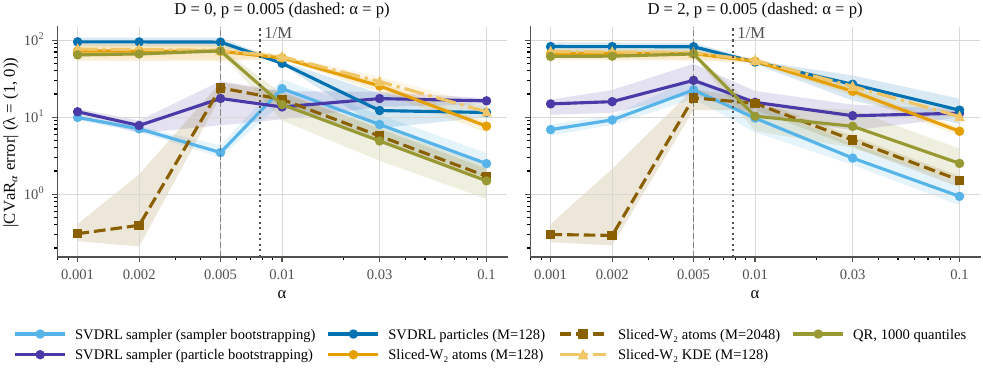}}
  \caption{\emph{Rare events.} Rare-event chain, catastrophe of mass $p=0.005$ at depth $D=0$ (left) and $D=2$ (right): CVaR$_\alpha$ error of $Z_1$. Below $1/M$, particle critics with $M=128$ miss the catastrophe and SVDRL's sampler errs an order of magnitude less; sliced $\W2$ needs $M=2048$ to locate it. IQM over ten seeds, $95\,\%$ bootstrap intervals.}
  \label{fig:rare}
\end{figure*}

\paragraph{Rare Events.}\label{sec:rare}\label{sec:capabilities}\label{sec:tail}
In the rare-event chain (\cref{app:protocols}) the risky root action leads, $D\in\{0,2\}$ steps later, to a catastrophe of mass $p\in\{0.002,0.005,0.01\}$ in a two-dimensional return, and the queries on the joint law are CVaR down to $\alpha=0.001$, event probabilities, chance constraints, conditional tails and entropic risk.
Six particle critics answer from their particles or from a kernel density estimate on them, a quantile critic is trained on the oracle scalarization, and SVDRL bootstraps from moment-matched sampler draws (\cref{thm:e2e}(ii), \cref{rem:e2ereadings}) and answers with $N=10^5$ draws.

At equal particle count, $M=128$, SVDRL has a lower IQM error than the best particle critic in 60 of 60 rare-event and 12 of 12 bulk comparisons, and on fresh seeds its error is lower by a median factor of 5.0 on rare events and 3.9 on the bulk (\cref{fig:rare}).
At comparable compute, against particle critics with $M=512$, it is still more accurate in 40 of 60 rare-event and 10 of 12 bulk comparisons.
At $p=0.005$ and $D=2$ it is more accurate than every particle critic on every query of \cref{tab:re1main}, from their atoms or from a kernel density estimate, and the quantile critic on the oracle scalarization matches it only at CVaR$_{0.01}$ of that scalarization. Its $|\log\hat P/P|$ there is 0.16, against 4.8 to 8.5 for the particle critics.
Particle critics need about $10/p$ particles to locate the catastrophe: with $M=2048$, sliced $\W2$ reaches a CVaR$_{0.001}$ error of 0.30 against SVDRL's 6.9.
Even then SVDRL is more accurate on the mass ($|\log\hat P/P|$ 0.16 against 0.26), on the chance-constrained decisions (accuracy 1.00 against 0.96) and on the bulk (CVaR$_{0.1}$ error 0.94 against 1.5).
With sampler bootstrapping the catastrophe survives the recursion even at $M=32$ for $p\geq0.005$, where every particle critic loses it (\cref{app:res-rare}).
This requires the rare mode to have a density.

Resource gathering from MO-Gymnasium \citep{felten_toolkit_2023} has atomic returns, with death as the rare event.
At the standard attack probability $p=0.1$ and at ${p=0.02}$, SVDRL's particles ($M=128$) are the best critic: their $|\log\hat P/P|$ for death is 0.028 and 0.15 against at best 0.19 and 0.45, their chance-constrained choice is wrong in 0 and 13\,\% of the cases against 14 and 41\,\%, and their CVaR error at $p=0.1$ is 0.0035 against 0.030 (medians over ten seeds, \cref{app:res-rg}).
The stratified target set holds each atom at its mass up to $1/M$, and at $p=0.005<1/M$ full replacement rounds the mass without bias: the time-averaged estimate is 0.0053 for $P=0.005$.
The sampler fails on these atomic laws, as \cref{lem:scorelip} predicts, and the hybrid rule matches the particles on the probability of death and on the chance-constrained choice.

\paragraph{Eight Synthetic MDPs.}\label{sec:tabular}\label{sec:against}
\begin{table}[t]
  \centering
  \caption{\emph{Eight synthetic MDPs.} Within-mode $\W2^2$ in units of $10^{-3}$, averaged over the last $20\,\%$ of training. SVDRL reports the hybrid of \cref{def:hybrid}. Ratios to the best competitor (\cref{tab:tab1comp}); own: as specified; all: also on our target set; $M=4096$: compute-matched, unimodal MDPs only. Bold: SVDRL is better by Welch's test. Mode masses in \cref{tab:mainmass}.}
  \label{tab:main}
  \small
  \input{tables/tab1_compact.tex}
\end{table}
We split the error into the total variation of the mode masses (\cref{tab:mainmass}) and the within-mode $\W2^2$ (\cref{tab:main}), because the full $\W2^2$ charges a mass error $\epsilon$ across a gap $\Delta$ as $\epsilon\Delta^2$ and is dominated by rounding when modes are far apart (\cref{app:res-mdp}).
On within-mode shape the hybrid scores 8/0/0 wins/ties/losses against the competitors as specified and 7/0/1 against every row, competitors given our target set included; a single configuration for all eight MDPs, without per-MDP selection, gives the same counts, and the separation test alone ($r>6$) makes the same choice of object on every MDP (\cref{tab:sel1objects}).
Its particles alone score 6/1/1 and 4/1/3 in the same two framings.
On mass every method sits at the $1/M$ limit (0/2/2), as \cref{prop:resolution}(i) predicts: SVDRL allocates within one particle of the weights and rounds stochastically, so its mass is unbiased over time (bias 0.0013 on MDP2, against a floor of 0.0031 at which persistent clouds get stuck).

On the four unimodal MDPs the hybrid reports $N=4096$ sampler draws against $M=64d$ particles, and at matched compute these wins vanish: competitors with $M=4096$ beat the sampler on all four (0/0/4).

The ablations and the checks of the theory support every component (\cref{app:res-mdp,app:res-theory}). On MDP1, MDP5 and MDP8 the Stein move beats Langevin dynamics along the same score, a Gaussian jitter and no move by factors of 1.19 to 1.39, and without repulsion the cloud collapses on MDP5 and MDP8, with 28.6 and 204 times the error (\cref{tab:rj1}). Removing the consistency loss raises the error significantly on four MDPs and never lowers it, and a single-level network misallocates the mode masses (0.32 against 0.056 on MDP2; \cref{rem:laddernecessity}). On trained networks \cref{thm:scoreTDinf}(ii) holds in 87 of 87 seeds and the bound of \cref{thm:gen} in 80 of 80 cases, but loosely: $\veps_{\mathrm{TV}}$ exceeds the measured total variation 4.4 to 34 times. The ground-truth-free gap $\veps_{\mathrm{St}}$ of \cref{thm:e2e} equals the particles' true error within 1 to 12\,\% on the unimodal MDPs.

\begin{table*}[t]
  \centering
  \caption{\emph{Continuous control.} MuJoCo after $10^6$ steps, IQM over five to ten seeds with its $95\,\%$ bootstrap interval; bold: inside the interval of the best IQM. Return: each run's mean evaluation return over the last $20\,\%$ of training. Critic shape: centered $\W2$ between the critic and Monte Carlo returns of the frozen policy, over their standard deviation. SVDRL's configurations in \cref{app:res-mujoco}.}
  \label{tab:mujoco}\label{tab:mc}
  \footnotesize
  \setlength{\tabcolsep}{4pt}
  \adjustbox{max width=0.98\textwidth}{\input{tables/tab_mujoco_main.tex}}
\end{table*}

\begin{figure*}[t]
  \centering
  \includegraphics{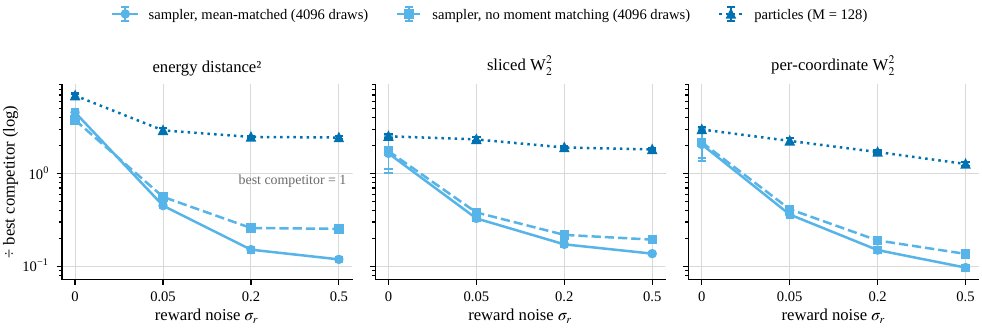}
  \caption{\emph{Chain with reward noise.} SVDRL's error over the best competitor's ($M=128$) per metric against the reward noise $\sigma_r$, for its sampler and particles; below one, SVDRL is better. The sampler loses on the atomic chain and wins for $\sigma_r\geq0.05$, until competitors get $M=4096$ particles. Means over ten seeds, $95\,\%$ intervals.}
  \label{fig:chain}
\end{figure*}

\paragraph{Chain with Reward Noise.}\label{sec:chain}
The chain of \citet{debes2026multivariate} has ten reward dimensions, to which we add Gaussian reward noise of scale $\sigma_r$.
On the atomic chain, $\sigma_r=0$, their persistent particles win, as \cref{lem:scorelip} predicts.
With reward noise, $\sigma_r\geq0.05$, and against their critics at their published settings with $M=128$ instead of $500$ particles, SVDRL's moment-matched sampler ($4096$ draws) wins every shape metric and the mean (\cref{fig:chain}), \eg, with 0.15 times the best competitor's energy distance and 0.56 times its mean error at $\sigma_r=0.2$.
With more particles they win shape at matched compute: with $M=4096$ their shape errors are 1.6 to 3.0 times lower than the sampler's.
Our stratified target set lowers the energy distance and mean error of every one of their variants, and running-mean moment matching, a shift of their particles to a running mean of the bootstrap targets, lets them win the mean too, so both of our components transfer to their critics (\cref{app:res-chain}).

\paragraph{Continuous Control.}\label{sec:mujoco}\label{sec:univariate}
On four MuJoCo environments \citep{towers2024gymnasium}, SAC with the SVDRL critic runs against SAC and against IQN, MD3QN and energy MMD critics on the same trunk (\cref{tab:mujoco}).
Its return is on par with SAC's, with overlapping $95\,\%$ intervals on all four environments, although its IQM on HalfCheetah is lower.
Its critic distribution is the most accurate on three of four: its centered $\W2$, which removes the twin-critic bias of the means, is the lowest on HalfCheetah, Walker2d and Hopper, with the interval entirely below IQN's, and second on Ant, behind energy MMD.
With the default configuration of HalfCheetah it is 0.70 on Walker2d and 0.57 on Ant (\cref{tab:mc_full}), the same ranks.

\section{Related Work}\label{sec:related}
\paragraph{Multivariate Distributional RL.} MD3QN \citep{zhang2021distributional} trains particles with an MMD loss, SinkhornDRL \citep{sun2024distributional} with entropic optimal transport, and sliced DRL \citep{debes2026multivariate} with sliced divergences, whose max-sliced version also contracts under matrix discounting.
Their analyses establish the contraction of the Bellman operator in the respective divergence, and \citet{wiltzer2024foundations} prove convergence of categorical dynamic programming and TD learning on a fixed support; our guarantees concern a learned score-based representation of possibly multimodal returns (\cref{thm:gen,thm:e2e}), also under matrix discounting (\cref{cor:e2eGamma}).
\citet{yoo2026knotherosenblatt} extend quantile regression to vector returns through the Knothe--Rosenblatt map.

\paragraph{Generative Critics.} Diffusion and flow critics such as VDRL \citep{hu2025value}, DBC \citep{ding2026distributional}, Value Flows \citep{dong2026value} and FlowIQN \citep{groom2026quantile} model scalar returns, and DBC and FlowIQN rely on one-dimensional structure, inverse distribution functions or sorted couplings.
Bellman Diffusion \citep{li2024bellman} learns unnormalized density fields of scalar returns on which the Bellman operator acts linearly, and \cref{cor:scorebellman} is the normalized counterpart in score space, which is nonlinear through the posterior branch weights.
The Bellman GAN \citep{freirich2019distributional} samples vector-valued returns and distributional Sobolev training \citep{debes2026distributional} returns with their action gradients, without a density or a score.
None of these works bounds the error of a learned multivariate representation.

\paragraph{Particle Transport.} SVGD \citep{liu2016stein} is a kernelized Wasserstein gradient flow of the KL divergence \citep{liu2017svgd,chewi2020svgd}, and its convergence guarantees rest on functional inequalities of the target, with constants that grow exponentially with the separation of modes (\cref{prop:dsmblind}(i)).
In RL, SVGD has so far trained policies \citep{liu2017stein,messaoud2024s2ac}, not return distributions.

\section{Limitations and Conclusion}\label{sec:conclusion}
The main errors of \cref{tab:capabilities} also mark where SVDRL is not the best critic.
The sampler fails on atomic and widely separated laws, such as the atomic chain, resource gathering and the far modes of MDP4, as \cref{lem:scorelip} predicts, and the hybrid rule reports the particles there.
A light rare mode far from the bulk comes out too wide (\cref{fig:rarelaw}), and a static fit of the score network to draws of the true law reproduces this, with 13 to 15 times the true width (\cref{app:limits}).
At equal $N$ the sampler loses to fitted particle clouds (\cref{sec:experiments}).
In SAC-SVDRL a linear-variance floor made one Ant seed diverge, which a log-variance head removes on HalfCheetah and a standardized moment head limits on multi-objective MuJoCo.

SVDRL learns a Bellman-consistent score field that yields a certified sampler without a resolution floor and Stein particles.
It answers rare-event queries that particles cannot represent, and two of its components improve every particle critic we tested.
At equal particle count it is more accurate than the best particle critic in 60 of 60 rare-event comparisons, and in continuous control SAC-SVDRL learns the most accurate critic distribution on three of four MuJoCo environments.

\clearpage
\subsection*{AI Use Statement}

In this work, we used generative AI tools for the following tasks.
It was used to optimize the original implementation of the methods by the authors to improve performance and experimental throughput. All generated code was reviewed by the authors, compared numerically against the original implementation and an extensive set of unit tests.
We further used these tools during the writing process as follows.
Drafting of sections, typesetting the proofs by the authors into their final form, and grammar checks.
The final form of the proofs was checked manually by the authors to ensure they remained consistent.
Whenever a tool suggested the addition of a reference in any of the sections, it was flagged and manually checked for correctness by the authors.
To ensure that results could not inadvertently be changed by the tool during typesetting, all are generated automatically from the logged experimental data.
Any error in the proofs as well as incorrect attribution of theoretical or empirical results is therefore due to the authors and we take full responsibility for it and everything written in the paper.

%% file: tables/tab1_compact.tex
\begin{tabular}{lcrrrr}
\toprule
 & & \multicolumn{4}{c}{Within-mode shape, $\W2^2$ ($\downarrow$)} \\
\cmidrule(lr){3-6}
 & & & \multicolumn{3}{c}{SVDRL / best competitor} \\
\cmidrule(lr){4-6}
MDP & $d$ & SVDRL & own & all & $M{=}4096$ \\
\midrule
MDP1 & 1 & 5.7 & \textbf{0.23} & \textbf{0.27} & 10 \\
MDP2 & 1 & 81 & \textbf{0.83} & 1.6 & -- \\
MDP3 & 2 & 389 & \textbf{0.20} & \textbf{0.80} & -- \\
MDP4 & 5 & 62 & \textbf{0.020} & \textbf{0.097} & -- \\
MDP5 & 2 & 6.1 & \textbf{0.17} & \textbf{0.20} & 1.2 \\
MDP6 & 2 & 50 & \textbf{0.054} & \textbf{0.12} & -- \\
MDP7 & 3 & 1.5 & \textbf{0.19} & \textbf{0.27} & 4.1 \\
MDP8 & 2 & 38 & \textbf{0.19} & \textbf{0.23} & 2.3 \\
\midrule
\rlap{Wins/ties/losses} & & & 8/0/0 & 7/0/1 & 0/0/4 \\
\rlap{\quad its particles alone} & & & 6/1/1 & 4/1/3 & 0/0/4 \\
\bottomrule
\end{tabular}

%% file: tables/tab_mujoco_main.tex
\begin{tabular}{lrrrrrcrrrr}
\toprule
 & \multicolumn{5}{c}{Return ($\uparrow$)} && \multicolumn{4}{c}{Critic shape ($\downarrow$)} \\
\cmidrule(lr){2-6}\cmidrule(lr){8-11}
 & SVDRL & SAC & IQN & MD3QN & E-MMD && SVDRL & IQN & MD3QN & E-MMD \\
\midrule
HalfCheetah & $8720\sd{[7565,\,9889]}$ & $\mathbf{10\,014}\sd{[9446,\,10\,512]}$ & $\mathbf{9219}\sd{[8338,\,9879]}$ & $\mathbf{8986}\sd{[8449,\,9877]}$ & $\mathbf{10\,098}\sd{[8885,\,11\,186]}$ && $\mathbf{0.43}\sd{[0.28,\,0.56]}$ & $1.05\sd{[0.82,\,1.41]}$ & $0.69\sd{[0.50,\,0.92]}$ & $\mathbf{0.52}\sd{[0.39,\,0.71]}$ \\
Walker2d & $\mathbf{5514}\sd{[5025,\,5780]}$ & $\mathbf{5372}\sd{[4939,\,5744]}$ & $5083\sd{[4466,\,5496]}$ & $\mathbf{5661}\sd{[5129,\,5907]}$ & $\mathbf{5609}\sd{[5361,\,5805]}$ && $\mathbf{0.92}\sd{[0.61,\,1.31]}$ & $3.99\sd{[2.32,\,6.13]}$ & $\mathbf{1.19}\sd{[0.90,\,2.75]}$ & $2.08\sd{[0.82,\,7.51]}$ \\
Hopper & $\mathbf{2187}\sd{[2047,\,2294]}$ & $\mathbf{2187}\sd{[2083,\,2319]}$ & $2019\sd{[1841,\,2221]}$ & $1822\sd{[1743,\,1889]}$ & $\mathbf{2180}\sd{[2047,\,2255]}$ && $\mathbf{2.25}\sd{[0.87,\,4.56]}$ & $10.36\sd{[6.69,\,20.64]}$ & $6.10\sd{[1.69,\,14.74]}$ & $21.26\sd{[10.69,\,44.55]}$ \\
Ant & $2429\sd{[2059,\,2654]}$ & $2426\sd{[2214,\,2759]}$ & $2411\sd{[2095,\,3029]}$ & $1233\sd{[895,\,2343]}$ & $\mathbf{2928}\sd{[2628,\,3598]}$ && $0.59\sd{[0.57,\,0.63]}$ & $0.73\sd{[0.42,\,1.19]}$ & $0.72\sd{[0.56,\,2.95]}$ & $\mathbf{0.49}\sd{[0.45,\,0.56]}$ \\
\bottomrule
\end{tabular}

%% file: appendix.tex
\thispagestyle{empty}
\aistatstitle{Supplementary Material}

\section{Proofs}\label{app:proofs}

This section proves the results of \cref{sec:method,sec:theory} in the order of the main text, together with the corollaries that \cref{sec:experiments} uses.
We assume \cref{asm:mdp} throughout.

\subsection{Notation and Tools}\label{app:tools}

\begin{definition}[Wasserstein-$p$ distance]\label{def:wp}
For $p\geq1$ and laws $\nu,\nu'$ on $\dsR^d$ with finite $p$-th moments, $w_p(\nu,\nu')\defeq\big(\inf_{\iota\in\Pi(\nu,\nu')}\int\|x-y\|_2^p\,\dd\iota\big)^{1/p}$, where $\Pi(\nu,\nu')$ is the set of couplings of $\nu$ and $\nu'$, and $\bar w_p(\eta,\eta')\defeq\sup_{(x,a)}w_p(\eta(x,a),\eta'(x,a))$ for return distribution functions $\eta,\eta'$.
\end{definition}

\paragraph{Posterior Form and Hessian.}
With $\varphi_\sigma$ the density of $\Norm(0,\sigma^2\Id)$, where $\sigma > 0$, write $p_\sigma(z)=\int\varphi_\sigma(z-y)\mu(\dd y)$
for the density of $\heat{\sigma}\mu$, and $Z=Y+\sigma\veps$ with $Y\sim\mu$,
$\veps\sim\Norm(0,\Id)$ independent.
Differentiating under the integral with
${\nabla_z\varphi_\sigma(z-y) = -\sigma^{-2}(z-y)\varphi_\sigma(z-y)}$ and
recognizing the posterior weights ${\varphi_\sigma(z-y)\mu(\dd y)/p_\sigma(z)}$ gives
\begin{equation}\label{eq:posterior-score}
  \sm{\mu}{\sigma}(z) = \frac{\nabla p_\sigma(z)}{p_\sigma(z)}
  = -\frac{z-m_\sigma(z)}{\sigma^2},
\end{equation}
where $m_\sigma(z) = \Ex{Y\mid Z=z}$, and $\mu(\dd y\mid z) \propto \varphi_\sigma(z-y)\mu(\dd y)$.
Since that posterior is an exponential family in $z$ with natural parameter
$z/\sigma^2$, we have ${\nabla_z m_\sigma(z)=\sigma^{-2}\Cov[Y\mid Z=z]}$, and
differentiating \eqref{eq:posterior-score} once more gives
\begin{equation}\label{eq:score-hessian}
  \nabla^2\log p_\sigma(z) = -\sigma^{-2}\Id + \sigma^{-4}\Cov[Y\mid Z=z].
\end{equation}

\paragraph{Denoising Score Matching.}
The DSM loss $L_{\mathrm{DSM}}(\psi;\sigma)=\Ep{Y,\veps}{\|s_\psi(Y+\sigma\veps,\sigma)+\veps/\sigma\|^2}$ satisfies Vincent's identity \citep{vincent2011connection}
\begin{equation}\label{eq:vincent}
  L_{\mathrm{DSM}}(\psi;\sigma)=\Ep{Z\sim\heat{\sigma}\mu}{\big\|s_\psi(Z,\sigma)-\sm{\mu}{\sigma}(Z)\big\|^2}+C_{\mathrm{DSM}}(\sigma),
\end{equation}
with $C_{\mathrm{DSM}}(\sigma)$ independent of $\psi$, because $-\veps/\sigma=-(Z-Y)/\sigma^2$ has the conditional mean $\sm{\mu}{\sigma}(z)$ given $Z=z$ by \eqref{eq:posterior-score}, and $C_{\mathrm{DSM}}(\sigma)$ is its conditional variance.
So the excess risk of DSM is $\calE_\sigma$ of \cref{def:ladderrisk}.

\paragraph{Gaussian Norms and CVaR.}
For $\veps\sim\Norm(0,\Id)$ the map $\veps\mapsto\|\veps\|$ is $1$-Lipschitz with $\Ex{\|\veps\|}\leq\sqrt d$, so Gaussian concentration \citep{boucheron2013concentration} gives
\begin{equation}\label{eq:gauss-tail}
  \Pr(\|\veps\|\geq\sqrt d+u)\leq e^{-u^2/2},\qquad u\geq0.
\end{equation}

\begin{lemma}[CVaR utility facts]\label{lem:cvarfacts}
Let $\alpha\in(0,1)$ and write the lower-tail CVaR of a law $P$ on $\dsR$ with finite first moment as $\cvar\alpha(P)=\sup_cg_P(c)$, $g_P(c)\defeq c-\alpha\inv\Ep{P}{(c-X)_+}$ \citep{rockafellar2002conditional}.
\emph{(i)} $|\cvar\alpha(P_1)-\cvar\alpha(P_2)|\leq\alpha\inv\W1(P_1,P_2)$.
\emph{(ii)} If $P_1,P_2$ are supported in a closed interval of length at most $D$, then $|\cvar\alpha(P_1)-\cvar\alpha(P_2)|\leq(D/\alpha)\|F_{P_1}-F_{P_2}\|_\infty$.
\emph{(iii)} For laws $\nu_1,\nu_2$ on $\dsR^d$ and $\|\lambda\|_2\leq1$, $\W1(\lambda_\#\nu_1,\lambda_\#\nu_2)\leq\W1(\nu_1,\nu_2)$.
\emph{(iv)} If $P_1,P_2$ are supported in a set of diameter $D$, then $\W1(P_1,P_2)\leq D\,\tv(P_1,P_2)$ and $\W2(P_1,P_2)\leq D\,\tv(P_1,P_2)^{1/2}$.
\end{lemma}
\begin{proof}
\emph{(i)} As $x\mapsto(c-x)_+$ is $1$-Lipschitz, $|g_{P_1}(c)-g_{P_2}(c)|\leq\alpha\inv\Ex{|X_1-X_2|}$ for every coupling and every $c$; take the infimum over couplings.
\emph{(ii)} For $P$ supported in $I=[a,b]$ the supremum over $c$ is attained in $I$, and $\Ep{P}{(c-X)_+}=\int_a^cF_P(x)\,\dd x$ for $c\in I$ by Fubini's theorem, which differs between $P_1$ and $P_2$ by at most $D\|F_{P_1}-F_{P_2}\|_\infty$.
\emph{(iii)} A coupling $(Z_1,Z_2)$ gives the coupling $(\lambda\T Z_1,\lambda\T Z_2)$ with $\Ex{|\lambda\T(Z_1-Z_2)|}\leq\Ex{\|Z_1-Z_2\|}$.
\emph{(iv)} With the common part $\nu\defeq P_1\wedge P_2$, of mass $1-\tv(P_1,P_2)$, the coupling $(\mathrm{id},\mathrm{id})_\#\nu+(P_1-\nu)\otimes(P_2-\nu)/\tv(P_1,P_2)$ moves only the mass $\tv(P_1,P_2)$, by at most $D$.
\end{proof}

\subsection{The Score Bellman Equation}\label{proof-commute}

\begin{proof}[Proof of \cref{lem:commute}]
\emph{(i)} For independent $Y\sim\mu$ and $\veps\sim\Norm(0,\Id)$ we have $Y+\frac\sigma\gamma\veps\sim\heat{\sigma/\gamma}\mu$ and $b_{r,\gamma}(Y+\frac\sigma\gamma\veps)=(r+\gamma Y)+\sigma\veps\sim\heat{\sigma}((b_{r,\gamma})_\#\mu)$.
\emph{(ii)} For every law $\nu$ and $\tau>0$, $\heat{\tau}\nu$ has the smooth, strictly positive density $\int\varphi_\tau(\cdot-w)\,\nu(\dd w)$, so its score exists everywhere.
By (i) and the change of variables $y=r+\gamma u$, the density of $\heat{\sigma}((b_{r,\gamma})_\#\mu)$ is $y\mapsto\gamma^{-d}p_{\heat{\sigma/\gamma}\mu}((y-r)/\gamma)$; under $\nabla_y\log$ the constant $\gamma^{-d}$ drops out, and the chain rule gives the factor $1/\gamma$.
\end{proof}

\begin{proof}[Proof of \cref{cor:scorebellman}]
By \eqref{eq:dbo}, $\dbo\eta(x,a)$ is the $q(\cdot\mid x,a)$-mixture of the branch laws $(b_{r_j,\gamma})_\#\eta(x'_j,a'_j)$, and $\heat{\sigma}$ is linear, so by \cref{lem:commute}(i) the density of $\heat{\sigma}\dbo\eta(x,a)$ is $p_\sigma=\int p_{j,\sigma}\,q(\dd j\mid x,a)$.
Differentiating under the integral,
\begin{equation*}
  \nabla\log p_\sigma(y)=\int\frac{p_{j,\sigma}(y)\,q(\dd j\mid x,a)}{p_\sigma(y)}\nabla\log p_{j,\sigma}(y)=\Ep{j\sim q_\sigma(\cdot\mid y)}{\nabla\log p_{j,\sigma}(y)},
\end{equation*}
and \cref{lem:commute}(ii) gives $\nabla\log p_{j,\sigma}(y)=\gamma\inv\sm{\eta(x'_j,a'_j)}{\sigma/\gamma}((y-r_j)/\gamma)$.
For discrete branches the exchange is a finite sum; in general it follows by dominated convergence for every $\eta$, since by \cref{lem:commute}(i) $p_{j,\sigma}$ is the convolution of $\varphi_\sigma$ with the probability measure $(b_{r_j,\gamma})_\#\eta(x'_j,a'_j)$, so that $p_{j,\sigma}\leq\sup\varphi_\sigma$ and $\|\nabla p_{j,\sigma}\|\leq\sup\|\nabla\varphi_\sigma\|<\infty$, uniformly in $j$ and $y$, and $p_\sigma>0$.
\end{proof}

The score Bellman equation defines the operator that the consistency loss applies to the target network.
\begin{definition}[Score Bellman image]\label{def:scoreimage}
Let $\nu\in\calP(B_R)^{\calX\times\calA}$ be a bootstrap source, $\mu\defeq\dbo\nu$, and $q_\sigma(\dd j\mid y)$ the posterior branch weights of \cref{cor:scorebellman} for the source $\nu$ at $(x,a)$. The \emph{score Bellman image} of a measurable field $f(\cdot;x',a',\sigma')$ is
\begin{equation*}
  (\calS_\nu f)(y;x,a,\sigma)\defeq\int q_\sigma(\dd j\mid y)\,\gamma\inv f\Big(\frac{y-r_j}{\gamma};x'_j,a'_j,\frac\sigma\gamma\Big),
\end{equation*}
so that $\calS_\nu\sm{\nu}{\cdot}=\sm{\mu}{\cdot}$ by \cref{cor:scorebellman}.
\end{definition}

\begin{corollary}[Commutation under invertible matrix discounting]
\label{cor:commuteGamma}
Let $\Gamma\in\dsR^{d\times d}$ be invertible and
${b_{r,\Gamma}(z)\defeq r+\Gamma z}$. Then for every
$\mu\in\calP(\dsR^d)$ and $\sigma>0$,
\begin{align*}
  &\text{(i) }{(b_{r,\Gamma})}_{\#}\Big(\mu * \Norm\big(0,\sigma^2
(\Gamma\T\Gamma)\inv\big)\Big)
= \heat{\sigma}\big({(b_{r,\Gamma})}_{\#}\mu\big), \text{ and }\\
  &\text{(ii) }
\sm{{(b_{r,\Gamma})}_{\#}\mu}{\sigma}(y)
= \Gamma^{-\mathsf T}
\nabla\log\big(\mu*\Norm(0,\sigma^2(\Gamma\T\Gamma)\inv)\big)
\big(\Gamma\inv(y-r)\big).
\end{align*}
\end{corollary}
\begin{proof}
  For independent $Z\sim\mu$ and $\veps\sim\Norm(0,\Id)$, $r+\Gamma(Z+\sigma\Gamma\inv\veps)=r+\Gamma Z+\sigma\veps$ and $\sigma\Gamma\inv\veps\sim\Norm(0,\sigma^2(\Gamma\T\Gamma)\inv)$, which gives (i); (ii) follows from (i) as in the proof of \cref{lem:commute}.
\end{proof}
Under matrix discounting the predecessor is smoothed anisotropically, so the multi-level analysis of \cref{proof-scoreTD}, which pairs the levels $\sigma$ and $\sigma/\gamma$, is for a scalar $\gamma$, whereas the policy-evaluation bound needs no such pairing (\cref{cor:e2eGamma}).

\subsection{One-Sample Bootstrap Targets}\label{proof-consbias}

\begin{lemma}[One-sample DSM is (U)]\label{lem:dsmU}
Fix $(x,a)$ and $\sigma>0$, let $J=(R,X',A')\sim q(\cdot\mid x,a)$ index the Bellman branches as in \cref{cor:scorebellman}, with branch targets ${\mu_j\defeq(b_{r_j,\gamma})_\#\eta(x_j',a_j')}$ and mixture target ${\mu\defeq\dbo\eta(x,a)}$, and let $L^{(j)}_{\mathrm{DSM}}(\psi;\sigma)$ be the DSM loss \eqref{eq:dsm} with $Y\sim\mu_j$. Then
\begin{equation}\label{eq:dsmU}
  \Ep{j\sim q(\cdot|x,a)}{L^{(j)}_{\mathrm{DSM}}(\psi;\sigma)} = \Ep{Z\sim\heat{\sigma}\mu}{\big\|s_\psi(Z,\sigma)-\sm{\mu}{\sigma}(Z)\big\|^2}+C_{\mathrm{DSM}}(\sigma),
\end{equation}
with $C_{\mathrm{DSM}}(\sigma)$ independent of $\psi$. So \emph{(i)} the population minimizer of the one-sample bootstrap DSM objective over measurable fields is $\sm{\dbo\eta(x,a)}{\sigma}$, the posterior-weighted branch mixture \eqref{eq:scorebellman}, and \emph{(ii)} whenever $\nabla_\psi$ commutes with the expectations in \eqref{eq:dsmU}, the sampled-branch gradient $\nabla_\psi L^{(J)}_{\mathrm{DSM}}(\psi;\sigma)$ is an unbiased estimate of $\nabla_\psi\Ep{\heat{\sigma}\mu}{\|s_\psi-\sm{\mu}{\sigma}\|^2}$.
\end{lemma}

\begin{proof}[Proof of \cref{lem:dsmU}]
Drawing $J\sim q(\cdot\mid x,a)$ and then $Y\sim\mu_J$ gives $Y\sim\mu$, so by Tonelli's theorem the left side of \eqref{eq:dsmU} is the DSM loss for $\mu$, and Vincent's identity \eqref{eq:vincent} gives \eqref{eq:dsmU}.
Its right side is minimized exactly by $\sm{\mu}{\sigma}$, which \cref{cor:scorebellman} identifies, and (ii) is the exchange of $\nabla_\psi$ with the expectation over $J$; the exchange holds if $\nabla_\psi s_\psi(z,\sigma)$ grows at most polynomially in $\|z\|$, locally uniformly in $\psi$, since $\heat{\sigma}\mu$ has Gaussian tails.
\end{proof}

\begin{lemma}[One-sample consistency is unbiased]\label{lem:consU}
Fix $(x,a)$, $\sigma>0$, a bootstrap source $\nu\in\calP(B_R)^{\calX\times\calA}$, $\mu\defeq\dbo\nu$, and a measurable field $f$. Draw $j=(r,x',a')\sim q(\cdot\mid x,a)$, then, given $j$, draw $z\sim\nu(x',a')$ and $\xi\sim\Norm(0,\Id)$ independently, and set
\begin{equation*}
  u\defeq z+\tfrac\sigma\gamma\,\xi,\qquad y\defeq r+\gamma u,\qquad T\defeq\gamma\inv f\big(u;x',a',\tfrac\sigma\gamma\big).
\end{equation*}
Then $y\sim\heat{\sigma}\mu(x,a)$ and $\Ex{T\mid y}=(\calS_\nu f)(y;x,a,\sigma)$ almost surely. If moreover $\Ex{\|T\|^2}<\infty$, then for every measurable $s$ with $\Ep{\heat{\sigma}\mu(x,a)}{\|s\|^2}<\infty$,
\begin{equation}\label{eq:consU}
  \Ex{\|s(y)-T\|^2}=\Ep{\heat{\sigma}\mu(x,a)}{\big\|s-(\calS_\nu f)(\cdot;x,a,\sigma)\big\|^2}+\Ex{\|T-\Ex{T\mid y}\|^2},
\end{equation}
where the last term does not depend on $s$ and vanishes when $q(\cdot\mid x,a)$ is a point mass.
\end{lemma}

\begin{proof}[Proof of \cref{lem:consU}]
Write $\nu_j\defeq\nu(x'_j,a'_j)$. Given $j$, $u\sim\heat{\sigma/\gamma}\nu_j$, so $y=r_j+\gamma u$ has the density $p_j$ of $(b_{r_j,\gamma})_\#\heat{\sigma/\gamma}\nu_j$, and $(j,y)$ has the joint law $p_j(y)\,q(\dd j\mid x,a)\,\dd y$.
By \cref{lem:commute} the $y$-marginal is $\heat{\sigma}\mu(x,a)$, and the law of $j$ given $y$ is the posterior branch law $q_\sigma(\dd j\mid y)$ of \cref{cor:scorebellman}.
Since $u=(y-r_j)/\gamma$, $T$ is a function of $(j,y)$, so $\Ex{T\mid y}=\int q_\sigma(\dd j\mid y)\,\gamma\inv f((y-r_j)/\gamma;x'_j,a'_j,\sigma/\gamma)=(\calS_\nu f)(y;x,a,\sigma)$, and \eqref{eq:consU} is the orthogonality of $T-\Ex{T\mid y}$ to the square-integrable functions of $y$.
If $q(\cdot\mid x,a)$ is a point mass, $T$ is a function of $y$ alone.
\end{proof}

Both lemmas obtain the posterior branch weights without evaluating a density, because $y$ is drawn through the sampled branch.
A target drawn independently of the branch would estimate the prior-weighted average of the branch scores, which is not the score of $\heat{\sigma}\mu(x,a)$ unless the posterior weights are constant.

\begin{remark}[Target variance of the two losses]\label{rem:consvar}
The one-sample DSM target $-\xi/\sigma$ at $y=y_0+\sigma\xi$, $y_0\sim\mu(x,a)$, has by Tweedie's formula $\Ex{\xi\mid y}=-\sigma\sm{\mu}{\sigma}(y)$ the normalized conditional variance $\sigma^2\,\Ex{\|\xi/\sigma-\Ex{\xi/\sigma\mid y}\|^2}=d-\sigma^2\calI(\heat{\sigma}\mu)$, with the Fisher information $\calI$ of \cref{def:ladderrisk}, which tends to $d$ as $\sigma\to0$ whenever $\calI(\mu)<\infty$.
The consistency target of \cref{lem:consU} with the exact field $f=\sm{\nu}{\cdot}$ has
\begin{equation*}
  \sigma^2\,\Ex{\|T-\Ex{T\mid y}\|^2}\leq\sigma^2\,\Ex{\|T\|^2}=\frac{\sigma^2}{\gamma^2}\int\calI\big(\heat{\sigma/\gamma}\nu(x'_j,a'_j)\big)\,q(\dd j\mid x,a)\leq\frac{\sigma^2}{\gamma^2}\sup_j\calI(\nu(x'_j,a'_j)),
\end{equation*}
which is $\calO(\sigma^2)$ when the successor laws have finite Fisher information; both bounds use $\calI(\heat{\tau}\rho)\leq\calI(\rho)$, i.e., Jensen's inequality for the smoothed score as the posterior mean of the unsmoothed one.
So the two targets stand to each other as a temporal-difference and a Monte Carlo target: at the lowest noise levels the consistency target has a far smaller variance, and its price is the accumulation of residuals in \cref{cor:consTD}.
\end{remark}

\subsection{Score-Space Temporal Differences}\label{proof-scoreTD}

This subsection states \cref{thm:scoreTDinf} in full and proves it.
For a measurable field $f$ and a law $\rho$ we write $\calE_\sigma(f;\rho)\defeq\Ep{\heat{\sigma}\rho}{\|f(\cdot,\sigma)-\sm{\rho}{\sigma}\|^2}$, as in \cref{def:ladderrisk}, and $\bar\calE_\sigma(f;\nu)\defeq\sup_{x',a'}\calE_\sigma(f(\cdot;x',a',\sigma);\nu(x',a'))$.

\begin{theorem}[Score-space TD is non-expansive in the multi-level risk]\label{thm:scoreTD}
Let $\nu\in\calP(B_R)^{\calX\times\calA}$ be a bootstrap source, $\mu\defeq\dbo\nu$, and $\calS_\nu$ the score Bellman image of \cref{def:scoreimage}.
\emph{(i) One level.} For every $(x,a)$ and $\sigma>0$,
\begin{equation}\label{eq:scoreTD-level}
  \calE_\sigma\big(\calS_\nu f;\mu(x,a)\big)\leq\gamma^{-2}\int\calE_{\sigma/\gamma}\big(f;\nu(x'_j,a'_j)\big)\,q(\dd j\mid x,a)\leq\gamma^{-2}\,\bar\calE_{\sigma/\gamma}(f;\nu),
\end{equation}
with equality in the first inequality whenever $q(\cdot\mid x,a)$ is a point mass.
\emph{(ii) All levels.} On the geometric grid $\sigma_{n+1}=\kappa_{\mathrm g}\sigma_n$ with $\kappa_{\mathrm g}=\gamma^{1/m}$, extended upward to the levels $\sigma_{-m}>\dots>\sigma_{-1}>\sigma_0$, one has $h_{n-m}=h_n/\gamma^2$ and
\begin{equation}\label{eq:scoreTD-ladder}
  \sum_{n=0}^{N_{\mathrm g}-1}h_n\,\calE_{\sigma_n}\big(\calS_\nu f;\mu(x,a)\big)\leq\sum_{n=-m}^{N_{\mathrm g}-1-m}h_n\,\bar\calE_{\sigma_n}(f;\nu).
\end{equation}
\emph{(iii) Boundary.} If $f(z;x',a',\sigma)=-(z-c)/\sigma^2$ with $c$ the center of $B_R$, then $\sigma^2\calE_\sigma(f;\rho)\leq D^2/(4\sigma^2)$ for every $\rho$ supported in $B_R$ and every $\sigma>0$.
\end{theorem}

\begin{proof}[Proof of \cref{thm:scoreTD}]
Part (i). Fix $(x,a)$ and $\sigma$, and write $\nu_j\defeq\nu(x'_j,a'_j)$ and $p_j$ for the density of $(b_{r_j,\gamma})_\#\heat{\sigma/\gamma}\nu_j$, so that $\heat{\sigma}\mu(x,a)$ has the density $p=\int p_j\,q(\dd j\mid x,a)$ (\cref{lem:commute}), $q_\sigma(\dd j\mid y)=p_j(y)\,q(\dd j\mid x,a)/p(y)$, and $\sm{\mu}{\sigma}=\int q_\sigma(\dd j\mid\cdot)\nabla\log p_j$ (\cref{cor:scorebellman}).
With $e_j(y)\defeq\gamma\inv\big(f-\sm{\nu_j}{\sigma/\gamma}\big)\big((y-r_j)/\gamma;x'_j,a'_j,\sigma/\gamma\big)$ we have $(\calS_\nu f)(y)-\sm{\mu}{\sigma}(y)=\int q_\sigma(\dd j\mid y)\,e_j(y)$, and Jensen's inequality for the kernel $q_\sigma(\cdot\mid y)$ gives
\begin{equation*}
  \calE_\sigma(\calS_\nu f;\mu(x,a))\leq\int\Big(\int\|e_j(y)\|^2p_j(y)\,\dd y\Big)q(\dd j\mid x,a)=\gamma^{-2}\int\calE_{\sigma/\gamma}(f;\nu_j)\,q(\dd j\mid x,a),
\end{equation*}
since $p_j(y)\,\dd y$ is the law of $r_j+\gamma u$ for $u\sim\heat{\sigma/\gamma}\nu_j$.
The second inequality of \eqref{eq:scoreTD-level} is the supremum over $(x',a')$, and for a point mass $q(\cdot\mid x,a)$ Jensen's inequality holds with equality.

Part (ii). With $\kappa_{\mathrm g}=\gamma^{1/m}$, $\sigma_{n-m}=\sigma_n/\gamma$ and $h_{n-m}=(1-\kappa_{\mathrm g}^2)\sigma_{n-m}^2=h_n/\gamma^2$, so \eqref{eq:scoreTD-level} at $\sigma=\sigma_n$ times $h_n$ gives $h_n\calE_{\sigma_n}(\calS_\nu f;\mu(x,a))\leq h_{n-m}\bar\calE_{\sigma_{n-m}}(f;\nu)$; sum over $n$.

Part (iii). By \eqref{eq:posterior-score}, $f(z)-\sm{\rho}{\sigma}(z)=(m_\sigma(z)-c)/\sigma^2$ with the posterior mean $m_\sigma(z)$ in the convex hull of $\supp\rho\subseteq B_R$, a box of circumradius $D/2$ centered at $c$, so $\|f(z)-\sm{\rho}{\sigma}(z)\|\leq D/(2\sigma^2)$ for every $z$.
\end{proof}
In the normalized error $\sigma\calE_\sigma^{1/2}$, the unit in which scores scale, the operator is an isometry with one branch and a contraction with several, and the level weights $h_n\propto\sigma_n^2$ are the squares of this normalization.
This is the sense of \cref{thm:scoreTDinf}(i): the multi-level risk of the image is at most that of the input on the levels shifted up by $m$.

\begin{corollary}[The consistency residual bounds the score error]\label{cor:consTD}
In the setting of \cref{thm:scoreTD}, let $s_\psi$ be the online and $s_{\bar\psi}$ the target score model, and let
\begin{equation*}
  \calR_n(x,a)\defeq\Ep{\heat{\sigma_n}\mu(x,a)}{\big\|s_\psi(\cdot;x,a,\sigma_n)-(\calS_\nu s_{\bar\psi})(\cdot;x,a,\sigma_n)\big\|^2}
\end{equation*}
be the consistency residual at level $n$. With the normalized quantities $\tilde e_n(\psi;\rho)\defeq\sigma_n\bar\calE_{\sigma_n}(s_\psi;\rho)^{1/2}$ and $\tilde r_n\defeq\sigma_n\sup_{x,a}\calR_n(x,a)^{1/2}$,
\begin{equation}\label{eq:consTD}
  \tilde e_n(\psi;\mu)\leq\tilde r_n+\tilde e_{n-m}(\bar\psi;\nu)\qquad\text{for all }n\geq0 .
\end{equation}
At a fixed point of training, $\bar\psi=\psi$ and $\nu=\mu=\eta^\pi$, and with the prior score of \cref{thm:scoreTD}(iii) above $\sigma_0$, iterating \eqref{eq:consTD} gives, with $L_n\defeq\lfloor n/m\rfloor$,
\begin{equation*}
  \tilde e_n(\psi;\eta^\pi)\leq\sum_{i=0}^{L_n}\tilde r_{n-im}+\frac{D}{2\sigma_0}.
\end{equation*}
So exact consistency at every level and every state--action pair determines the score field of $\eta^\pi$ up to the normalized error $D/(2\sigma_0)$ at every noise level.
\end{corollary}

\begin{proof}[Proof of \cref{cor:consTD}]
The triangle inequality in $L^2(\heat{\sigma_n}\mu(x,a))$ and \cref{thm:scoreTD}(i) with $f=s_{\bar\psi}$ give $\calE_{\sigma_n}(s_\psi;\mu(x,a))^{1/2}\leq\calR_n(x,a)^{1/2}+\gamma\inv\bar\calE_{\sigma_{n-m}}(s_{\bar\psi};\nu)^{1/2}$, as $\sigma_n/\gamma=\sigma_{n-m}$; the supremum over $(x,a)$ times $\sigma_n$ is \eqref{eq:consTD}.
At the fixed point, iterating \eqref{eq:consTD} $L_n+1$ times reaches the index $n-(L_n+1)m<0$, a level $\sigma\in(\sigma_0,\sigma_0/\gamma]$ where \cref{thm:scoreTD}(iii) bounds the normalized error of the prior score by $D/(2\sigma)\leq D/(2\sigma_0)$.
\end{proof}
By \cref{lem:consU} with $f=s_{\bar\psi}$, $\calR_n(x,a)$ is the population consistency loss \eqref{eq:cons} at level $n$ minus a conditional variance that does not depend on $\psi$, and equals it for point-mass branch laws.
So the gradient of the one-sample consistency loss is an unbiased estimate of the gradient of $\calR_n$, the residuals need no ground truth and are bounded by the training objective, and \cref{thm:scoreTDinf}(ii) is the iterated bound of \cref{cor:consTD}.

\begin{corollary}[Score-space TD converges in finitely many sweeps]\label{cor:scoreTDconv}
Let the branch laws be point masses, so that $\calS\defeq\calS_\nu$ does not depend on $\nu$. Let $f_0$ be any field on the grid that equals the prior score of \cref{thm:scoreTD}(iii) above $\sigma_0$, and let $f_{k+1}\defeq\calS f_k$ on the grid, the field above $\sigma_0$ being held fixed, with the normalized per-sweep residuals $\tilde r^{(k)}_n\defeq\sigma_n\sup_{x,a}\|f_k(\cdot;x,a,\sigma_n)-(\calS f_{k-1})(\cdot;x,a,\sigma_n)\|_{L^2(\heat{\sigma_n}\eta^\pi(x,a))}$ of an inexact sweep ($\tilde r\equiv0$ for the exact iteration). Then for every level $n$ and every $k\geq L_n+1$, with $L_n=\lfloor n/m\rfloor$,
\begin{equation*}
  \tilde e_n(f_k;\eta^\pi)\leq\sum_{i=0}^{L_n}\tilde r^{(k-i)}_{n-im}+\frac D{2\sigma_0},
\end{equation*}
whatever $f_0$ is on the grid. So the exact iteration reaches all levels within $\lceil N_{\mathrm g}/m\rceil=\lceil\log(\sigma_0/\sigma_{N_{\mathrm g}})/\log(1/\gamma)\rceil$ sweeps, with a normalized error of at most $D/(2\sigma_0)$, and the residual of an inexact sweep is charged at most $L_n+1$ times, never accumulated over sweeps.
\end{corollary}

\begin{proof}[Proof of \cref{cor:scoreTDconv}]
With point-mass branch laws the proof of \eqref{eq:consTD} uses only \cref{thm:scoreTD}(i), so it applies to the pair $(f_k,f_{k-1})$ with $\nu=\mu=\eta^\pi$ and gives $\tilde e_n(f_k)\leq\tilde r^{(k)}_n+\tilde e_{n-m}(f_{k-1})$.
Unrolling it $L_n+1$ times reaches a level above $\sigma_0$, where every $f_j$ is the prior score, whose normalized error is at most $D/(2\sigma_0)$ by \cref{thm:scoreTD}(iii).
\end{proof}
For stochastic branch laws $\calS_\nu$ moves with the source $\nu$, and \cref{cor:consTD} at a fixed point replaces the finite-sweep statement.
The statement needs a lagged target: with the online network as its own target, every update moves the target through the parameters shared across noise levels, which breaks the triangular order of the sweeps, and consistency-only training then diverges on a self-looping state--action pair (\cref{app:res-theory}).

\subsection{The Sampler}\label{proof-gen}
\begin{remark}[Why one noise level does not suffice]\label{rem:laddernecessity}
No isoperimetric constant enters because the sampler is compared with the reference process whose marginals $\heat{\sigma_n}\mu$ are exactly the measures under which DSM certifies $s_\psi$. At a single level the DSM risk is blind to mode masses: for two modes at separation $\Delta\geq2\sigma$ whose weights are perturbed by $\epsilon$, the risk at level $\sigma$ is $\calO(\epsilon^2e^{-\Delta^2/(8\sigma^2)})$ (\cref{prop:dsmblind}(ii)). By \eqref{eq:gen-kl} and Pinsker's inequality, a weight error in the sampler law must instead show up at the levels $\sigma\gtrsim\Delta$ (\cref{app:res-theory} checks this on trained networks). The grid must start above the spread of the target, $\sigma_0^2\gg m_2(\mu;z_c)$, and $\veps_{\mathrm{disc}}$ vanishes as $\kappa_{\mathrm g}\to1$.
\end{remark}

In this subsection $\mu$ is supported in $B_R\subseteq B(0,R_Z)$, $X_0\sim\mu$, and $(W_u)_{u\geq0}$ is a standard Brownian motion on $\dsR^d$ independent of $X_0$.
For $\sigma>0$ we write $X^{(\sigma)}\defeq X_0+W_{\sigma^2}$, so that $X^{(\sigma)}\sim\heat{\sigma}\mu$.

\begin{lemma}[The score along the heat flow is a reverse martingale]\label{lem:scoremart}
For $\sigma>0$ let $\calG_\sigma$ be the $\sigma$-algebra generated by $\{X^{(\sigma')}\colon\sigma'\geq\sigma\}$. Then, for all $0<\sigma\leq\sigma'$,
\begin{equation*}
  \Ex{\sm{\mu}{\sigma}\big(X^{(\sigma)}\big)\,\Big|\,\calG_{\sigma'}}=\sm{\mu}{\sigma'}\big(X^{(\sigma')}\big)
  \qquad\text{and}\qquad
  \Ex{\big\|\sm{\mu}{\sigma}\big(X^{(\sigma)}\big)-\sm{\mu}{\sigma'}\big(X^{(\sigma')}\big)\big\|^2}=\calI(\heat{\sigma}\mu)-\calI(\heat{\sigma'}\mu).
\end{equation*}
The map $\sigma\mapsto\calI(\heat{\sigma}\mu)$ is nonincreasing, and $\calI(\heat{\sigma}\mu)\leq d/\sigma^2$.
\end{lemma}

\begin{proof}[Proof of \cref{lem:scoremart}]
By the posterior form \eqref{eq:posterior-score},
\begin{equation*}
  \sm{\mu}{\sigma}\big(X^{(\sigma)}\big)=-\frac{X^{(\sigma)}-\Ex{X_0\mid X^{(\sigma)}}}{\sigma^2}=-\frac{\Ex{W_{\sigma^2}\mid X^{(\sigma)}}}{\sigma^2}.
\end{equation*}
Jensen's inequality gives $\Ex{\|\sm{\mu}{\sigma}(X^{(\sigma)})\|^2}\leq\Ex{\|W_{\sigma^2}\|^2}/\sigma^4=d/\sigma^2$, so $\calI(\heat{\sigma}\mu)\leq d/\sigma^2$ and all second moments below are finite.
The $\sigma$-algebra $\calG_\sigma$ is generated by $X^{(\sigma)}$ and the increments $(W_u-W_{\sigma^2})_{u\geq\sigma^2}$, which are independent of $(X_0,W_{\sigma^2})$, so $\Ex{W_{\sigma^2}\mid X^{(\sigma)}}=\Ex{W_{\sigma^2}\mid\calG_\sigma}$.

Now fix $\sigma<\sigma'$ and let $\calH$ be the $\sigma$-algebra generated by $X_0$, $W_{\sigma'^2}$ and $(W_u-W_{\sigma'^2})_{u\geq\sigma'^2}$, which contains $\calG_{\sigma'}$.
The pair $(W_{\sigma^2},W_{\sigma'^2})$ is independent of $X_0$ and of the later increments, so the Brownian bridge gives $\Ex{W_{\sigma^2}\mid\calH}=\Ex{W_{\sigma^2}\mid W_{\sigma'^2}}=(\sigma^2/\sigma'^2)\,W_{\sigma'^2}$.
We have
\begin{align*}
  \Ex{\sm{\mu}{\sigma}\big(X^{(\sigma)}\big)\,\Big|\,\calG_{\sigma'}}
  &=-\frac{\Ex{W_{\sigma^2}\mid\calG_{\sigma'}}}{\sigma^2} &\prstep{(tower property, $\calG_{\sigma'}\subseteq\calG_\sigma$)}\\
  &=-\frac{\Ex{\Ex{W_{\sigma^2}\mid\calH}\mid\calG_{\sigma'}}}{\sigma^2} &\prstep{(tower property, $\calG_{\sigma'}\subseteq\calH$)}\\
  &=-\frac{\Ex{W_{\sigma'^2}\mid\calG_{\sigma'}}}{\sigma'^2}=\sm{\mu}{\sigma'}\big(X^{(\sigma')}\big),
\end{align*}
proving the first claim.
So the increment $\sm{\mu}{\sigma}(X^{(\sigma)})-\sm{\mu}{\sigma'}(X^{(\sigma')})$ is orthogonal in $L^2(\dsP)$ to the $\calG_{\sigma'}$-measurable $\sm{\mu}{\sigma'}(X^{(\sigma')})$, and Pythagoras' identity with $\Ex{\|\sm{\mu}{\sigma}(X^{(\sigma)})\|^2}=\calI(\heat{\sigma}\mu)$ gives the second claim, which implies the monotonicity.
\end{proof}
The lemma is the heat-flow form of the stochastic-localization argument that \citet[Lemmas~1 to~4]{benton2024nearly} use for the Ornstein--Uhlenbeck process, where the same control of the discretization error is obtained by It\^o calculus.\footnote{They remark that their analysis extends to the variance-exploding process.}

\begin{proof}[Proof of \cref{thm:gen}]
\emph{Reverse process.} Let $T\defeq\sigma_0^2-\sigma_{N_{\mathrm g}}^2$, $\sigma(t)\defeq(\sigma_0^2-t)^{1/2}$ for $t\in[0,T]$, and $t_n\defeq\sigma_0^2-\sigma_n^2$, so that $\sigma(t_n)=\sigma_n$ and $t_{n+1}-t_n=h_n$.
The process $Y_t\defeq X^{(\sigma(t))}$ is the heat flow on the variance interval $[\sigma_{N_{\mathrm g}}^2,\sigma_0^2]$, a Brownian motion started from $\heat{\sigma_{N_{\mathrm g}}}\mu$, read backwards in time.
By the time-reversal theorem \citep{anderson1982reverse,haussmann1986time} it solves $\dd Y_t=\sm{\mu}{\sigma(t)}(Y_t)\,\dd t+\dd B_t$ with a Brownian motion $B$, $Y_0\sim\heat{\sigma_0}\mu$, and $Y_T\sim\heat{\sigma_{N_{\mathrm g}}}\mu$.
The sampler of \cref{def:ladderrisk} is the Euler--Maruyama scheme of this equation with the learned score frozen on the grid: the process $\dd\hat X_t=s_\psi(\hat X_{t_n},\sigma_n)\,\dd t+\dd B_t$ on $[t_n,t_{n+1})$, with $\hat X_0\sim\Norm(z_c,\sigma_0^2\Id)$, satisfies $\hat X_{t_{n+1}}=\hat X_{t_n}+h_ns_\psi(\hat X_{t_n},\sigma_n)+(B_{t_{n+1}}-B_{t_n})$, which is the recursion of \cref{def:ladderrisk}, so $\hat X_T\sim\hat Q_\psi$.

\emph{Girsanov and data processing.} Let $P^Y$ and $P^{\hat X}$ be the path laws of $Y$ and $\hat X$ on $[0,T]$, and let $P^{\hat X}_{\mathrm{ref}}$ be the path law of the scheme started from $\heat{\sigma_0}\mu$ instead of $\Norm(z_c,\sigma_0^2\Id)$.
The two scheme laws share their transition kernels, so the chain rule for the relative entropy gives
\begin{equation*}
  \KLdiv{P^Y}{P^{\hat X}}=\KLdiv{\heat{\sigma_0}\mu}{\Norm(z_c,\sigma_0^2\Id)}+\KLdiv{P^Y}{P^{\hat X}_{\mathrm{ref}}}.
\end{equation*}
For the second term, Girsanov's theorem for two diffusions with a common initial law, in the form obtained by the approximation argument of \citet[inequality~(5.5) and Section~5.2]{chen2023sampling},\footnote{The argument avoids Novikov's condition. For a diffusion coefficient $g$ instead of their $\sqrt2$ the bound is $\frac12\Ex{\int\|(b^Y_t-b^X_t)/g\|^2\dd t}$ for the drifts $b^Y,b^X$; here $g=1$.} gives
\begin{equation*}
  \KLdiv{P^Y}{P^{\hat X}_{\mathrm{ref}}}\leq\frac12\sum_{n=0}^{N_{\mathrm g}-1}\int_{t_n}^{t_{n+1}}\Ex{\big\|\sm{\mu}{\sigma(t)}(Y_t)-s_\psi(Y_{t_n},\sigma_n)\big\|^2}\dd t,
\end{equation*}
whenever the right side is finite, which the bounds below show.
By data processing for the time-$T$ marginals, $\KLdiv{\heat{\sigma_{N_{\mathrm g}}}\mu}{\hat Q_\psi}\leq\KLdiv{P^Y}{P^{\hat X}}$.

\emph{Initialization.} By joint convexity of the Kullback--Leibler divergence, $\KLdiv{\heat{\sigma_0}\mu}{\Norm(z_c,\sigma_0^2\Id)}\leq\int\KLdiv{\Norm(y,\sigma_0^2\Id)}{\Norm(z_c,\sigma_0^2\Id)}\,\mu(\dd y)=m_2(\mu;z_c)/(2\sigma_0^2)$.

\emph{Score and discretization errors.} For $t\in[t_n,t_{n+1})$ we have
\begin{align*}
  \frac12\Ex{\big\|\sm{\mu}{\sigma(t)}(Y_t)-s_\psi(Y_{t_n},\sigma_n)\big\|^2}
  &\leq\Ex{\big\|\sm{\mu}{\sigma_n}(Y_{t_n})-s_\psi(Y_{t_n},\sigma_n)\big\|^2}\\
  &\quad+\Ex{\big\|\sm{\mu}{\sigma(t)}(Y_t)-\sm{\mu}{\sigma_n}(Y_{t_n})\big\|^2} &\prstep{($\|a+b\|^2\leq2\|a\|^2+2\|b\|^2$)}\\
  &=\calE_{\sigma_n}(\psi;\mu)+\calI(\heat{\sigma(t)}\mu)-\calI(\heat{\sigma_n}\mu) &\prstep{(\cref{lem:scoremart})}\\
  &\leq\calE_{\sigma_n}(\psi;\mu)+\calI(\heat{\sigma_{n+1}}\mu)-\calI(\heat{\sigma_n}\mu).
\end{align*}
The equality uses $Y_{t_n}\sim\heat{\sigma_n}\mu$ for the first term and $(Y_{t_n},Y_t)=(X^{(\sigma_n)},X^{(\sigma(t))})$ with $\sigma(t)\leq\sigma_n$ for the second, and the last step is the monotonicity in \cref{lem:scoremart}, as $\sigma(t)\geq\sigma_{n+1}$.
Integrating over $[t_n,t_{n+1})$, an interval of length $h_n$, and summing over $n$ proves \eqref{eq:gen-kl}.

\emph{Geometric grid.} With $\sigma_{n+1}=\kappa_{\mathrm g}\sigma_n$ we have $h_n=(1-\kappa_{\mathrm g}^2)\sigma_n^2$, and summation by parts gives
\begin{equation*}
  \sum_{n=0}^{N_{\mathrm g}-1}\sigma_n^2\big(\calI(\heat{\sigma_{n+1}}\mu)-\calI(\heat{\sigma_n}\mu)\big)\leq\sigma_{N_{\mathrm g}-1}^2\,\calI(\heat{\sigma_{N_{\mathrm g}}}\mu)+\sum_{n=1}^{N_{\mathrm g}-1}(\sigma_{n-1}^2-\sigma_n^2)\,\calI(\heat{\sigma_n}\mu)\leq\frac{d}{\kappa_{\mathrm g}^2}\big(1+(N_{\mathrm g}-1)(1-\kappa_{\mathrm g}^2)\big),
\end{equation*}
by $\calI(\heat{\sigma}\mu)\leq d/\sigma^2$ and $\sigma_{n-1}^2=\sigma_n^2/\kappa_{\mathrm g}^2$.
As $1-e^{-2x}\leq2x$, $(N_{\mathrm g}-1)(1-\kappa_{\mathrm g}^2)\leq2N_{\mathrm g}\log(1/\kappa_{\mathrm g})=2\log(\sigma_0/\sigma_{N_{\mathrm g}})$, and multiplying by $1-\kappa_{\mathrm g}^2$ gives the bound on $\veps_{\mathrm{disc}}$.

\emph{Total variation and $\W1$.} Pinsker's inequality, $\tv\leq(\KL/2)^{1/2}$, gives the first part of \eqref{eq:gen-w1}.
The projection $\mathrm{proj}_{B_R}$ is $1$-Lipschitz, as $B_R$ is closed and convex, and it fixes $\mu$, so $\W1(Q_\psi,\mu)$ is at most $\W1\big((\mathrm{proj}_{B_R})_\#\hat Q_\psi,(\mathrm{proj}_{B_R})_\#\heat{\sigma_{N_{\mathrm g}}}\mu\big)+\W1\big(\heat{\sigma_{N_{\mathrm g}}}\mu,\mu\big)$.
Both laws in the first term are supported in $B_R$, of diameter $D$, so \cref{lem:cvarfacts}(iv) and data processing bound it by $D\,\tv(\hat Q_\psi,\heat{\sigma_{N_{\mathrm g}}}\mu)\leq D\veps_{\mathrm{TV}}$, and the coupling $(Y+\sigma_{N_{\mathrm g}}\veps,Y)$ bounds the second by $\sigma_{N_{\mathrm g}}\Ex{\|\veps\|}\leq\sigma_{N_{\mathrm g}}\sqrt d$.
\end{proof}

\begin{remark}[Score models of linear growth]\label{rem:unbounded}
The score network of \cref{app:method} is not bounded on $\dsR^d$: the prior term $-\|z\|^2/(2(\sigma^2+\sigma_{\mathrm{d}}^2))$ of its energy makes its score grow linearly in $z$.
\Cref{thm:gen} covers it through a clipped model.
For a radius $R_{\mathrm c}>0$ let $\tilde s_\psi(z,\sigma)\defeq s_\psi(\mathrm{proj}_{\bar B(z_c,R_{\mathrm c})}(z),\sigma)$, which is bounded on $\dsR^d$ at every grid level, as $s_\psi$ is continuous, and equals $s_\psi$ on the closed ball $\bar B(z_c,R_{\mathrm c})$.
Started from the same draw and driven by the same noise, the two samplers coincide on the event that $\hat X_0,\dots,\hat X_{N_{\mathrm g}-1}$ all lie in this ball, so the coupling inequality bounds the total variation between their output laws by the probability $\veps_{\mathrm{clip}}$ that one of these iterates leaves it, which the sampler's own paths estimate.
So \eqref{eq:gen-w1} holds for $s_\psi$ with $\veps_{\mathrm{TV}}$ computed for $\tilde s_\psi$ and $\veps_{\mathrm{clip}}$ added to it, and the multi-level risk of $\tilde s_\psi$ differs from that of $s_\psi$ only through the mass of $\heat{\sigma_n}\mu$ outside the ball.
\end{remark}
\begin{remark}[Estimating the multi-level risk]\label{rem:ladderestimator}
If the level index $n<N_{\mathrm g}$ is drawn with probability $p_n>0$ and $\ell_n(\psi)$ is the one-sample bootstrap DSM loss \eqref{eq:dsm} at level $\sigma_n$, then $\Ex{(h_n/p_n)\,\ell_n(\psi)}=\calE_{\mathrm{ml}}(\psi;\dbo\nu(x,a))+C$ with $C$ independent of $\psi$, by \cref{lem:dsmU} at every level.
On a geometric grid $h_n\propto\sigma_n^2$, so uniform draws call for the weight $\sigma_n^2$, the likelihood weighting of score-based diffusion models.
Our implementation does not use this estimator: it draws one noise level per sample and weights every sample equally (\cref{app:method}), so its training loss is not an unbiased estimate of $\calE_{\mathrm{ml}}$, and \cref{app:res-theory} measures the multi-level risk and the bound of \cref{thm:gen} on the trained networks instead.
\end{remark}

The consistency residuals of \cref{cor:consTD} turn \cref{thm:gen} into a certificate that needs no ground truth at a fixed point of training.
\begin{corollary}[Score Bellman residual certificate]\label{cor:residualcert}
In the setting of \cref{cor:consTD}, at a fixed point of training with the prior score above $\sigma_0$, let $\tilde e^*_n\defeq\sum_{i=0}^{L_n}\tilde r_{n-im}+D/(2\sigma_0)$. Then for every $(x,a)$,
\begin{equation*}
  \calE_{\mathrm{ml}}\big(\psi;\eta^\pi(x,a)\big)\leq(1-\kappa_{\mathrm g}^2)\sum_{n=0}^{N_{\mathrm g}-1}(\tilde e^*_n)^2,
\end{equation*}
and \cref{thm:gen} with $\mu=\eta^\pi(x,a)$ and this bound on the multi-level risk in $\veps_{\mathrm{TV}}$ gives $\W1(Q_\psi(x,a),\eta^\pi(x,a))\leq D\veps_{\mathrm{TV}}+\sigma_{N_{\mathrm g}}\sqrt d$. A field with vanishing consistency residuals satisfies $\calE_{\mathrm{ml}}(\psi;\eta^\pi(x,a))\leq\log(\sigma_0/\sigma_{N_{\mathrm g}})\,D^2/(2\sigma_0^2)$.
\end{corollary}

\begin{proof}[Proof of \cref{cor:residualcert}]
For each $n$, $h_n\calE_{\sigma_n}(\psi;\eta^\pi(x,a))\leq(1-\kappa_{\mathrm g}^2)\sigma_n^2\bar\calE_{\sigma_n}(s_\psi;\eta^\pi)=(1-\kappa_{\mathrm g}^2)\tilde e_n(\psi;\eta^\pi)^2\leq(1-\kappa_{\mathrm g}^2)(\tilde e^*_n)^2$ by \cref{cor:consTD}, and summing over $n$ gives the first claim.
The sampler bound is \cref{thm:gen} for $\mu=\eta^\pi(x,a)$, which is supported in $B_R$ by \cref{asm:mdp}.
With vanishing residuals every $\tilde e^*_n$ equals $D/(2\sigma_0)$, and $(1-\kappa_{\mathrm g}^2)N_{\mathrm g}\leq2N_{\mathrm g}\log(1/\kappa_{\mathrm g})=2\log(\sigma_0/\sigma_{N_{\mathrm g}})$ gives the last claim.
\end{proof}
The residuals $\tilde r_n$ are bounded by the population consistency loss (\cref{lem:consU}), so every term of the certificate is computed from the training objective and the bootstrap targets.

The next proposition is the witness behind \cref{rem:laddernecessity}.
It uses the following functional inequalities.
\begin{definition}[Log-Sobolev and Poincar\'e inequalities]\label{def:funcineq}
Let $\rho\in\calP(\dsR^d)$ and $C>0$. The law $\rho$ satisfies the \emph{logarithmic Sobolev inequality} $\mathrm{LSI}(C)$ if ${\KLdiv{\nu}{\rho}\leq\frac C2\Ep{\nu}{\|\nabla\log(\dd\nu/\dd\rho)\|^2}}$ for every $\nu\ll\rho$ with a strictly positive, locally Lipschitz density $\dd\nu/\dd\rho$, and the \emph{Poincar\'e inequality} $\mathrm{P}(C)$ if ${\Vrp{\rho}{f}\leq C\Ep{\rho}{\|\nabla f\|^2}}$ for every locally Lipschitz $f\in L^2(\rho)$. We write $C_{\mathrm{LSI}}$ and $C_P$ for the smallest admissible constants; $\Norm(0,\Id)$ satisfies both with $C=1$.
\end{definition}

\begin{proposition}[Single-level DSM is blind to mode masses]\label{prop:dsmblind}
Let $\sigma>0$ and $\Delta\geq2\sigma$, let $e_1$ be the first unit vector of $\dsR^d$, and let $\mu\defeq\frac12(\delta_{-\Delta e_1/2}+\delta_{\Delta e_1/2})$.
\emph{(i)} The log-Sobolev and Poincar\'e constants of $\heat{\sigma}\mu$ satisfy
\begin{equation}\label{eq:lsi-lower}
  C_{\mathrm{LSI}}(\sigma)\geq C_P(\sigma)\geq\frac{\sqrt{2\pi}}{4}\,\sigma^2\,e^{(\Delta-2\sigma)^2/(8\sigma^2)}.
\end{equation}
\emph{(ii)} For $0<\epsilon<\frac12$ let $\mu_\epsilon\defeq(\frac12+\epsilon)\delta_{-\Delta e_1/2}+(\frac12-\epsilon)\delta_{\Delta e_1/2}$ be the same law with perturbed weights, and let the score model be its smoothed score, $s_\psi(\cdot,\sigma)\defeq\sm{\mu_\epsilon}{\sigma}$. Then $\W1(\mu_\epsilon,\mu)=\W1(\heat{\sigma}\mu_\epsilon,\heat{\sigma}\mu)=\epsilon\Delta$, whereas
\begin{equation*}
  \calE_\sigma(\psi;\mu)\leq\frac{32\,\Delta\,\epsilon^2}{\sqrt{2\pi}\,\sigma^3(1-2\epsilon)^2}\,e^{-\Delta^2/(8\sigma^2)}.
\end{equation*}
\end{proposition}

\begin{proof}[Proof of \cref{prop:dsmblind}]
Under $\heat{\sigma}\mu$ the first coordinate has the law $\frac12\Norm(-\Delta/2,\sigma^2)+\frac12\Norm(\Delta/2,\sigma^2)$, with density $p^{(1)}_\sigma$, and the other coordinates are independent $\Norm(0,\sigma^2)$.

Part (i). Let $f(z)\defeq\min\{1,\max\{-1,z_1/\sigma\}\}$, which is bounded and Lipschitz.
We have $\Ep{\heat{\sigma}\mu}{f}=0$ by symmetry, $f^2=1$ on $\{|z_1|\geq\sigma\}$, and $\|\nabla f\|^2=\sigma^{-2}\ind\{|z_1|<\sigma\}$ almost everywhere.
On either component $z_1=\pm\Delta/2+\sigma G$ with $G\sim\Norm(0,1)$, so the event $\{|z_1|\geq\sigma\}$ has probability at least $\frac12$, as $\Delta\geq2\sigma$.
For $|u|\leq\sigma\leq\Delta/2$ both Gaussian terms of $p^{(1)}_\sigma(u)$ are at most $\frac{1}{2\sqrt{2\pi}\sigma}e^{-(\Delta/2-\sigma)^2/(2\sigma^2)}$.
So
\begin{align*}
  \Vrp{\heat{\sigma}\mu}{f}&\geq(\heat{\sigma}\mu)(\{|z_1|\geq\sigma\})\geq\frac12,\\
  \Ep{\heat{\sigma}\mu}{\|\nabla f\|^2}&=\sigma^{-2}\int_{-\sigma}^{\sigma}p^{(1)}_\sigma(u)\,\dd u\leq\frac{2}{\sqrt{2\pi}\,\sigma^2}\,e^{-(\Delta-2\sigma)^2/(8\sigma^2)},
\end{align*}
and $\mathrm{P}(C_P(\sigma))$ applied to $f$ gives the second inequality of \eqref{eq:lsi-lower}.
For the first, $\mathrm{LSI}(C)$ implies $\mathrm{P}(C)$: for bounded locally Lipschitz $f$ with $\Ep{\rho}{f}=0$, take $\dd\nu/\dd\rho\propto(1+\veps f)^2$; then $\KLdiv{\nu}{\rho}=2\veps^2\Vrp{\rho}{f}+o(\veps^2)$ and $\Ep{\nu}{\|\nabla\log(\dd\nu/\dd\rho)\|^2}=4\veps^2\Ep{\rho}{\|\nabla f\|^2}+o(\veps^2)$, and letting $\veps\to0$ in $\mathrm{LSI}(C)$ and truncating gives $\mathrm{P}(C)$.

Part (ii). Write $\pi_1,\pi_2$ for the posterior responsibilities of the two components under $\heat{\sigma}\mu$, i.e., $\pi_1(z)=(1+e^{\Delta z_1/\sigma^2})\inv=1-\pi_2(z)$.
The density of $\heat{\sigma}\mu_\epsilon$ is that of $\heat{\sigma}\mu$ times $1+2\epsilon(\pi_1-\pi_2)$, and $\nabla(\pi_1-\pi_2)=-(2\Delta/\sigma^2)\pi_1\pi_2\,e_1$, so
\begin{equation*}
  \sm{\mu_\epsilon}{\sigma}-\sm{\mu}{\sigma}=\nabla\log\big(1+2\epsilon(\pi_1-\pi_2)\big)=-\frac{4\epsilon\Delta}{\sigma^2}\,\frac{\pi_1\pi_2}{1+2\epsilon(\pi_1-\pi_2)}\,e_1.
\end{equation*}
With $1+2\epsilon(\pi_1-\pi_2)\geq1-2\epsilon$ and $\pi_1\pi_2\leq e^{-\Delta|z_1|/\sigma^2}$ we get
\begin{equation*}
  \calE_\sigma(\psi;\mu)\leq\frac{16\epsilon^2\Delta^2}{\sigma^4(1-2\epsilon)^2}\int_{\dsR}e^{-2\Delta|u|/\sigma^2}\,p^{(1)}_\sigma(u)\,\dd u.
\end{equation*}
Completing the square in each of the four half-line Gaussian integrals and using $\int_x^\infty e^{-t^2/2}\,\dd t\leq e^{-x^2/2}/x$, the four terms carry the factors $\frac13$, $\frac15$, $\frac13$ and $\frac15$, and the integral is at most $\frac{16}{15}\frac{\sigma}{\sqrt{2\pi}\Delta}e^{-\Delta^2/(8\sigma^2)}\leq\frac{2\sigma}{\sqrt{2\pi}\Delta}e^{-\Delta^2/(8\sigma^2)}$, which gives the bound on $\calE_\sigma(\psi;\mu)$.
For the distances, the coupling that moves the mass $\epsilon$ from $\Delta e_1/2$ to $-\Delta e_1/2$ gives $\W1(\mu_\epsilon,\mu)\leq\epsilon\Delta$, and adding the same Gaussian noise to both of its coordinates gives $\W1(\heat{\sigma}\mu_\epsilon,\heat{\sigma}\mu)\leq\epsilon\Delta$.
Conversely, the $1$-Lipschitz function $z\mapsto z_1$ has the mean $-\epsilon\Delta$ under $\mu_\epsilon$ and under $\heat{\sigma}\mu_\epsilon$, and the mean $0$ under $\mu$ and under $\heat{\sigma}\mu$, so both distances are at least $\epsilon\Delta$ by the Kantorovich--Rubinstein duality.
\end{proof}
So on this family $\W1(\heat{\sigma}\mu_\epsilon,\heat{\sigma}\mu)/\calE_\sigma(\psi;\mu)^{1/2}$ grows at least like $e^{\Delta^2/(16\sigma^2)}$, and a bound of $\W1$ by a function of the risk at the single level $\sigma$ that holds for all such models carries a factor of this order.
By (i), the Poincar\'e constant, which is the relaxation time of the Langevin diffusion toward $\heat{\sigma}\mu$, grows exponentially in $(\Delta/\sigma)^2$ as well.

\begin{corollary}[Mode masses of the sampler]\label{cor:modemass}
  In the setting of \cref{thm:gen}, let $\mu=\sum_jw_j\mu_j$ and let $\{A_j\}$ be a Borel partition of $\dsR^d$ with ${\mathrm{dist}(\supp\mu_j,A_j^c)\geq m_{\mathrm{cell}}>0}$ for every $j$. If $\sigma_{N_{\mathrm g}}\leq m_{\mathrm{cell}}/\sqrt d$, then
\begin{equation*}
  \sum_j\big|\hat Q_\psi(A_j)-w_j\big|\leq2\,\veps_{\mathrm{TV}}+2\,e^{-(m_{\mathrm{cell}}/\sigma_{N_{\mathrm g}}-\sqrt d)^2/2}.
\end{equation*}
\end{corollary}

\begin{proof}[Proof of \cref{cor:modemass}]
For every Borel partition, $\sum_j|P_1(A_j)-P_2(A_j)|\leq2\,\tv(P_1,P_2)$, which with \cref{thm:gen} gives $\sum_j|\hat Q_\psi(A_j)-(\heat{\sigma_{N_{\mathrm g}}}\mu)(A_j)|\leq2\veps_{\mathrm{TV}}$.
The leakage satisfies $(\heat{\sigma_{N_{\mathrm g}}}\mu)(A_j)-w_j=\sum_{i\neq j}w_i(\heat{\sigma_{N_{\mathrm g}}}\mu_i)(A_j)-w_j(\heat{\sigma_{N_{\mathrm g}}}\mu_j)(A_j^c)$, so $\sum_j|(\heat{\sigma_{N_{\mathrm g}}}\mu)(A_j)-w_j|\leq2\sum_iw_i(\heat{\sigma_{N_{\mathrm g}}}\mu_i)(A_i^c)$.
For $Y\sim\mu_i$ and $\veps\sim\Norm(0,\Id)$, the event $Y+\sigma_{N_{\mathrm g}}\veps\notin A_i$ forces $\sigma_{N_{\mathrm g}}\|\veps\|\geq m_{\mathrm{cell}}$, and \eqref{eq:gauss-tail} gives $\Pr(\|\veps\|\geq m_{\mathrm{cell}}/\sigma_{N_{\mathrm g}})\leq e^{-(m_{\mathrm{cell}}/\sigma_{N_{\mathrm g}}-\sqrt d)^2/2}$ for $m_{\mathrm{cell}}/\sigma_{N_{\mathrm g}}\geq\sqrt d$.
The triangle inequality completes the proof.
\end{proof}

\subsection{Policy Evaluation}\label{proof-e2e}

\begin{proof}[Proof of \cref{thm:e2e}]
Fix $k$ and $(x,a)$.
By \cref{asm:mdp}, $\dbo$ maps $\calP(B_R)^{\calX\times\calA}$ into itself, so $\mu_k(x,a)$ is supported in $B_R$, and \cref{thm:gen} gives $\W1(Q_{k+1}(x,a),\mu_k(x,a))\leq\veps_{\mathrm{gen}}$.
By \cref{thm:contraction} with $p=1$ and $\eta^\pi=\dbo\eta^\pi$, $\W1(\mu_k(x,a),\eta^\pi(x,a))\leq\gamma\,\Wbar1(\nu_k,\eta^\pi)$.

Part (i). For $\nu_k=Q_k$, which is supported in $B_R$, the triangle inequality gives $\W1(Q_{k+1}(x,a),\eta^\pi(x,a))\leq\veps_{\mathrm{gen}}+\gamma\,\Wbar1(Q_k,\eta^\pi)$.
Taking the supremum over $(x,a)$ and unrolling the recursion gives $\Wbar1(Q_K,\eta^\pi)\leq\gamma^K\Wbar1(Q_0,\eta^\pi)+\veps_{\mathrm{gen}}\sum_{i<K}\gamma^i$, which implies (i).

Part (ii). For $\nu_k=\hat\eta_k$, one more triangle inequality through $Q_{k+1}(x,a)$ gives $\W1(\hat\eta_{k+1}(x,a),\eta^\pi(x,a))\leq\veps_{\mathrm{St}}+\veps_{\mathrm{gen}}+\gamma\,\Wbar1(\hat\eta_k,\eta^\pi)$, which unrolls in the same way.
The bound on $Q_K$ is the first paragraph of this proof at $k=K-1$.
\end{proof}

\begin{remark}[The smoothing term of \cref{thm:e2e}]\label{rem:e2ereadings}
The term $\sigma_{N_{\mathrm g}}\sqrt d/(1-\gamma)$ inside $\veps_{\mathrm{gen}}/(1-\gamma)$ is not an artifact of the proof.
Sampler draws carry the smoothing variance $\sigma_{N_{\mathrm g}}^2$ per coordinate, and on a self-loop with a deterministic transition and an otherwise exact sampler the bootstrap compounds it to an excess of $\gamma^2\sigma_{N_{\mathrm g}}^2/(1-\gamma^2)$ per coordinate in the variance of the Bellman target at the fixed point, since the target variance $v$ then solves $v=\gamma^2(v+\sigma_{N_{\mathrm g}}^2)$.
\Cref{app:res-theory} measures this accumulation, and the affine moment matching of \cref{app:method}, followed by the projection onto $B_R$, removes it.
The moment-matched draws are a particle set supported in $B_R$, so \cref{thm:e2e}(ii) covers sampler bootstrapping with moment matching, with these draws in the role of the Stein particles and their measured distance to the sampler law as $\veps_{\mathrm{St}}$.
\end{remark}

\begin{corollary}[Supremal Wasserstein-$2$]\label{cor:e2eW2}
In the setting of \cref{thm:e2e}, with $\veps_{\mathrm{gen},2}\defeq D\sup_{k,(x,a)}\veps_{\mathrm{TV},k}(x,a)^{1/2}+\sigma_{N_{\mathrm g}}\sqrt d$, sampler bootstrapping satisfies
\begin{equation*}
  \Wbar2(Q_K,\eta^\pi)\leq\gamma^K\,\Wbar2(Q_0,\eta^\pi)+\frac{\veps_{\mathrm{gen},2}}{1-\gamma},
\end{equation*}
and bootstrapping from Stein particles the same bound with $\sup_{k,(x,a)}\W2(\hat\eta_{k+1}(x,a),Q_{k+1}(x,a))$ added to the numerator.
\end{corollary}

\begin{proof}[Proof of \cref{cor:e2eW2}]
\Cref{lem:cvarfacts}(iv) for $(\mathrm{proj}_{B_R})_\#\hat Q_\psi$ and $(\mathrm{proj}_{B_R})_\#\heat{\sigma_{N_{\mathrm g}}}\mu$, data processing for the projection, \cref{prop:smoothbias}(i) and \cref{thm:gen} give $\W2(Q_\psi,\mu)\leq D\veps_{\mathrm{TV}}^{1/2}+\sigma_{N_{\mathrm g}}\sqrt d$.
The rest is the proof of \cref{thm:e2e} with $\W1$ replaced by $\W2$, using \cref{thm:contraction} with $p=2$.
\end{proof}
The bound needs no dissipativity or mode-separation assumption, and the model error enters through its square root, which is the price of an argument without isoperimetry.

\begin{corollary}[Mode masses of the evaluated policy]\label{cor:modemasse2e}
In the setting of \cref{thm:e2e}, fix $(x,a)$, write $\eta^\pi(x,a)=\sum_jw_j\mu_j$, and let $\{A_j\}$ be a Borel partition of $\dsR^d$ with $\mathrm{dist}(\supp\mu_j,A_j^c)\geq m_{\mathrm{cell}}>0$ for every $j$. Let $Q$ be any law on $B_R$ and $\sigma_{\mathrm c}\leq m_{\mathrm{cell}}/\sqrt d$. Then the $\sigma_{\mathrm c}$-smoothed cell masses of $Q$ satisfy
\begin{equation*}
  \sum_j\big|(\heat{\sigma_{\mathrm c}}Q)(A_j)-w_j\big|\leq\frac{2}{\sigma_{\mathrm c}\sqrt{2\pi}}\,\W1\big(Q,\eta^\pi(x,a)\big)+2\,e^{-(m_{\mathrm{cell}}/\sigma_{\mathrm c}-\sqrt d)^2/2},
\end{equation*}
so with $Q=Q_K(x,a)$ the $\Wbar1$ bound of \cref{thm:e2e} certifies the mode masses of the evaluated policy at every $(x,a)$.
\end{corollary}

\begin{proof}[Proof of \cref{cor:modemasse2e}]
For $x,y\in\dsR^d$, $\tv\big(\Norm(x,\sigma_{\mathrm c}^2\Id),\Norm(y,\sigma_{\mathrm c}^2\Id)\big)=2\Phi\big(\|x-y\|/(2\sigma_{\mathrm c})\big)-1\leq\|x-y\|/(\sigma_{\mathrm c}\sqrt{2\pi})$, since the standard normal density is at most $(2\pi)^{-1/2}$.
Joint convexity of $\tv$ over the $\W1$-optimal coupling of $Q$ and $\eta^\pi(x,a)$ gives $\tv(\heat{\sigma_{\mathrm c}}Q,\heat{\sigma_{\mathrm c}}\eta^\pi(x,a))\leq\W1(Q,\eta^\pi(x,a))/(\sigma_{\mathrm c}\sqrt{2\pi})$.
Then $\sum_j|(\heat{\sigma_{\mathrm c}}Q)(A_j)-(\heat{\sigma_{\mathrm c}}\eta^\pi(x,a))(A_j)|\leq2\,\tv(\heat{\sigma_{\mathrm c}}Q,\heat{\sigma_{\mathrm c}}\eta^\pi(x,a))$, and the leakage bound of the proof of \cref{cor:modemass} at the scale $\sigma_{\mathrm c}$ completes the proof.
\end{proof}

\begin{corollary}[Matrix discounting]\label{cor:e2eGamma}
Let the Bellman target be $b_{r,\Gamma}(z)=r+\Gamma z$ with a matrix $\Gamma\in\dsR^{d\times d}$ of operator norm $\|\Gamma\|_{\mathrm{op}}<1$, and let \cref{asm:mdp} hold with $B_R$ replaced by a box that every $b_{r,\Gamma}$ maps into itself. Then \cref{thm:e2e,thm:cvarrev} hold verbatim with $\gamma$ replaced by $\|\Gamma\|_{\mathrm{op}}$.
\end{corollary}

\begin{proof}[Proof of \cref{cor:e2eGamma}]
The affine map $z\mapsto r+\Gamma z$ contracts Euclidean $\W1$ by the factor $\|\Gamma\|_{\mathrm{op}}$, and mixtures over transitions preserve the contraction, so \cref{thm:contraction} holds with this factor.
The proofs of \cref{thm:e2e,thm:cvarrev} use $\gamma$ only through \cref{thm:contraction}, and \cref{thm:gen} does not involve the discount.
\end{proof}
Under matrix discounting, sliced divergences contract through the max-slicing construction of \citet{debes2026multivariate}, whereas this bound needs no such construction, and \cref{cor:commuteGamma} gives the matching score-space identity.

\subsection{No Global Dissipativity}\label{proof-nodissip}

The Stein particles enter \cref{thm:e2e} only through $\veps_{\mathrm{St}}$, which we measure, because the contraction that an unconditional bound would need fails for SVGD.
Fix a level $\sigma>0$ and $\mu$ supported in $B_R$, and let $k(z,z')=\kappa(z-z')$ be a translation-invariant kernel with an even $\kappa\in C^1(\dsR^d)$ that is bounded with a bounded gradient.
We write
\begin{equation*}
  \Phi_\nu(z)\defeq\Ep{z'\sim\nu}{k(z',z)\,\sm{\mu}{\sigma}(z')+\nabla_{z'}k(z',z)},\qquad
  \dsF(z_{1:M})_i\defeq\frac1M\sum_{j=1}^M\big[k(z_j,z_i)\,\sm{\mu}{\sigma}(z_j)+\nabla_{z_j}k(z_j,z_i)\big]
\end{equation*}
for the mean-field and the joint $M$-particle exact-score SVGD drifts.
The drift $\Phi$ is \emph{globally dissipative} with constant $\lambda>0$ if $\Ex{\langle X-Y,\Phi_\nu(X)-\Phi_{\nu'}(Y)\rangle}\leq-\lambda\,\Ex{\|X-Y\|^2}$ for all $\nu,\nu'\in\calP_2(\dsR^d)$ and all couplings $(X,Y)$ of them, and $\dsF$ is globally dissipative if $\langle z-y,\dsF(z)-\dsF(y)\rangle\leq-\lambda|z-y|^2$ for all $z,y\in(\dsR^d)^M$, with $|\cdot|$ the Euclidean norm of $\dsR^{dM}$.
Either property would make the exact-score SVGD flow a contraction toward its target.

\begin{proposition}[No global dissipativity for the exact-score SVGD drift]\label{prop:nodissip}
Let $\mu$ be supported in $B_R$ and $\sigma>0$.
\emph{(i)} For every even kernel $\kappa\in C^1(\dsR^d)$ that is bounded with a bounded gradient, the mean-field drift $\Phi$ is not globally dissipative with any constant $\lambda>0$.
\emph{(ii)} For the RBF kernel $\kappa(u)=e^{-\|u\|^2/(2h^2)}$ and every $M\geq2$ with
\begin{equation*}
  M\geq1+\frac{h^2}{\sigma^2},
\end{equation*}
the joint drift $\dsF$ is not globally dissipative with any constant $\lambda>0$. With a bandwidth $h\leq C_h\sigma$ it therefore fails for every $M\geq1+C_h^2$.
\emph{(iii)} For every $M\geq2$ and every kernel as in (i), global dissipativity of $\dsF$ implies that every zero of $\dsF$ is a collapsed configuration $(c,\dots,c)$ with $\sm{\mu}{\sigma}(c)=0$, and that the $M$-particle exact-score SVGD flow converges exponentially to it from every initial cloud.
\end{proposition}

\begin{proof}[Proof of \cref{prop:nodissip}]
Part (i). The density $p_\sigma$ of $\heat{\sigma}\mu$ is smooth, positive and has Gaussian tails, so it has a maximizer $\bar z$, and $\sm{\mu}{\sigma}(\bar z)=0$.
For every $z$, the function $g(z')\defeq\kappa(z'-z)\,p_\sigma(z')$ is integrable and continuously differentiable with the integrable gradient $\nabla g=\kappa(\cdot-z)\,p_\sigma\,\sm{\mu}{\sigma}+\nabla\kappa(\cdot-z)\,p_\sigma$, as $\kappa$ and $\nabla\kappa$ are bounded and $\Ep{\heat{\sigma}\mu}{\|\sm{\mu}{\sigma}\|}\leq\calI(\heat{\sigma}\mu)^{1/2}<\infty$ (\cref{lem:scoremart}).
So $\Phi_{\heat{\sigma}\mu}(z)=\int\nabla g(z')\,\dd z'=0$, which is Stein's identity; the integral of the gradient of an integrable $C^1$ function with an integrable gradient vanishes, as one sees by integrating against $\chi(\cdot/R)$ for a smooth cutoff $\chi$ and letting $R\to\infty$. At the point mass, $\Phi_{\delta_{\bar z}}(\bar z)=\kappa(0)\,\sm{\mu}{\sigma}(\bar z)+\nabla\kappa(0)=0$, since $\nabla\kappa(0)=0$ for an even $C^1$ function.
For $\nu=\heat{\sigma}\mu$, $\nu'=\delta_{\bar z}$ and their only coupling $(X,\bar z)$, the left side of the dissipativity inequality is $\Ex{\langle X-\bar z,0\rangle}=0$, whereas the right side is $-\lambda\,\Ex{\|X-\bar z\|^2}<0$, as $\heat{\sigma}\mu$ is not a point mass.

Part (ii). With $\bar z$ as in part (i), the collapsed configuration $z^*\defeq(\bar z,\dots,\bar z)$ is a zero of $\dsF$, since $\dsF(z^*)_i=\kappa(0)\,\sm{\mu}{\sigma}(\bar z)+\nabla\kappa(0)=0$.
As $\dsF\in C^1$, dissipativity at the pairs $(z^*+tv,z^*)$ and $t\downarrow0$ give $\langle v,D\dsF(z^*)v\rangle\leq-\lambda|v|^2$ for every $v\in(\dsR^d)^M$.
Take $v=(\delta,0,\dots,0)$ with $\delta\in\dsR^d$.
In the derivative of $\dsF(z)_1=\frac1M\sum_j\big[\kappa(z_j-z_1)\,\sm{\mu}{\sigma}(z_j)+\nabla\kappa(z_j-z_1)\big]$ with respect to $z_1$ at $z^*$, the term $j=1$ contributes $\frac1M\kappa(0)\nabla\sm{\mu}{\sigma}(\bar z)$, and each term $j\neq1$ contributes $\frac1M\big[-\sm{\mu}{\sigma}(\bar z)\nabla\kappa(0)\T-\nabla^2\kappa(0)\big]=-\frac1M\nabla^2\kappa(0)$.
For the RBF kernel $\kappa(0)=1$ and $\nabla^2\kappa(0)=-h^{-2}\Id$, so
\begin{align*}
  \delta\T\Big(\frac{\partial\dsF_1}{\partial z_1}(z^*)\Big)\delta
  &=\frac1M\Big[\delta\T\nabla^2\log p_\sigma(\bar z)\,\delta+\frac{M-1}{h^2}\|\delta\|^2\Big]\\
  &\geq\frac1M\Big[\frac{M-1}{h^2}-\frac1{\sigma^2}\Big]\|\delta\|^2. &\prstep{(\eqref{eq:score-hessian} gives $-\nabla^2\log p_\sigma\preceq\sigma^{-2}\Id$)}
\end{align*}
The left side is at most $-\lambda\|\delta\|^2<0$, which forces $(M-1)/h^2<1/\sigma^2$, i.e., $M<1+h^2/\sigma^2$, proving (ii).

Part (iii). If $z^*\neq y^*$ were two zeros of $\dsF$, dissipativity at the pair would give $0\leq-\lambda|z^*-y^*|^2<0$, so $\dsF$ has at most one zero.
The zero exists, as $(\bar z,\dots,\bar z)$ is one.
The drift is permutation-equivariant, $\dsF(Pz)=P\dsF(z)$ for every permutation $P$ of the particles, so the unique zero satisfies $Pz^*=z^*$ for all $P$, i.e., $z^*=(c,\dots,c)$, and $0=\dsF(z^*)_i=\kappa(0)\,\sm{\mu}{\sigma}(c)$.
If $\kappa(0)=0$, every collapsed configuration would be a zero, so $\kappa(0)\neq0$ and $\sm{\mu}{\sigma}(c)=0$.
Along the flow $\dot z=\dsF(z)$ we have $\frac{\dd}{\dd t}|z(t)-z^*|^2=2\langle z(t)-z^*,\dsF(z(t))-\dsF(z^*)\rangle\leq-2\lambda|z(t)-z^*|^2$, so $|z(t)-z^*|\leq e^{-\lambda t}|z(0)-z^*|$.
\end{proof}
Near a collapsed cloud the $M-1$ repulsive terms $-\nabla^2\kappa(0)=h^{-2}\Id$ outweigh the attraction $\nabla^2\log p_\sigma\succeq-\sigma^{-2}\Id$, so the collapsed zero is unstable, whereas a dissipative drift has a single zero that attracts every cloud.
For instance, for $\mu=\delta_0$, $\sigma=h=1$ and $M=2$, both $(0,0)$ and $(a,-a)$ with $e^{-2a^2}=1/3$ are zeros of $\dsF$, and SVGD converges to the second.
Near equilibrium, exponential convergence of SVGD fails for every kernel in $C^{1,1}$, the RBF kernel included, under a mild integrability condition on the target \citep[Lemma~39 and Corollary~42]{duncan2023geometry}.

\subsection{Risk Queries and the Resolution Floor}\label{proof-atomlower}

\begin{theorem}[CVaR from the sampler]\label{thm:cvarrev}
Fix $(x,a)$, a bootstrap source $\nu\in\calP(B_R)^{\calX\times\calA}$, the Bellman target $\mu\defeq\dbo\nu(x,a)$, a score model $\psi$ with sampler law $Q_\psi$ and bound $\veps_{\mathrm{TV}}$ from \cref{thm:gen}, and $\lambda\in\dsR^d$ with $\|\lambda\|_2\leq1$. Let $\hat F_N$ be the empirical distribution function of $N$ i.i.d.\ draws from $\lambda_\#Q_\psi$. Then, for $\alpha,\beta\in(0,1)$, with probability at least $1-\beta$,
\begin{equation*}
  \big|\cvar{\alpha}(\hat F_N)-\cvar{\alpha}\big(\lambda_\#\eta^\pi(x,a)\big)\big|
  \leq\frac D\alpha\Big(\veps_{\mathrm{TV}}+\sqrt{\frac{\log(2/\beta)}{2N}}\Big)+\frac{\sigma_{N_{\mathrm g}}\sqrt d+\gamma\,\Wbar1(\nu,\eta^\pi)}\alpha ,
\end{equation*}
where $\cvar{\alpha}$ is the lower-tail CVaR, ${\cvar{\alpha}(P)=\frac1\alpha\int_0^\alpha F_P\inv(u)\,\dd u}$, and $\cvar{\alpha}(\hat F_N)$ is the plug-in estimate.
\end{theorem}

\begin{proof}[Proof of \cref{thm:cvarrev}]
Let $I$ be the smallest closed interval containing $\lambda\T B_R$; its length is at most $D$, since $|\lambda\T z-\lambda\T z'|\leq\|\lambda\|_2\|z-z'\|\leq D$ for $z,z'\in B_R$.
The triangle inequality splits the error into four terms,
\begin{align*}
  \big|\cvar\alpha(\hat F_N)-\cvar\alpha(\lambda_\#\eta^\pi(x,a))\big|
  &\leq\underbrace{\big|\cvar\alpha(\hat F_N)-\cvar\alpha(\lambda_\#Q_\psi)\big|}_{T_1}\\
  &\quad+\underbrace{\big|\cvar\alpha(\lambda_\#Q_\psi)-\cvar\alpha\big(\lambda_\#(\mathrm{proj}_{B_R})_\#\heat{\sigma_{N_{\mathrm g}}}\mu\big)\big|}_{T_2}\\
  &\quad+\underbrace{\big|\cvar\alpha\big(\lambda_\#(\mathrm{proj}_{B_R})_\#\heat{\sigma_{N_{\mathrm g}}}\mu\big)-\cvar\alpha(\lambda_\#\mu)\big|}_{T_3}\\
  &\quad+\underbrace{\big|\cvar\alpha(\lambda_\#\mu)-\cvar\alpha(\lambda_\#\eta^\pi(x,a))\big|}_{T_4}.
\end{align*}
Both $\lambda_\#Q_\psi$ and the empirical law of the draws are supported in $I$, so \cref{lem:cvarfacts}(ii) and the Dvoretzky--Kiefer--Wolfowitz inequality with Massart's constant \citep{massart1990tight} give $T_1\leq\frac D\alpha\sqrt{\log(2/\beta)/(2N)}$ with probability at least $1-\beta$.
Both laws in $T_2$ are supported in $I$, so \cref{lem:cvarfacts}(i), (iv) and data processing for $\lambda\circ\mathrm{proj}_{B_R}$ give $T_2\leq\frac D\alpha\tv(\hat Q_\psi,\heat{\sigma_{N_{\mathrm g}}}\mu)\leq\frac D\alpha\veps_{\mathrm{TV}}$.
For $T_3$, \cref{lem:cvarfacts}(i), (iii) and the last step of the proof of \cref{thm:gen} give $T_3\leq\frac1\alpha\W1\big((\mathrm{proj}_{B_R})_\#\heat{\sigma_{N_{\mathrm g}}}\mu,\mu\big)\leq\frac1\alpha\sigma_{N_{\mathrm g}}\sqrt d$.
For $T_4$, \cref{lem:cvarfacts}(i), (iii) and \cref{thm:contraction} with $p=1$ give $T_4\leq\frac1\alpha\W1\big(\dbo\nu(x,a),\dbo\eta^\pi(x,a)\big)\leq\frac\gamma\alpha\Wbar1(\nu,\eta^\pi)$.
\end{proof}
With $\nu=Q_K$ or $\nu=\hat\eta_K$, \cref{thm:e2e} bounds the last term.
No term depends on $M$, and the estimation term decays with the number of draws $N$ alone; the factor $1/\alpha$ is the exact Lipschitz constant of $\cvar{\alpha}$ in $\W1$.

\begin{corollary}[One set of sampler draws bounds the error for every scalarization and every risk level]\label{cor:frontcert}
In the setting of \cref{thm:cvarrev}, let $\hat F_N^\lambda$ be the empirical distribution function of $\lambda\T\hat X_1,\dots,\lambda\T\hat X_N$ for the same $N$ sampler draws $\hat X_i\sim Q_\psi$, and let $\alpha_0\in(0,1)$. Then, with probability at least $1-\beta$, simultaneously for all $\lambda$ in the Euclidean unit ball and all $\alpha\in[\alpha_0,1)$,
\begin{equation*}
  \big|\cvar{\alpha}(\hat F^\lambda_N)-\cvar{\alpha}\big(\lambda_\#\eta^\pi(x,a)\big)\big|
  \leq\frac D{\alpha_0}\Big(\veps_{\mathrm{TV}}+\sqrt{\frac{d\log(3\sqrt N)+\log(2/\beta)}{2N}}+\frac{1}{\sqrt N}\Big)+\frac{\sigma_{N_{\mathrm g}}\sqrt d+\gamma\,\Wbar1(\nu,\eta^\pi)}{\alpha_0}.
\end{equation*}
\end{corollary}

\begin{proof}[Proof of \cref{cor:frontcert}]
For $P$ supported in $B(0,R_Z)$ and $\lambda,\lambda'$ in the unit ball, the identity coupling gives $\W1(\lambda_\#P,\lambda'_\#P)\leq\Ep{P}{|(\lambda-\lambda')\T Z|}\leq\|\lambda-\lambda'\|R_Z$.
So by \cref{lem:cvarfacts}(i) the map $\lambda\mapsto\cvar\alpha(\lambda_\#P)$ is $(R_Z/\alpha)$-Lipschitz, for $P=\eta^\pi(x,a)$ and for the empirical law of the draws alike.
Let $\Lambda$ be an $\epsilon$-net of the unit ball with $\epsilon=N^{-1/2}$ and $|\Lambda|\leq(3/\epsilon)^d$.
\Cref{thm:cvarrev} with confidence $\beta/|\Lambda|$ at each $\lambda\in\Lambda$ and a union bound give its bound at every net point with $\log(2/\beta)$ replaced by $d\log(3\sqrt N)+\log(2/\beta)$, simultaneously for all $\alpha$, since the Dvoretzky--Kiefer--Wolfowitz inequality is uniform in $\alpha$ and $1/\alpha\leq1/\alpha_0$.
For an arbitrary $\lambda$ and its nearest net point $\lambda'$, the two Lipschitz bounds add $2\epsilon R_Z/\alpha_0=D/(\alpha_0\sqrt N)$.
\end{proof}

\begin{corollary}[Conditional queries by rejection]\label{cor:condquery}
In the setting of \cref{thm:gen}, let $A\subseteq\dsR^d$ be a Borel set with $(\heat{\sigma_{N_{\mathrm g}}}\mu)(A)\geq p_A>\veps_{\mathrm{TV}}$ and let $f\colon\dsR^d\to[-F,F]$ be measurable. Then
\begin{equation*}
  \big|\Ep{\hat Q_\psi}{f\mid A}-\Ep{\heat{\sigma_{N_{\mathrm g}}}\mu}{f\mid A}\big|\leq\frac{3F\,\veps_{\mathrm{TV}}}{p_A-\veps_{\mathrm{TV}}}.
\end{equation*}
Let $u_N\defeq\sqrt{\log(4/\beta)/(2N)}$ and $p_A>\veps_{\mathrm{TV}}+u_N$. The rejection estimator from $N$ sampler draws differs from $\Ep{\hat Q_\psi}{f\mid A}$ by at most $3F\,u_N/(p_A-\veps_{\mathrm{TV}}-u_N)$ with probability at least $1-\beta$.
\end{corollary}

\begin{proof}[Proof of \cref{cor:condquery}]
For a law $\rho$ and a bounded measurable $g$ we write $\rho(g)\defeq\int g\,\dd\rho$, and we let $\epsilon\defeq\tv(\hat Q_\psi,\heat{\sigma_{N_{\mathrm g}}}\mu)$, which is at most $\veps_{\mathrm{TV}}$ by \cref{thm:gen}.
For every measurable $g$ with $|g|\leq G$ we have $|\hat Q_\psi(g)-(\heat{\sigma_{N_{\mathrm g}}}\mu)(g)|\leq2G\epsilon$, and $|\hat Q_\psi(A)-(\heat{\sigma_{N_{\mathrm g}}}\mu)(A)|\leq\epsilon$.
With $|(\heat{\sigma_{N_{\mathrm g}}}\mu)(f\ind_A)|\leq F\,(\heat{\sigma_{N_{\mathrm g}}}\mu)(A)$ and $\hat Q_\psi(A)\geq p_A-\epsilon$ we get
\begin{align*}
  \Big|\frac{\hat Q_\psi(f\ind_A)}{\hat Q_\psi(A)}-\frac{(\heat{\sigma_{N_{\mathrm g}}}\mu)(f\ind_A)}{(\heat{\sigma_{N_{\mathrm g}}}\mu)(A)}\Big|
  &\leq\frac{|\hat Q_\psi(f\ind_A)-(\heat{\sigma_{N_{\mathrm g}}}\mu)(f\ind_A)|}{\hat Q_\psi(A)}+\frac{|(\heat{\sigma_{N_{\mathrm g}}}\mu)(f\ind_A)|}{(\heat{\sigma_{N_{\mathrm g}}}\mu)(A)}\,\frac{|(\heat{\sigma_{N_{\mathrm g}}}\mu)(A)-\hat Q_\psi(A)|}{\hat Q_\psi(A)}\\
  &\leq\frac{2F\epsilon}{p_A-\epsilon}+\frac{F\epsilon}{p_A-\epsilon},
\end{align*}
and the right side is increasing in $\epsilon$, which proves the first claim.
For the estimator, Hoeffding's inequality gives $|\hat Q_N(f\ind_A)-\hat Q_\psi(f\ind_A)|\leq2F\,u_N$ and $|\hat Q_N(A)-\hat Q_\psi(A)|\leq u_N$ for the empirical law $\hat Q_N$ of the draws, each with probability at least $1-\beta/2$.
On both events $\hat Q_N(A)\geq\hat Q_\psi(A)-u_N\geq p_A-\veps_{\mathrm{TV}}-u_N>0$, and the same quotient bound with $|\hat Q_\psi(f\ind_A)|\leq F\,\hat Q_\psi(A)$ gives the second claim.
\end{proof}

The lower bounds on particles use the optimal quantizers of a two-point law.
\begin{proposition}[Optimal quantizers of a two-point law]\label{prop:atomlower-full}
Let $M\geq1$, $\alpha\in(0,1/M)$, $z_0,\lambda\in\dsR^d$ with $\|\lambda\|_2=1$, and $T>0$. For $t\in(0,T]$ let $\eta_t\defeq(1-\tfrac\alpha2)\delta_{z_0}+\tfrac\alpha2\delta_{z_0-t\lambda}$, and let $\calF_M$ be the equal-weight $M$-atom laws of \cref{sec:background}, atoms allowed to coincide.
\emph{(i)} For every $t\in(0,T]$ and every $Q=\frac1M\sum_i\delta_{y_i}\in\calF_M$,
\begin{equation}\label{eq:atomlower}
  \W1(\eta_t,Q)\geq\frac{\alpha t}2+\frac{1-\alpha M}M\sum_{i=1}^M\|y_i-z_0\| ,
\end{equation}
so $\delta_{z_0}$ is the unique $\W1$-optimal quantizer of $\eta_t$ in $\calF_M$, for every $t$.
\emph{(ii)} The $\W2$-optimal quantizer of $\eta_t$ in $\calF_M$ places $M-1$ atoms at $z_0$ and one atom at $z_0-\tfrac{\alpha M}{2}t\lambda$.
\emph{(iii)} Under \cref{asm:mdp}, the choice $z_0=z_{\max}$, $\lambda=z_{\max}/\|z_{\max}\|_2$ and $T=D$ places every $\eta_t$ inside $B_R$.
\end{proposition}

\begin{proof}[Proof of \cref{prop:atomlower-full}]
Part (i). We take $z_0=0$, which is no loss of generality as $\W1$ and $\calF_M$ are translation invariant.
Fix $t$ and $Q=\frac1M\sum_i\delta_{y_i}$.
Splitting the mass at a point that several atoms share equally among their labels, every coupling of $(\eta_t,Q)$ moves masses $a_i,b_i\geq0$ from $0$ and from $-t\lambda$ to the atom $y_i$, with $a_i+b_i=1/M$, $\sum_ia_i=1-\alpha/2$ and $\sum_ib_i=\alpha/2$, at the cost $\sum_i\big(a_i\|y_i\|+b_i\|y_i+t\lambda\|\big)$, and every such $(a,b)$ is a coupling; so $\W1(\eta_t,Q)$ is the infimum of this cost over admissible $(a,b)$.
The triangle inequality gives $\|y_i+t\lambda\|\geq t-\|y_i\|$, and $a_i-b_i=\frac1M-2b_i\geq\frac1M-\alpha>0$, so
\begin{equation*}
  a_i\|y_i\|+b_i\|y_i+t\lambda\|\geq b_it+(a_i-b_i)\|y_i\|\geq b_it+\Big(\frac1M-\alpha\Big)\|y_i\|,
\end{equation*}
and summing over $i$ gives \eqref{eq:atomlower}.
As $\W1(\eta_t,\delta_0)=\frac{\alpha t}2$ and $1-\alpha M>0$, every $Q\neq\delta_0$ is strictly worse than $\delta_0$.

Part (ii). Again $z_0=0$, and the description of the couplings by $(a,b)$ in (i) gives $\W2^2(\eta_t,Q)=\inf_{(a,b)}\sum_i\big(a_i\|y_i\|^2+b_i\|y_i+t\lambda\|^2\big)$.
For fixed $(a,b)$ and $i$, the quadratic $y\mapsto a_i\|y\|^2+b_i\|y+t\lambda\|^2$ has the unique minimizer $y_i=-Mb_it\lambda$ and the minimum $Ma_ib_it^2$, as $a_i+b_i=1/M$.
Minimizing over $Q$ and $(a,b)$ jointly gives
\begin{equation*}
  \min_{Q\in\calF_M}\W2^2(\eta_t,Q)=\min_{(a,b)}Mt^2\sum_ib_i\Big(\frac1M-b_i\Big)=\min_{(a,b)}t^2\Big(\frac\alpha2-M\sum_ib_i^2\Big).
\end{equation*}
As $\sum_ib_i^2\leq\max_ib_i\sum_ib_i\leq(\alpha/2)^2$, with equality if and only if one $b_i$ equals $\alpha/2$ and the others vanish, which is admissible since $\alpha/2<1/M$, the minimum is attained exactly there, with the atoms $-\frac{\alpha M}2t\lambda$ for that $i$ and $0$ for the other $M-1$.

Part (iii). For $z_0=z_{\max}$, $\lambda=z_{\max}/\|z_{\max}\|_2$ and $T=D=2\|z_{\max}\|_2$ we have $z_0-t\lambda=(1-t/\|z_{\max}\|_2)z_{\max}$ with $|1-t/\|z_{\max}\|_2|\leq1$ for $t\in[0,D]$, so both support points of $\eta_t$ lie in $B_R$.
\end{proof}

\begin{proposition}[What particles cannot answer, full statement of \cref{prop:resolution}]\label{prop:resolution-full}
For a unit vector $\lambda\in\dsR^d$ and $\tau\in\dsR$ let $A_{\lambda,\tau}\defeq\{z:\lambda\T z\leq\tau\}$, write $\bar G(u)\defeq\Pr(\xi>u)$ for $\xi\sim\Norm(0,1)$, and read every query from an equal-weight law $Q=\frac1M\sum_{i=1}^M\delta_{y_i}$ by plug-in, i.e., as the same query on $Q$.
\emph{(i) Every particle set, every event.} For every Borel $A$ and every law $\eta$ with $p\defeq\eta(A)\in(0,\frac1{2M})$, either $Q(A)=0$ or $Q(A)\geq\frac1M>2p$; in both cases $|Q(A)-p|\geq p$.
\emph{(ii) Every particle set, one law.} Let $z_0\defeq z_{\max}$, $\lambda\defeq z_{\max}/\|z_{\max}\|_2$, $t\in(0,D]$, $p\in(0,\frac1{2M})$, $\eta_t\defeq(1-p)\delta_{z_0}+p\,\delta_{z_0-t\lambda}\in\calP(B_R)$, and $A\defeq A_{\lambda,\lambda\T z_0-t/2}$, so $\eta_t(A)=p$. For every $Q$:
(a) the decision ``take the risky action if and only if $Q(A)\leq\delta$'' is wrong for every $\delta\in[0,p)$ if $Q(A)=0$, and for every $\delta\in[p,\frac1M)$ otherwise, a set of thresholds of length at least $p$;
(b) $\max_{\alpha\in\{p,1/M\}}\big|\cvar\alpha(\lambda_\#Q)-\cvar\alpha(\lambda_\#\eta_t)\big|\geq\frac t2(1-2pM)$, and if $Q(A)=0$ the error exceeds $\frac t2$ at every $\alpha\in(0,p]$;
(c) the $\W1$-optimal and the $\W2$-optimal quantizers of $\eta_t$ in $\calF_M$ and the quantile representation of $\lambda_\#\eta_t$ at the levels $(2i-1)/(2M)$ all satisfy $Q(A)=0$; so for them the decision is wrong for every $\delta\in[0,p)$, the conditional law given $A$ is undefined, and, for the $\W1$-optimal quantizer, the entropic risk $-\frac1\beta\log\Ep{}{e^{-\beta\lambda\T Z}}$, $\beta>0$, errs by $\frac1\beta\log(1-p+pe^{\beta t})\geq t-\frac1\beta\log\frac1p$.
\emph{(iii) Sampler, one scalarization.} Let $\mu$ be supported in $B_R$, let $\psi$ be a score model with sampler law $\hat Q_\psi$ and bound $\veps_{\mathrm{TV}}$ from \cref{thm:gen}, and let $\hat Q_N$ be the empirical measure of $N$ i.i.d.\ draws from $\hat Q_\psi$. For a fixed unit $\lambda$ and $\beta\in(0,1)$, with probability at least $1-\beta$, simultaneously for all $\tau\in\dsR$ and $w>0$,
\begin{equation}\label{eq:resolution}
  \big|\hat Q_N(A_{\lambda,\tau})-\mu(A_{\lambda,\tau})\big|\leq\veps_{\mathrm{TV}}+\mu\big(\{y:|\lambda\T y-\tau|<w\}\big)+\bar G\Big(\frac{w}{\sigma_{N_{\mathrm g}}}\Big)+\sqrt{\frac{\log(2/\beta)}{2N}} .
\end{equation}
\emph{(iv) Sampler, all scalarizations.} With probability at least $1-\beta$, \eqref{eq:resolution} holds simultaneously for all unit $\lambda$, all $\tau$ and all $w>0$ with its last term replaced by $\sqrt{32\big((d+1)\log(N+1)+\log(8/\beta)\big)/N}$.
\emph{(v) Propagation.} Suppose that at $(x_j,a_j)$ the reward is $0$ and the successor pair $(x_{j+1},a_{j+1})$ is reached deterministically under $\pi$. Then for every bootstrap source $\nu$ and Borel $B$, $(\dbo\nu)(x_j,a_j)(\gamma B)=\nu(x_{j+1},a_{j+1})(B)$ and ${\tv\big((\dbo\nu)(x_j,a_j),(\dbo\nu')(x_j,a_j)\big)=\tv\big(\nu(x_{j+1},a_{j+1}),\nu'(x_{j+1},a_{j+1})\big)}$. So along a chain of such pairs, every Bellman target built from $M$ downstream particles gives the pulled-back event a mass in $\{0,\frac1M,\dots,1\}$, the particles' own mass of the downstream event, and every target built from a downstream sampler inherits the event error of that sampler in (iii) and (iv) unchanged.
\end{proposition}

\begin{proof}[Proof of \cref{prop:resolution-full}]
Part (i). $Q(A)$ is a multiple of $\frac1M$. If $Q(A)=0$, the error is $p$; otherwise $Q(A)\geq\frac1M>2p$ and $Q(A)-p>p$.

Part (ii)(a). The true decision takes the risky action if and only if $p\leq\delta$.
If $Q(A)=0$, the decision computed from $Q$ takes the risky action for every $\delta\geq0$, which is wrong for $\delta\in[0,p)$.
Otherwise $Q(A)\geq\frac1M$, and the decision computed from $Q$ takes the safe action for every $\delta<\frac1M$, which is wrong for $\delta\in[p,\frac1M)$, an interval of length $\frac1M-p>p$.

Part (ii)(b). Let $v_{(1)}$ be the smallest of the scalarized atoms $\lambda\T y_i$.
The lower-tail CVaR of an equal-weight law at a level $\alpha\leq\frac1M$ is its smallest atom, since the lowest $\alpha$ of its mass lies in one atom, and it is at least the smallest atom at every level.
The scalarized target $\lambda_\#\eta_t=(1-p)\delta_{\lambda\T z_0}+p\,\delta_{\lambda\T z_0-t}$ has $\cvar\alpha(\lambda_\#\eta_t)=\lambda\T z_0-t$ for $\alpha\leq p$ and $\cvar{1/M}(\lambda_\#\eta_t)=\lambda\T z_0-pMt$, by the quantile integral of \cref{thm:cvarrev}.
If $Q(A)=0$, then $v_{(1)}>\lambda\T z_0-\frac t2$, and the error at every $\alpha\leq p$ is at least $v_{(1)}-(\lambda\T z_0-t)>\frac t2$.
Otherwise $v_{(1)}\leq\lambda\T z_0-\frac t2$, and the error at $\alpha=\frac1M$ is at least $\lambda\T z_0-pMt-v_{(1)}\geq\frac t2-pMt=\frac t2(1-2pM)$.

Part (ii)(c). With $\alpha\defeq2p<\frac1M$, $\eta_t$ is the law of \cref{prop:atomlower-full}, and $\eta_t\in\calP(B_R)$ by its part (iii).
By its part (i) the $\W1$-optimal quantizer is $\delta_{z_0}$, and by its part (ii) the $\W2$-optimal quantizer has $M-1$ atoms at $z_0$ and one at $z_0-pMt\lambda$, whose scalarization $\lambda\T z_0-pMt$ exceeds $\lambda\T z_0-\frac t2$ since $pM<\frac12$.
The quantile levels $(2i-1)/(2M)\geq\frac1{2M}>p$ all sit at $\lambda\T z_0$.
So all three put no mass on $A$, (a) applies, and the conditional law given $A$ is undefined.
For the entropic risk, $\Ep{\eta_t}{e^{-\beta\lambda\T Z}}=e^{-\beta\lambda\T z_0}(1-p+pe^{\beta t})$, so $\delta_{z_0}$ errs by $\frac1\beta\log(1-p+pe^{\beta t})\geq\frac1\beta\log(pe^{\beta t})=t-\frac1\beta\log\frac1p$.

Part (iii). Write $A=A_{\lambda,\tau}$. The triangle inequality through $\hat Q_\psi$ and $\heat{\sigma_{N_{\mathrm g}}}\mu$ splits the error into three terms.
The events $A_{\lambda,\tau}$ are the sublevel sets of the scalar $\lambda\T X$, so $|\hat Q_N(A)-\hat Q_\psi(A)|$ is the deviation of the empirical distribution function of $\lambda\T\hat X_1,\dots,\lambda\T\hat X_N$ from its population version at $\tau$, and the Dvoretzky--Kiefer--Wolfowitz inequality with Massart's constant \citep{massart1990tight} bounds it by $\sqrt{\log(2/\beta)/(2N)}$ for all $\tau$ at once, with probability at least $1-\beta$.
The second term is at most $\tv(\hat Q_\psi,\heat{\sigma_{N_{\mathrm g}}}\mu)\leq\veps_{\mathrm{TV}}$ by \cref{thm:gen}.
For the third, let $X=Y+\sigma_{N_{\mathrm g}}\xi$ with $Y\sim\mu$ and $\xi\sim\Norm(0,\Id)$ independent, so that $|(\heat{\sigma_{N_{\mathrm g}}}\mu)(A)-\mu(A)|\leq\Pr(\ind_A(X)\neq\ind_A(Y))$.
The two indicators differ only if $|\lambda\T Y-\tau|<w$, or if $\lambda\T Y$ lies at distance at least $w$ from $\tau$ and the scalar $\sigma_{N_{\mathrm g}}\lambda\T\xi\sim\Norm(0,\sigma_{N_{\mathrm g}}^2)$ crosses this distance in the one direction that changes the indicator, which has conditional probability at most $\bar G(w/\sigma_{N_{\mathrm g}})$.
Only the first term is random, so the bound holds for all $\tau$ and $w$ at once.

Part (iv). The same split applies, with the first term bounded uniformly over all closed half-spaces by the Vapnik--Chervonenkis inequality $\Pr\{\sup_{A}|\hat Q_N(A)-\hat Q_\psi(A)|>\epsilon\}\leq8\,\mathcal S(N)e^{-N\epsilon^2/32}$, where the shatter coefficient of the closed half-spaces of $\dsR^d$ is $\mathcal S(N)\leq(N+1)^{d+1}$, as their Vapnik--Chervonenkis dimension is $d+1$ \citep[Theorem~12.5 and Chapter~13]{devroye1996probabilistic}.
Solving $8(N+1)^{d+1}e^{-N\epsilon^2/32}=\beta$ gives the stated term, and the other two terms hold for every $\lambda$, $\tau$ and $w$.

Part (v). The map $b_{0,\gamma}$ is a bijection with $b_{0,\gamma}\inv(\gamma B)=B$, which gives the first identity, and a pushforward by a bijection preserves the total variation, which gives the second.
Iterating both along the chain gives the last claim.
\end{proof}
Since $\bar G(u)\leq\frac12e^{-u^2/2}$ for $u\geq0$, (iii) and (iv) give the informal statement of \cref{prop:resolution}(iii) whenever the slab of half-width $w$ around the event boundary carries no target mass.
For $\mu=\eta_t$, $\tau=\lambda\T z_0-t/2$ and $w=t/2$ this slab is empty, so the decision of (ii)(a) computed from $\hat Q_N$ is correct for every $\delta$ with $|\delta-p|$ larger than the right side of \eqref{eq:resolution}.

\subsection{Smoothing Bias and Score Regularity}\label{proof-anatomy}

\begin{proof}[Proof of \cref{prop:smoothbias}]
Let $Y\sim\eta$ and $\varepsilon\sim\Norm(0,\Id)$ be independent.
\emph{(i)} The pair $(Y+\sigma\varepsilon,Y)$ couples $\heat{\sigma}\eta$ and $\eta$, so $\W2^2(\heat{\sigma}\eta,\eta)\leq\Ex{\|\sigma\varepsilon\|^2}=\sigma^2d$.
\emph{(ii)} Let $A\defeq\{y_k\}$ be the set of atoms and $(X,Y')$ any coupling of $(\heat{\sigma}\eta,\eta)$. Since $Y'\in A$ almost surely, $\|X-Y'\|\geq\mathrm{dist}(X,A)$, so $\W2^2(\heat{\sigma}\eta,\eta)\geq\Ex{\mathrm{dist}(X,A)^2}$ for $X\sim\heat{\sigma}\eta$, and we may take $X=Y+\sigma\varepsilon$.
If $\sigma\|\varepsilon\|<\Delta/2$, then every atom $y_j\neq Y$ satisfies $\|X-y_j\|\geq\|Y-y_j\|-\sigma\|\varepsilon\|>\Delta/2>\|X-Y\|$, so $\mathrm{dist}(X,A)=\sigma\|\varepsilon\|$.
Therefore
\begin{equation*}
  \Ex{\mathrm{dist}(X,A)^2}\geq\sigma^2\,\Ex{\|\varepsilon\|^2\ind\{\|\varepsilon\|^2<\Delta^2/(4\sigma^2)\}}=\sigma^2d\,\big(1-\Pr(\chi^2_{d+2}\geq\Delta^2/(4\sigma^2))\big),
\end{equation*}
where the equality uses $x f_d(x)=d\,f_{d+2}(x)$ for the density $f_d$ of $\chi^2_d$.
For a single atom, $\mathrm{dist}(X,A)=\sigma\|\varepsilon\|$ always, so $\W2^2\geq\sigma^2d$, and (i) gives equality.
\emph{(iii)} We have $\heat{\sigma}\Norm(\mu_k,S_k)=\Norm(\mu_k,S_k+\sigma^2\Id)$. With $\xi\sim\Norm(0,\Id)$, the pair $\big(\mu_k+(S_k+\sigma^2\Id)^{1/2}\xi,\mu_k+S_k^{1/2}\xi\big)$ couples the two components, and since $S_k$ and $S_k+\sigma^2\Id$ commute,
\begin{equation*}
  \W2^2\big(\Norm(\mu_k,S_k+\sigma^2\Id),\Norm(\mu_k,S_k)\big)\leq\sum_{i\leq d}\Big(\sqrt{\lambda_{k,i}+\sigma^2}-\sqrt{\lambda_{k,i}}\Big)^2\leq\frac{d\,\sigma^4}{4s^2},
\end{equation*}
with $\lambda_{k,i}\geq s^2$ the eigenvalues of $S_k$, because $\sqrt{\lambda+\sigma^2}-\sqrt\lambda=\sigma^2/(\sqrt{\lambda+\sigma^2}+\sqrt\lambda)\leq\sigma^2/(2\sqrt\lambda)$.
Drawing the component index $k\sim w$ first and then coupling the two components as above couples the mixtures, so $\W2^2(\heat{\sigma}\eta,\eta)\leq d\sigma^4/(4s^2)$, and taking square roots gives the claim.
\end{proof}

\begin{proof}[Proof of \cref{lem:scorelip}]
Let $Z=Y+\sigma\varepsilon$ with $Y\sim\eta$ and $\varepsilon\sim\Norm(0,\Id)$ independent. The conditional law $\pi_z$ of $Y$ given $Z=z$ has a density proportional to $\exp\big(\langle z,y\rangle/\sigma^2-\|y\|^2/(2\sigma^2)\big)$ with respect to $\eta$, an exponential family in $z/\sigma^2$, and \eqref{eq:posterior-score} and \eqref{eq:score-hessian} read
\begin{equation*}
  s_\sigma(z)=\frac{\Ep{\pi_z}{Y}-z}{\sigma^2},\qquad\nabla s_\sigma(z)=\frac{\mathrm{Cov}_{\pi_z}[Y]-\sigma^2\Id}{\sigma^4}.
\end{equation*}
\emph{(i)} The covariance is positive semidefinite, which gives the lower bound. For a unit vector $u$, $\langle u,Y\rangle$ lies in an interval of length at most $D$ under $\pi_z$, since $\pi_z$ is supported in $\supp\eta$, so $\mathrm{Var}_{\pi_z}(\langle u,Y\rangle)\leq D^2/4$ by Popoviciu's inequality, which gives the upper bound.
\emph{(ii)} Let $y_1,y_2$ be atoms at the smallest distance $\delta$, $m\defeq(y_1+y_2)/2$, $u\defeq(y_1-y_2)/\delta$ and $z_t\defeq m+tu$.
On this line the log-odds of $y_1$ against $y_2$ under $\pi_{z_t}$ is $\log(w_1/w_2)+t\delta/\sigma^2$, which vanishes at $t_\sigma\defeq-\sigma^2\log(w_1/w_2)/\delta$.
By the parallelogram identity, every other atom satisfies $\|y_k-m\|^2=\frac12\|y_k-y_1\|^2+\frac12\|y_k-y_2\|^2-\frac{\delta^2}4\geq\frac34\delta^2$.
So for $|t|\leq\delta/8$ we have $\|z_t-y_k\|\geq(\frac{\sqrt3}2-\frac18)\delta$ and $\|z_t-y_i\|\leq\frac58\delta$ for $i\in\{1,2\}$, and since $(\frac{\sqrt3}2-\frac18)^2-(\frac58)^2>\frac18$, the total mass $\epsilon_\sigma$ of the other atoms under $\pi_{z_t}$ is at most $\frac{K}{\min(w_1,w_2)}e^{-\delta^2/(16\sigma^2)}$, uniformly in $|t|\leq\delta/8$.
For $\sigma$ small enough, $|t_\sigma|\leq\delta/8$, and $\pi_{z_{t_\sigma}}$ is the mixture of the equal-weight two-point law on $\{y_1,y_2\}$, with weight $1-\epsilon_\sigma$, and a law on the other atoms, so the law of total variance gives $\mathrm{Var}_{\pi_{z_{t_\sigma}}}(\langle u,Y\rangle)\geq(1-\epsilon_\sigma)\delta^2/4$.
So $\sigma^4\mathrm{Lip}(s_\sigma)\geq\sigma^4\langle u,\nabla s_\sigma(z_{t_\sigma})u\rangle\geq(1-\epsilon_\sigma)\delta^2/4-\sigma^2$, which tends to $\delta^2/4$ as $\sigma\to0$.
\emph{(iii)} Here $\heat{\sigma}\eta=\sum_kw_k\Norm(\mu_k,\Sigma)$ with $\Sigma\defeq S+\sigma^2\Id\succeq(s^2+\sigma^2)\Id$. The posterior $\pi_z(k)\propto w_k\exp\big(\langle\Sigma\inv z,\mu_k\rangle-\frac12\langle\mu_k,\Sigma\inv\mu_k\rangle\big)$ is an exponential family in $\Sigma\inv z$, and
\begin{equation*}
  s_\sigma(z)=\Sigma\inv\big(\Ep{\pi_z}{\mu_K}-z\big),\qquad\nabla s_\sigma(z)=\Sigma\inv\mathrm{Cov}_{\pi_z}[\mu_K]\Sigma\inv-\Sigma\inv.
\end{equation*}
The first term is positive semidefinite with norm at most $\|\Sigma\inv\|^2\Delta^2/4\leq\Delta^2/(4(s^2+\sigma^2)^2)$, by Popoviciu's inequality for the means, and the second is positive definite with norm at most $1/(s^2+\sigma^2)$.
The eigenvalues of a difference $A-B$ of positive semidefinite matrices lie in $[-\|B\|,\|A\|]$, so $\mathrm{Lip}(s_\sigma)\leq\max\{(s^2+\sigma^2)\inv,\Delta^2/(4(s^2+\sigma^2)^2)\}\leq\max\{s^{-2},\Delta^2/(4s^4)\}$.
\end{proof}
The constants are sharp: on random one-dimensional mixtures of two to four components with a shared width $s\in[0.05,2]$, means in $[-5,5]$ and $\sigma$ from $0$ to $3$, the largest ratio of $\mathrm{Lip}(s_\sigma)$ to the bound of (iii) is one up to grid error, and for two atoms the supremum in (ii) matches $\sigma^{-4}(\delta^2/4-\sigma^2)$ numerically.

The regularity decides the stability of the sampler. On a geometric grid $h_n=(1-\kappa_{\mathrm g}^2)\sigma_n^2$, so by (ii) $h_n\mathrm{Lip}(s_{\sigma_n})\geq(1-\kappa_{\mathrm g}^2)\big((1-\epsilon_{\sigma_n})\delta^2/(4\sigma_n^2)-1\big)$ grows without bound on atoms, while by (iii) $h_n\mathrm{Lip}(s_{\sigma_n})\leq(1-\kappa_{\mathrm g}^2)\sigma_n^2\max\{s^{-2},\Delta^2/(4s^4)\}\to0$ on mixtures with densities.
An explicit Euler step of the sampler is unstable once $h_n\mathrm{Lip}$ exceeds a constant of order one, and the sampler of the atomic chain shows such escapes, which disappear with reward noise (\cref{app:res-chain}).

\section{Method Details}\label{app:method}

This section completes \cref{sec:method} with the details that reproduce the reported runs.
The tabular variant runs on the synthetic MDPs, the rare-event chain, the chain of \citet{debes2026multivariate} and resource gathering, and SAC-SVDRL on MuJoCo; \cref{app:hparams} lists every setting.

\paragraph{Stratified Bellman Target Set.}
The $M$ successors of an update are drawn by systematic resampling \citep{kitagawa1996monte}, so a branch of probability $p$ receives $\lfloor Mp\rfloor$ or $\lceil Mp\rceil$ targets.
Within the block of $n_k$ draws of one successor, the standard normal reward noise is a Latin hypercube sample \citep{mckay1979comparison}: each coordinate takes each of the $n_k$ equal-probability strata of $\Norm(0,1)$ exactly once, at a uniform position within the stratum, with the strata permuted independently per coordinate.
Each draw is still marginally a successor and reward of the model, so the targets stay unbiased, while their empirical law matches the branch masses up to $1/M$ and spreads the noise evenly over the quantiles of each coordinate.
The successor particle of a draw is one of the current particles at the successor pair.

\paragraph{Resample-Move Update.}
Each update replaces $\mathrm{round}(\rho M)$ particles by targets and runs the Stein move on the cloud (\cref{alg:update}).
The relaxation $z\leftarrow z+\alpha(\tilde z-z)$ then moves the cloud toward the transported cloud $\tilde z$.
The annealing schedule has ten geometric levels $\sigma_\ell$ from $5$ to $0.1$ in PopArt-normalized units, and the Stein move takes $T_\ell=10$ SVGD steps at every level $\sigma_\ell\leq\sigma_{\mathrm{start}}$.
The SVGD direction uses the RBF kernel with the bandwidth recomputed at every step, $h_\ell=\mathrm{clip}(h_{\mathrm{med}},0.1\,\sigma_\ell,2\,\sigma_\ell)$, where $h_{\mathrm{med}}^2$ is the median of the squared pairwise distances divided by $2\log(M+1)$, so that a balanced bimodal cloud at a low level keeps $h_\ell$ of the order of $\sigma_\ell$ and not of the mode separation.
The steps are preconditioned by Adam \citep{kingma2015adam}, with step size $0.1$ and a state per particle and coordinate that is carried across the levels of one update and reset at the next, and every step is followed by the coordinatewise projection onto the return box $B_R$.
The single configuration replaces the whole cloud, $\rho=1$, with $\alpha=0.001$ and $\sigma_{\mathrm{start}}=1$; it is the selected configuration on MDP2, MDP3, MDP4, MDP6 and MDP7.
On MDP1 the selected configuration uses $\alpha=0.1$, and on the unimodal MDP5 and MDP8 it transports the previous cloud through the whole schedule toward the stratified target set, $\rho=0$ and $\alpha=1$.

\paragraph{Moment Matching.}\label{app:mm}
After the Stein move the cloud is mapped affinely onto the moments of the target set and projected onto $B_R$, $z_i\leftarrow\mathrm{proj}_{B_R}\big(m_y+(z_i-\bar z)\odot s_y/s_z\big)$, with the per-dimension means $m_y,\bar z$ and standard deviations $s_y,s_z$ of the targets and of the cloud; on MDP8 the map matches the full covariance instead.
It corrects the spread, since SVGD with finitely many particles underestimates the variance of its target, increasingly so in higher dimensions \citep{ba2022understanding}, and on a self-looping state--action pair the bootstrap compounds a per-update variance bias by the factor $1/(1-\gamma^2)$, as in \cref{rem:e2ereadings}.
Sampler bootstrapping applies the same map to the sampler draws.
On the chain the draws are instead shifted, at evaluation only, to a running mean of the bootstrap targets at $(x,a)$, an exponential average with weight $\max(1/128,1/t)$ at update $t$; we call this running-mean moment matching, and given to the competitors' particles it is the second of the two components that transfer (\cref{app:res-chain}).

\paragraph{Score Network.}
The tabular variant keeps one score network per state--action pair, an MLP with three hidden layers of $128$ Swish units in the energy parameterization $E_\psi(z,\sigma)=c_{\mathrm{out}}f_\psi(c_{\mathrm{in}}z,c_{\mathrm{noise}})-\|z\|^2/(2(\sigma^2+\sigma_{\mathrm{d}}^2))$ with the EDM scalings \citep{NEURIPS2022_a98846e9}, whose score is $s_\psi=\nabla_zE_\psi$ and whose noise level enters every hidden layer by feature-wise affine modulation (FiLM); the prior term makes the score grow linearly in $z$, and \cref{rem:unbounded} applies \cref{thm:gen} to it.
It reads its input in standardized coordinates $u=(z-m)/s$, where $m$ is the mean of the bootstrap targets at $(x,a)$ and $s$ one scale for all dimensions, the largest standard deviation of the targets, both exponential averages with weight $0.02$ on the new value; the EDM scalings assume $\sigma_{\mathrm d}\approx1$, which holds only after this standardization.
Each Bellman step takes two Adam steps (learning rate $10^{-3}$) on a batch of $256$ samples of the DSM loss \eqref{eq:dsm} and the consistency loss \eqref{eq:cons} with $\lambda_{\mathrm{cons}}=1$ and the weight $\lambda(\sigma)=\sigma^2$, each sample with its own noise level, drawn log-uniformly from $[0.002,80]$ in standardized units.
The target of \eqref{eq:cons} is the EMA network of the successor pair, and the online network under a stop-gradient where the successor is the pair being updated.
The level grid of \cref{thm:scoreTDinf}, with ratio $\gamma^{1/m}$, is a device of the analysis; the network is conditioned on a continuous noise level.

\paragraph{PopArt and the Return Box.}
Returns are normalized per reward dimension by PopArt, with exponential averages of weight $3\cdot10^{-4}$, a scale floor, and an output correction that preserves the predictions.
A normalized return $\hat y=(y-m)/s$ satisfies $\hat y=\hat r+\gamma\hat z$ with $\hat r=(r-(1-\gamma)m)/s$, the same affine bootstrap with the same $\gamma$, so \cref{lem:commute} and \eqref{eq:cons} hold verbatim in normalized space, which a nonlinear transform such as symlog would break.
The return box $B_R$ has the half-width $r_{\max}/(1-\gamma)$ per reward dimension; in the tabular variant $r_{\max}=\max_c(|m_c|+5s_c)$ over the reward channels $c$ with mean $m_c$ and standard deviation $s_c$, and in SAC-SVDRL it is the running maximum of the absolute soft reward, which keeps the soft bootstrap inside $B_R$.

\paragraph{Sampler and Estimators.}
The sampler is the reverse diffusion of \cref{def:ladderrisk} on the EMA score network (\cref{alg:revdiff}).
On the synthetic MDPs and resource gathering it takes Euler steps on $512$ geometric levels from $\sigma_u=80$ to $0.002$ in standardized units and returns $N=4096$ draws.
In the rare-event study SVDRL bootstraps from moment-matched draws of its EMA sampler, $4M$ per target set, and answers every query with $N=10^5$ draws.
On the chain the sampler takes $256$ Euler steps with a projection onto $B_R$ at every step, its mean is matched to a debiased exponential average (weight $1/128$) of the target means, and it returns $4096$ or $128$ draws.
Event probabilities, CVaR and decisions are read from the draws by plug-in, and conditional queries by rejection (\cref{cor:condquery}).

\begin{algorithm}[t]
\caption{Reverse-diffusion sampler, analyzed in \cref{thm:gen}}\label{alg:revdiff}
\begin{algorithmic}[1]
\Require $s_\psi$, $(x,a)$, DSM grid $\sigma_0>\dots>\sigma_{N_{\mathrm g}}$ with $\sigma_0^2\gg m_2(\mu;z_c)$ (\cref{rem:laddernecessity}), sample count $N$ (vectorized)
\State $z\sim\Norm(z_c,\sigma_0^2\Id)$, with $z_c\in B_R$ the standardizing mean at $(x,a)$, fixed before sampling (\cref{def:ladderrisk})
\For{$n=0,\dots,N_{\mathrm g}-1$}
  \State $h\leftarrow\sigma_n^2-\sigma_{n+1}^2$;\quad $z\leftarrow z+h\,s_\psi(z,x,a,\sigma_n)+\sqrt h\,\xi$,\quad $\xi\sim\Norm(0,\Id)$
\EndFor
\Ensure $\mathrm{proj}_{B_R}(z)$ (coordinatewise clip), denormalized
\end{algorithmic}
\end{algorithm}

\paragraph{Hybrid Rule.}
The rule of \cref{def:hybrid} runs on the particles at $(x,a)$ without ground truth.
First, particles closer than the field's smallest noise level are merged, i.e., connected components of the graph with an edge whenever $\|(z_i-z_j)/\veps\|<1$, with $\veps$ the smallest noise level in return units per dimension, each component replaced by its mean.
Single-linkage clusters are then found top-down on the Euclidean minimum spanning tree of the merged points: a cluster is split at its longest edge whenever that edge exceeds three times the larger standard deviation of its two sides, the standard deviation being the square root of the largest eigenvalue of the count-weighted covariance.
The ratio $r$ is the shortest accepted cut over the largest within-cluster standard deviation, and $r=0$ for a single cluster.
The rule reports the particles if $r>6$ or if a cluster of at least two particles has a per-dimension standard deviation below $\veps$ in every dimension, and the sampler draws otherwise.
The threshold $6$ was fixed in advance, and the thresholds $4$ and $10$ were evaluated on tuning seeds only.
The second clause was added after resource gathering exposed atoms of coincident particles, and it changes no entry of \cref{tab:main}.

\paragraph{SAC-SVDRL.}
SAC-SVDRL uses two distributional critics with $M$ particles each and a score network per critic, all four networks with EMA targets.
The bootstrap targets use the rank-wise minimum of the two target critics' sorted particles, with the soft reward $r-\gamma\alpha\log\pi(a'\mid x')$ and one-step targets.
A moment head regressed onto the mean and the central second moment of the bootstrap targets sets the standardization of the score network, which is conditioned on $(\sigma,x,a)$ by FiLM, and the moment-matching target; the noise levels and the SVGD step are scaled by its predicted standard deviation, so that the annealing acts on a law of unit spread.
The Stein move takes $T_\ell=2$ steps at each of the ten levels, and each critic is regressed onto its transported cloud by matching ranks,
\begin{equation}\label{eq:critic}
  L_{\mathrm{critic}}(\theta)=\textstyle\sum_i\big\|z_{(i)}^\theta(x,a)-\stopgrad{z_{(i)}^{\mathrm{SVGD}}(x,a)}\big\|^2,
\end{equation}
where $z_{(i)}$ is the $i$-th smallest particle of the one-dimensional return, which is the $\W2$-optimal matching in $d=1$.
For a single Euler step of the transport the gradient of \eqref{eq:critic} is, by the chain rule, $-2\veps\sum_iJ_i\T\hat\phi^*(z_i^\theta)$ with the Jacobians $J_i$ of the critic's particles, the direction of amortized SVGD \citep{feng2017learning}.
The policy minimizes $\Ep{x}{\Ep{a\sim\pi(\cdot|x)}{\alpha\log\pi(a|x)-Q(x,a)}}$ with $Q$ the smaller of the two critics' particle means, the entropy coefficient $\alpha$ is adapted as in SAC, and every head uses the same running observation normalization.
HalfCheetah uses this default configuration; on Walker2d, Hopper and Ant moment matching takes its mean from the bootstrap targets of the sampled transitions instead of the moment head, which gives the critic the temporal-difference signal directly.
A variant with a log-variance moment head removes the floor of the linear variance head that made one Ant seed diverge (\cref{app:res-mujoco}).

\subsection{Hyperparameters}\label{app:hparams}
Configurations were chosen by their performance on tuning seeds: five for the synthetic MDPs, the chain and resource gathering, and two for the rare-event study (\cref{app:protocols}).
Every reported tabular, chain and rare-event number comes from seeds that the selection did not use.
For MuJoCo, SVDRL uses the default configuration on HalfCheetah and the variant that takes the moment-matching mean from the bootstrap targets on Walker2d, Hopper and Ant. On Walker2d and Ant both variants ran on five seeds, and the one with the higher IQM return of \cref{tab:mujoco} was declared and extended to ten seeds; the other is reported in \cref{tab:mujoco_returns,tab:mc_full}. Hopper ran the declared variant only.
\Cref{tab:hparams} lists the most important settings of the reported runs per experiment, read from their configurations.
All agents are implemented in JAX \citep{jax2018github} with Equinox \citep{kidger2021equinox}, and jaxtyping \citep{kidger2022jaxtyping} checks the shapes and dtypes of their arrays.
Code that reproduces every run will be released upon acceptance.
\begin{table*}[t]
  \centering
  \caption{\emph{Hyperparameters.} Settings of the reported runs per experiment, read from their configurations. Configurations: the number of distinct ones, tuning and ablations included; a list gives the values they take; --: not a setting of the method. An annealing start at the highest level runs the whole schedule.}
  \label{tab:hparams}
  \small
  \setlength{\tabcolsep}{4pt}
  \adjustbox{max width=\textwidth}{%
  \begin{tabular}{lc*{6}{c}}
  \toprule
   &  & Synthetic & Rare &  & Resource &  & Pixel \\
   &  & MDPs & events & Chain & gathering & MuJoCo & maze \\
  \midrule
  \multicolumn{8}{l}{\emph{SVDRL}} \\
  Configurations &  & 62 & 51 & 14 & 18 & 9 & 6 \\
  Particles & $M$ & $64d$${}^{a}$ & $32$, $128$, $512$, $2048$ & $128$ & $32$, $128$, $512$ & $64$ & $200$ \\
  Sampler draws & $N$ & $4096$ & $10^5$ & $128$, $4096$ & $4096$, $10^5$ & -- & -- \\
  Resampled fraction & $\rho$ & $0$, $1$ & $0$, $0.5$ & $0$, $1$ & $0$, $1$ & -- & -- \\
  Relaxation & $\alpha$ & $0$, $10^{-3}$, $0.1$, $1$ & $1$ & $0$, $0.1$, $1$ & $10^{-3}$, $1$ & -- & -- \\
  Annealing start & $\sigma_{\mathrm{start}}$ & $1$, $5$ & $0.6$, $5$ & $1$, $5$ & $1$, $5$ & $5$ & -- \\
  SVGD steps per level & $T_\ell$ & $10$ & $10$ & $10$ & $10$ & $2$ & $10$ \\
  Annealing levels &  & $10$ & $10$ & $10$ & $10$ & $10$ & $10$ \\
  Level range & $\sigma_\ell$ & $[0.1,5]$ & $[0.1,5]$ & $[0.1,5]$ & $[0.1,5]$ & $[0.1,5]$ & $[0.1,5]$ \\
  Consistency weight & $\lambda_{\mathrm{cons}}$ & $1$ & $1$ & $1$ & $1$ & $1$ & $0$ \\
  Score learning rate &  & $10^{-3}$ & $10^{-3}$ & $10^{-3}$ & $10^{-3}$ & $10^{-3}$ & $10^{-3}$ \\
  Score batch size &  & $32$, $256$ & $32$, $256$ & $32$ & $256$ & -- & -- \\
  \midrule
  \multicolumn{8}{l}{\emph{Competitors}} \\
  Methods &  & 10 & 7 & 6 & 7 & 4 & 3 \\
  Configurations &  & 191 & 115 & 100 & 51 & 16 & 9 \\
  Particles & $M$ & $64d$${}^{a}$ & $32$, $128$, $512$, $2048$ & $128$, $512$, $2048$, $4096$ & $32$, $128$ & $64$ & $200$ \\
  Quantiles &  & $64$ & $1000$ & -- & $1000$ & $32$ & -- \\
  C51 atoms &  & $101$ & -- & -- & -- & -- & -- \\
  Learning rate${}^\dagger$ &  & $2.5\cdot10^{-3}$ to $0.3$${}^{b}$ & -- & $2.5\cdot10^{-3}$ & -- & -- & -- \\
  \midrule
  \multicolumn{8}{l}{\emph{All methods}} \\
  Discount & $\gamma$ & $0.95$ & $0.95$ & $0.9$ & $0.95$ & $0.99$ & $0.99$ \\
  Batch size &  & -- & -- & -- & -- & $256$ & -- \\
  Critic learning rate &  & -- & -- & -- & -- & $3\cdot10^{-4}$ & $10^{-5}$ \\
  \bottomrule
  \end{tabular}}
  \par\smallskip
  \parbox{\textwidth}{\footnotesize ${}^{a}$\,$M=64d$ on an MDP of dimension $d$; the compute-matched comparison adds $M\in\{512,4096\}$ and the dimension sweep $M\in\{140,200,520\}$. ${}^{b}$\,The values $\{2.5\cdot10^{-3},3\cdot10^{-3},5\cdot10^{-3},0.01,0.03,0.1,0.3\}$. ${}^\dagger$\,The variants of \citet{debes2026multivariate}. Multi-objective MuJoCo runs IQN with $32$ quantiles and the other heads with $M=64$ particles, and the defaults of the code otherwise.}
\end{table*}

\section{Protocols}\label{app:protocols}

This section specifies the environments, the statistics and the metrics of \cref{sec:experiments}, and the measured compute. Gymnasium and MO-Gymnasium are released under the MIT license and MuJoCo under the Apache~2.0 license; we reimplement the chain of \citet{debes2026multivariate} and the maze of \citet{zhang2021distributional} from their published definitions.

\paragraph{Synthetic MDPs.}
All synthetic MDPs have $\gamma=0.95$ and Gaussian rewards, deterministic at the root states, so their return laws are Gaussian mixtures.
MDP1 has a single absorbing state with reward $\Norm(1,0.5^2)$.
In MDP2 the root has reward $0$ and branches with probabilities $0.3$ and $0.7$ into absorbing states with rewards $\Norm(1,0.3^2)$ and $\Norm(2,0.4^2)$.
MDP3 is two-dimensional with three modes, reached from the root with probabilities $0.4$, $0.35$ and $0.25$.
MDP4 is five-dimensional with three modes reached with the same probabilities, whose mean rewards are $3$ times standard normal draws and whose covariances have eigenvalues drawn uniformly from $[0.05,0.5]$ in a random orthonormal basis, all drawn with the fixed random key $42$, which is part of the specification; the dimension sweep builds the MDP4 family in the same way for $d\in\{2,3,5,8,10\}$.

MDP5 to MDP8 have one decision state whose actions lead to absorbing reward channels with i.i.d.\ Gaussian rewards from $t=0$ on, so each channel's return is Gaussian with $1/(1-\gamma)$ times the mean and $1/(1-\gamma^2)$ times the covariance of its reward.
MDP5 ($d=2$) has four deterministic actions with mean rewards $(3,0.5)$, $(2,2)$, $(0.5,3)$ and the dominated $(1,1)$, all with covariance $0.04\,\Id$.
MDP6 ($d=2$) has a safe action with mean reward $(2,1)$, a risky action that reaches $(4,0.8)$ or $(0.2,1)$ with probability $\frac12$ each and variance $0.09$ in the first objective, and an action with $(0.5,3)$; all other variances are $0.04$.
MDP7 ($d=3$) has five deterministic actions with mean rewards $(3,0.5,0.5)$, $(0.5,3,0.5)$, $(0.5,0.5,3)$, $(2,2,0.3)$ and $(0.3,2,2)$ and full covariances, with diagonal entries $0.03$ to $0.06$ and correlations of both signs.
MDP8 ($d=2$) has two actions with the same mean reward $(2,2)$ and the covariances $\begin{psmallmatrix}0.5&\mp0.45\\\mp0.45&0.5\end{psmallmatrix}$, hedged and concentrated, and a low-variance action with mean reward $(0.5,3)$ and covariance $0.04\,\Id$.
The multimodal evaluated state--action pairs are those of MDP2, MDP3, MDP4 and MDP6.
Every run takes $8000$ Bellman steps on MDP1 and MDP2 and $16\,000$ on the other MDPs, with an evaluation every $800$ steps up to $80\,\%$ of training and every $100$ steps afterwards.
\paragraph{Choice of the Synthetic MDPs.}
Each MDP probes one property of the return law on which the critic families of \cref{tab:capabilities} differ, and every return law is a Gaussian mixture, so every error is measured against the exact law. MDP1 and MDP2 are one-dimensional. MDP1 bootstraps through a self-loop, where any per-update bias in the spread of the cloud is amplified by the factor $1/(1-\gamma^2)$, and MDP2 has two modes with unequal masses and widths. MDP3 and MDP4 have three separated modes in two and five dimensions, with masses $0.4$, $0.35$ and $0.25$, which $M=64d$ particles hold exactly only at $d=5$ and $d=10$, so the MDP4 family separates the resolution floor of \cref{prop:resolution} from the growth of the error with the dimension. MDP5 to MDP8 have one decision state each and differ in what separates its actions. In MDP5 three Pareto-optimal actions and a dominated one differ only in their means, the baseline for decisions. In MDP6 a safe and a risky action have almost the same mean, $(2,1)$ against $(2.1,0.9)$, and differ in their tails. MDP7 has correlated objectives in three dimensions. In MDP8 two actions have the same mean and the same marginals and differ only in the sign of the correlation, which no model of the per-dimension marginals can resolve. Together the eight MDPs cover unimodal laws, on which the within-mode shape decides (MDP1, MDP5, MDP7 and MDP8), and separated modes, on which the mode masses decide (MDP2, MDP3, MDP4 and MDP6), in one to five dimensions and up to ten in the dimension sweep.

Gaussian rewards make the returns unbounded, so \cref{asm:mdp} holds for the synthetic MDPs only up to a truncation.
The return box of \cref{app:method} keeps every channel's return within $5\big((1+\gamma)/(1-\gamma)\big)^{1/2}$ standard deviations of its mean in each coordinate, about $31$ at $\gamma=0.95$, so the true laws and their truncations, which the theory covers, coincide for every practical purpose.

\paragraph{Rare-Event Chain.}
The rare-event chain has $d=2$, $\gamma=0.95$ and i.i.d.\ Gaussian per-step rewards with covariance $0.04\,\Id$ in every channel.
At the root, the risky action reaches a bulk channel with mean reward $(3,1)$ with probability $1-p$ and a rare channel with mean reward $(-4,2)$ with probability $p$, and the safe action reaches a channel with mean reward $(2.5,1.5)$.
The rare channel is adverse in the first coordinate and shifted in the second, so conditional queries on the catastrophe are not trivial.
The branching state is preceded by $D\in\{0,2\}$ deterministic transient states with reward $0$, so for $D=2$ the rare branch is decided two bootstrapping steps below the root, whose return law is $\gamma^D$ times that of the branching state.
At $D=0$ the return means are $(60,20)$ for the bulk, $(-80,40)$ for the rare and $(50,30)$ for the safe channel, with standard deviation $0.64$ per coordinate.
We use $p\in\{0.002,0.005,0.01\}$, so $p<1/M$ already at $M=128$, and the event $Z_1\leq-10\gamma^D$, between the bulk and the catastrophe.
The ground truth is in closed form and agrees with $2\cdot10^5$ Monte Carlo returns.
Every run takes $16\,000$ Bellman steps for $D=0$ and $28\,800$ for $D=2$, with $50$ evaluations.

\paragraph{Chain.}
The chain of \citet{debes2026multivariate} has $K=10$ reward dimensions, forward probability $0.9$, $\gamma=0.9$ and the rewards $e_{S'}+\sigma_r\varepsilon$ with $\sigma_r\in\{0,0.05,0.2,0.5\}$.
Every method uses $M=128$ particles and near-exact TD targets for $5\cdot10^4$ updates, and the ground truth uses $10^6$ Monte Carlo rollouts per start state.
The six variants of \citet{debes2026multivariate} run on their own target sets and on ours, with RAdam at learning rate $2.5\cdot10^{-3}$, gradient clipping at $0.1$ and $64$ projections; with more particles, sliced $\W2$ and sliced Cram\'er at $M\in\{512,2048,4096\}$ and MMD at $M=512$, for $\sigma_r\in\{0.05,0.2\}$.
SVDRL replaces the cloud by the stratified target set with $\alpha=0.1$ and trains the score network of \cref{app:method} alongside.

\paragraph{Resource Gathering.}
The resource-gathering environment of MO-Gymnasium \citep{felten_toolkit_2023} is a $5\times5$ grid with $d=3$, which we reimplement with the attack probability as a parameter, $p\in\{0.1,0.02,0.005\}$, where the package fixes $0.1$.
We use $\gamma=0.95$ and three fixed policies from value iteration (safe, mixed and risky), and the ground truth comes from exact enumeration.
SVDRL runs its single tabular configuration without tuning at $M\in\{32,128\}$, reported as particles, sampler and hybrid, and the rare-event configuration with $N=10^5$ sampler draws; the particle critics run at $M\in\{32,128\}$ (Sinkhorn at $32$ only), and a quantile critic on the oracle scalarization with its learning rate tuned on the five tuning seeds.
Every run takes $15\,680$ Bellman steps with $50$ evaluations.

\paragraph{MuJoCo.}
HalfCheetah, Walker2d, Hopper and Ant (v5) \citep{towers2024gymnasium} run for $10^6$ environment steps with $\gamma=0.99$ and $M=64$ particles, with an evaluation of ten deterministic episodes every $10^4$ steps.
The heads are SAC and SAC with IQN, MD3QN, energy MMD and SVDRL critics on the same trunk.
For the critic shape, the policy at $10^6$ steps is frozen, and $64$ fixed probes each receive $1000$ Monte Carlo rollouts of $1000$ steps of the soft return; on Walker2d only the probes that restart exactly are used.

\paragraph{Statistics and Selection.}
In the tabular experiments and on the chain a seed's error is its mean over the evaluations in the last $20\,\%$ of training (on the chain, the last eight of $40$), and we report the interquartile mean (IQM) or the median over seeds with $95\,\%$ percentile-bootstrap intervals.
Configurations are selected on tuning seeds and reported on final seeds that the selection did not use: five tuning and ten final seeds for the synthetic MDPs, the chain and resource gathering, and two tuning and ten final seeds for the rare-event study, which adds ten fresh seeds.
On the synthetic MDPs and the chain a comparison is a win or a loss when the unpaired Welch $t$ statistic of SVDRL's errors against the competitor's is at most $-2$ or at least $2$, and a tie otherwise; the best competitor is the row with the lowest seed mean, per MDP, block and framing.
In the rare-event study SVDRL is counted as more accurate when its IQM error is smaller than that of the best of the six particle critics at the same $M$, answering from its particles or from a Gaussian kernel density estimate on them (Scott's bandwidth), without a significance test.
The three framings of the synthetic MDPs compare SVDRL with the competitors on their own target sets, with every row including the competitors on our target set (marked ${}^{*}$), and with the best competitor at $M\in\{512,4096\}$ on the four unimodal MDPs (\cref{tab:cmtab}).
No competitor is limited by the training budget: on the eight MDPs none improves significantly over the last $20\,\%$ of its updates, and SVDRL reaches its plateau early in training.
On MuJoCo a run's return is its mean evaluation return over $800$k to $1$M steps ($21$ evaluations), and we report per environment and head the IQM over five to ten seeds with a $95\,\%$ percentile-bootstrap interval; a run stopped by the divergence guard counts with the evaluations it logged.

\paragraph{Metrics.}
On the synthetic MDPs, mass is the total variation between the mode masses of a cloud and the ground-truth weights, next to its floor $\min_k\frac12\sum_i|k_i/M-w_i|$ over integer allocations $k$, and shape is the within-mode error $\sum_iw_i\W2^2$, in closed form for $d=1$, by exact assignment against $4096$ ground-truth draws for $d=2$, and by quantile sliced $\W2^2$ with $128$ projections for $d>2$.
The full $\W2^2$ charges a mass error $\epsilon$ across a gap $\Delta$ as $\epsilon\Delta^2$, so on separated modes it measures mass rounding rather than shape (\cref{tab:fullw2,fig:dimsweep}).
Sliced $\W2$ has two independent estimators, the quantile estimator and a CDF-based estimator that integrates the analytic projected distribution functions of the ground truth \citep{vauthier2026highly}, and a disagreement between them fails the asset pipeline.
In the rare-event study the queries are CVaR down to $\alpha=0.001$, $|\log\hat P/P|$ of the event (with a floor of $10^{-6}$ on $\hat P$), chance-constrained decisions over thresholds $\delta$, conditional tails, entropic risk and the expected scalarized return, ten rare-event and two bulk query columns, i.e., $60$ rare-event and $12$ bulk comparisons over the six environments.
On the chain we measure $\W2$ by exact assignment, sliced and per-coordinate $\W2^2$, the mean error and the dispersion against $10^4$ fresh Monte Carlo rollouts, and the squared energy distance against the $10^6$-rollout bank.
On resource gathering we measure $|\log\hat P/P|$ of death, the error rates of the chance-constrained and the CVaR decisions, CVaR errors, the error of $\Ex{\text{treasure}\mid\text{alive}}$, the squared energy distance and sliced $\W2^2$, and $|\log\hat P/P|$ of the time-averaged $\hat P$.
On MuJoCo the critic shape is the $\W2$ distance between the critic's return law and the Monte Carlo law, each centered at its own mean, over the Monte Carlo standard deviation, exact in $d=1$; per run we take the median over the two critics and then over the probes, and per environment and head the IQM over seeds with a $5000$-resample bootstrap.
Every head's mean is biased by the twin-critic minimum, so the centered distance is the one that compares the heads.
Every evaluation draws fresh randomness, and a size mismatch between a cloud and its reference fails loudly instead of subsampling.

\subsection{Compute}\label{app:compute}
The Stein move costs at most $(L+1)T_\ell M$ score evaluations per state--action pair and update, plus $\calO(M^2)$ kernel evaluations per step, independently of the horizon.
A sampler run of $N$ draws costs $NN_{\mathrm g}$ score evaluations, vectorized over the draws, and it runs only at evaluation time and inside sampler bootstrapping.
Acting takes one forward pass of the critic, which returns the $M$ Stein particles.
Every GPU run used an NVIDIA GeForce RTX 4090, and training runs in single precision.

\section{Additional Results}\label{app:results}

This section gives the full tables and figures behind \cref{sec:experiments}, in its order, and the checks of the theory on trained networks.

\subsection{Rare Events}\label{app:res-rare}

\Cref{tab:re1main} lists every estimator behind the comparisons of \cref{sec:experiments} at $p=0.005$ and $D=2$.
SVDRL's sampler is more accurate than every particle critic in every column, from their atoms or from a kernel density estimate.
The quantile critic on the oracle scalarization matches it at CVaR$_{0.01}$ of that scalarization but not at CVaR$_{0.001}$, and it answers no other scalarization.
SVDRL's own Stein particles share the floor of the particle critics, so the gain comes from the sampler.

\begin{table*}[t]
  \centering
  \caption{\emph{Rare-event queries.} Rare-event chain, $p=0.005$, $D=2$, $M=128$: IQM over the ten final seeds with its $95\,\%$ interval below; bold: inside the interval of the best IQM. Queries: CVaR of the first coordinate and of the equal-weight scalarization, the catastrophe's probability, the conditional means of $Z_2$ given $Z_1\leq\tau$ and at the rare mode, the entropic risk, and the accuracies of the chance-constrained decision and of the decision by expected scalarized return (ESR). SVDRL's sampler answers with $N=10^5$ draws, and Langevin dynamics on the partial score answers the conditional mean at the rare mode (quadrature of the energy: 1.7 for particle bootstrapping); ${}^{\ast}$: the reported configuration. Particle critics answer from their atoms or from a kernel density estimate (KDE) on them.}
  \label{tab:re1main}
  \small
  \adjustbox{max width=\textwidth}{\input{tables/tab_re1_appendix.tex}}
\end{table*}

\begin{figure*}[t]
  \centering
  \includegraphics{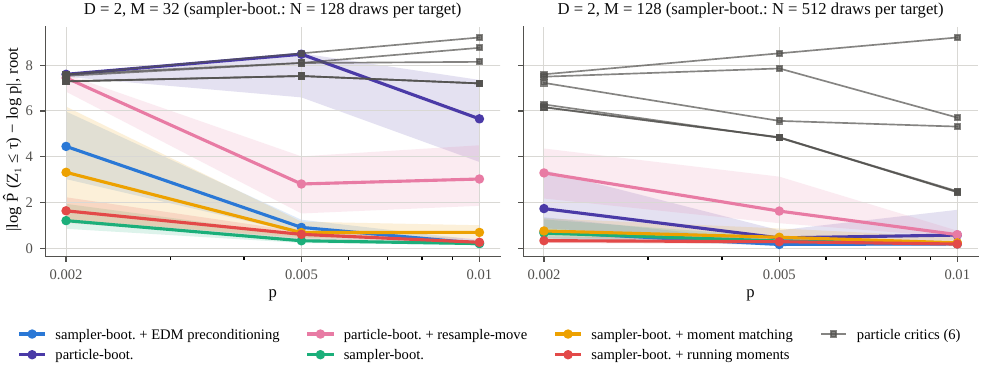}\\[2pt]
  \includegraphics{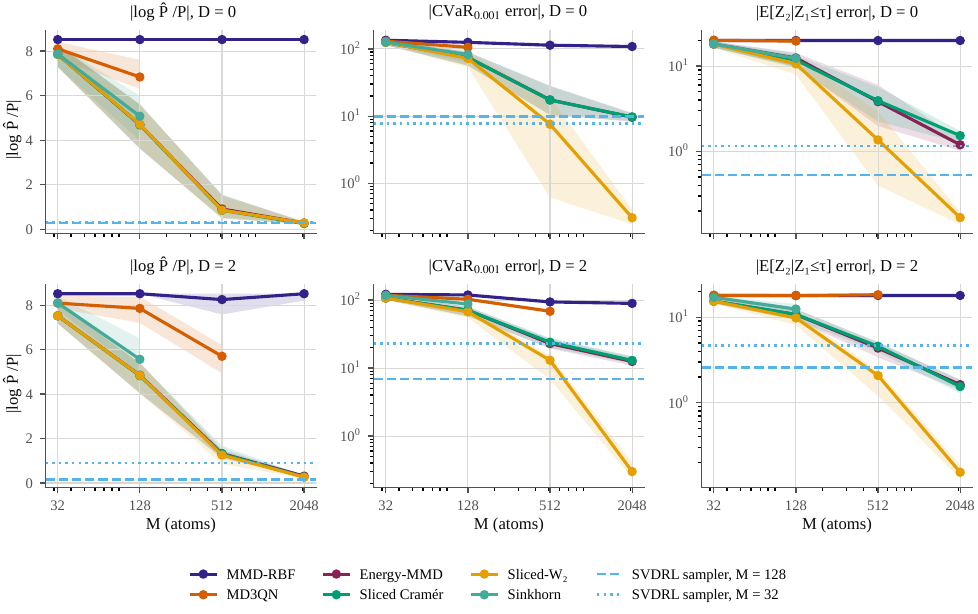}
  \caption{\emph{Propagation and particle count.} Rare-event chain, final seeds. (a) Error of the catastrophe's probability at the root against $p$ for $D=2$, with $M=32$ and $M=128$, for SVDRL's variants and the six particle critics: with sampler bootstrapping in the reported configuration (moment-matched, EDM preconditioning) the catastrophe survives the recursion for $p\geq0.005$, while particle bootstrapping and the particle critics lose it. (b) Errors of the particle critics against their particle count $M$ at $D=0$ (top) and $D=2$ (bottom), with SVDRL's sampler at $M=128$ and $M=32$ as horizontal lines: sliced $\W2$ needs about $10/p=2048$ particles to resolve the location and the conditional tail of the catastrophe.}
  \label{fig:capabilities}
\end{figure*}

\begin{figure*}[t]
  \centering
  \includegraphics{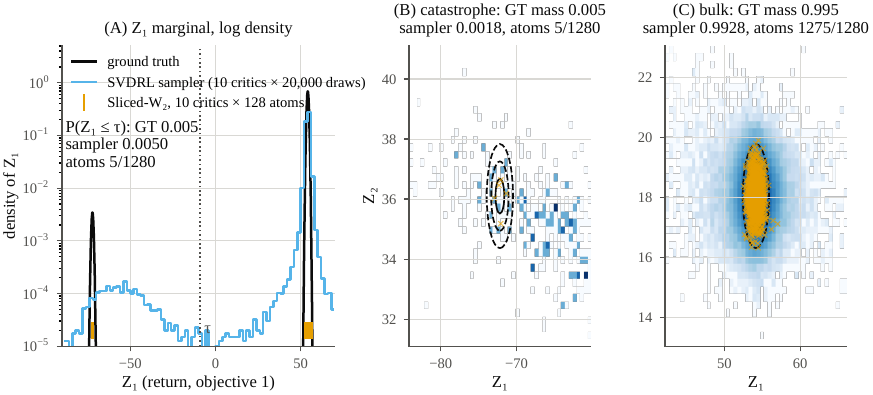}
  \caption{\emph{Return law of the risky root action.} In the rare-event chain, a catastrophe of mass ${p=0.005<1/M}$ lies two steps below the root. (A) Log density of the first coordinate: the true law, the sampler law of ten SVDRL critics and the atoms of ten sliced $\W2$ critics ($M=128$). (B) The catastrophe and (C) the bulk. The sampler gets the probability of the catastrophe right but spreads the rare mode too wide; particles can only hold whole atoms of mass $1/M$ there.}
  \label{fig:rarelaw}
\end{figure*}

Particle critics resolve the catastrophe only with about $10/p$ particles (\cref{fig:capabilities}b).
With $M=2048$ sliced $\W2$ reaches the CVaR$_{0.001}$ error 0.30, below SVDRL's 6.9, and answers the conditional tail more accurately, while SVDRL keeps the more accurate mass, chance-constrained decisions and bulk (\cref{sec:experiments}).
SVDRL's CVaR error at the lowest levels reflects the width of the far rare mode (\cref{fig:rarelaw}), whose mass the sampler gets right but which the score network spreads too wide, 13 to 15 times in a static fit to draws of the true law (\cref{app:limits}).

The catastrophe survives the recursion only under sampler bootstrapping (\cref{fig:capabilities}a), as \cref{prop:resolution-full}(v) predicts: a target built from downstream particles gives the catastrophe a multiple of $1/M$, which is zero for $p<1/M$ in most updates, whereas a target built from the downstream sampler inherits its event error.
At $p=0.002$ and $M=32$ the reported configuration loses the catastrophe as well.

\subsection{Eight Synthetic MDPs}\label{app:res-mdp}

\Cref{tab:sel1objects} separates the objects behind the hybrid of \cref{tab:main}.
The hybrid reports the sampler draws on the unimodal MDPs and mostly the particles on the separated modes of MDP2, MDP3, MDP4 and MDP6, where the particles have the lower mass error.
The particles are also better in shape on MDP4 and MDP6, whereas on MDP2 and MDP3 the sampler would be.
The separation test alone ($r>6$) makes the same choices on every MDP.

\begin{table*}[t]
  \centering
  \caption{\emph{The objects behind the hybrid.} Synthetic MDPs, final seeds: every SVDRL object behind \cref{tab:main}, per MDP and error (mass: total variation of the mode masses; shape: within-mode $\W2^2$), in units of $10^{-3}$. The hybrid, the sampler ($N=4096$), the particles ($M=64d$), the object chosen on the tuning seeds, and the hybrid with the separation test $r>6$ alone, each with its ratio to the best competitor of all rows, including those on our target set; bold: a win, underlined: a tie by Welch's test; --: the competitor's error is zero. The hybrid reports the particles where their separation ratio exceeds $6$ after merging particles closer than the smallest noise level, or where a cluster of at least two particles is narrower than it, and the sampler draws otherwise; the last column is the object it reports. Wins/ties/losses over the MDPs against the competitors as specified (own) and including those on our target set (all).}
  \label{tab:sel1objects}
  \small
  \adjustbox{max width=\textwidth}{\input{tables/tabA_sel1_compact.tex}}
\end{table*}

\Cref{tab:mainmass} gives the total variation of the mode masses on the four MDPs with separated modes, next to the within-mode shape of \cref{tab:main}.

\begin{table}[t]
  \centering
  \caption{\emph{Mode masses on the synthetic MDPs.} Total variation of the mode masses on the four MDPs with separated modes, in units of $10^{-3}$, averaged over the last $20\,\%$ of training; floor: the $1/M$ limit of the mass error. Ratios to the best competitor as in \cref{tab:main}. Underlined: a tie by Welch's test; --: the competitor's error is zero.}
  \label{tab:mainmass}
  \small
  \input{tables/tab1_mass.tex}
\end{table}

\Cref{tab:tab1comp} names the best competitor behind each ratio of \cref{tab:main,tab:mainmass}, with its target set, learning rate and error.

\begin{table}[t]
  \centering
  \caption{\emph{Best competitors.} The competitor behind each ratio of \cref{tab:main,tab:mainmass}, per MDP, error and comparison, with its error in units of $10^{-3}$. Own: the competitors as specified; all: also the competitors given our target set; $M=4096$: the compute-matched comparison on the unimodal MDPs. $\dagger$: an objective of \citet{debes2026multivariate}, implemented in our agent; base: the learning rate of the competitor's base configuration (\cref{tab:hparams}).}
  \label{tab:tab1comp}
  \small
  \input{tables/tab1_competitors.tex}
\end{table}

On the unimodal MDPs the hybrid's wins come from $N=4096$ sampler draws, and they vanish at matched compute (\cref{tab:cmtab}): competitors with $M=4096$ particles on our target set beat the sampler on all four.
SVDRL's own particles without the Stein move beat the sampler at $M=4096$ as well, so at equal size the particle sets win on unimodal laws.

\begin{table*}[t]
  \centering
  \caption{\emph{Compute-matched comparison.} Within-mode shape on the unimodal MDPs: the hybrid's sampler ($4096$ draws) against the best competitor on our target set and against SVDRL's transport-free particles ($\alpha=0$), at $M\in\{64d,512,4096\}$; ${}^\dagger$: an objective of \citet{debes2026multivariate}, implemented in our agent. Error mean $\pm$ standard error over the final seeds, and the hybrid's ratio to each row with its Welch $t$ (negative: the hybrid is better); bold: the lowest mean error of each MDP.}
  \label{tab:cmtab}
  \adjustbox{max width=\textwidth}{\input{tables/tabA_cmtab.tex}}
\end{table*}

\Cref{tab:fullw2} reports the full $\W2^2$, which charges mass and shape together; on separated modes it is dominated by mass rounding, which every method shares (\cref{fig:dimsweep}), so \cref{tab:main} reports mass and shape apart.

\begin{table}[t]
  \centering
  \caption{\emph{Mass and shape together.} Full $\W2^2$ per MDP: SVDRL's selected and single configuration against the best competitor on its own target set and among every row, with the selected configuration's ratio and the Welch decision; ${}^\dagger$: an objective of \citet{debes2026multivariate}, implemented in our agent, with its learning rate where it is not the base one; ${}^{*}$: on our target set.}
  \label{tab:fullw2}
  \adjustbox{max width=\linewidth}{\input{tables/tabA_fullw2.tex}}
\end{table}

Mass rounding also decides the full error across the reward dimension (\cref{fig:dimsweep}).
The MDP4 family has the mode masses $0.40$, $0.35$ and $0.25$ for every $d$, which are multiples of $1/M$ for $M=64d$ only at $d=5$ and $d=10$, and exactly there the full error of SVDRL and of the competitors on our target set is small.
SVDRL's ratio to the best competitor shows no trend with $d$ (\cref{fig:dimsweepratio}).

\begin{figure*}[t]
  \centering
  \includegraphics{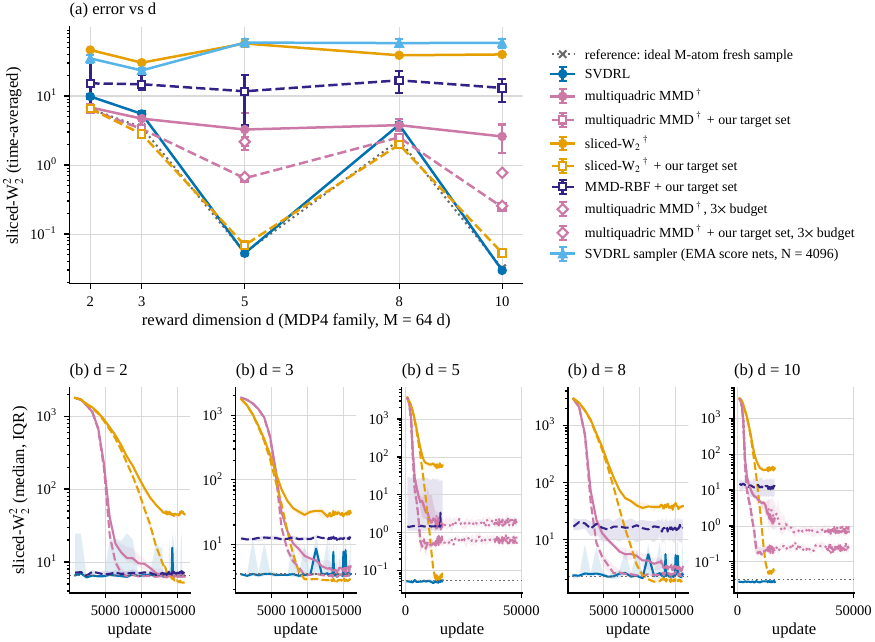}
  \caption{\emph{Scaling with the reward dimension.} MDP4 family across the reward dimension $d$, $M=64d$, ten final seeds. (a) Time-averaged sliced $\W2^2$ against $d$ for SVDRL and the competitors on their own and on our target set. (b) Learning curves per $d$, median and interquartile range. The full error is small where the mode masses are multiples of $1/M$, at $d=5$ and $d=10$, and dominated by mass rounding elsewhere. ${}^\dagger$: an objective of \citet{debes2026multivariate}, implemented in our agent; $3\times$ budget: three times the training steps.}
  \label{fig:dimsweep}
\end{figure*}

\begin{figure*}[t]
  \centering
  \includegraphics{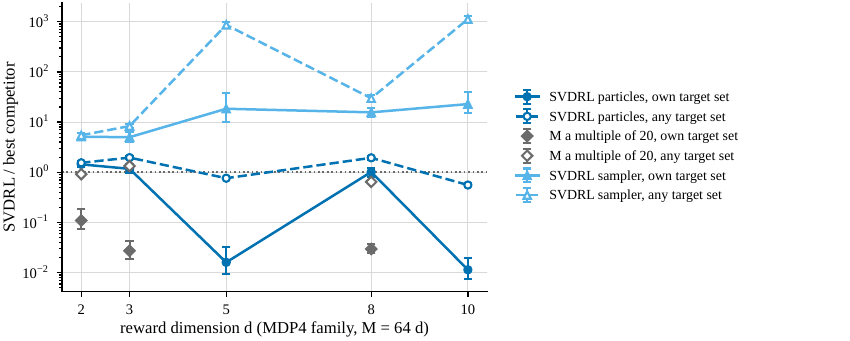}
  \caption{\emph{SVDRL against the best competitor across the reward dimension.} MDP4 family, $M=64d$, ten final seeds: SVDRL's error over the best competitor's (ratio of means with its $95\,\%$ interval; below one, SVDRL is better), for its particles and its sampler, against the competitors on their own target set and on any target set. $M$ a multiple of $20$: the allocation control, at which every mode mass is a multiple of $1/M$.}
  \label{fig:dimsweepratio}
\end{figure*}

The relaxation $\alpha$ of the Stein move and the stratified target set are ablated on every MDP in \cref{tab:alpha}, and the measured gap $\veps_{\mathrm{St}}$ of \cref{thm:e2e} in \cref{tab:epsst}.
On the unimodal MDPs $\veps_{\mathrm{St}}$ equals the particles' true error within 1 to 12\,\%; on separated modes it flags a disagreement between the particles and the sampler but can exceed both of their errors.

\input{tables/tabA_alpha.tex}
\input{tables/tabA_epsst.tex}

The Stein move is specific where the selected configuration uses it (\cref{tab:rj1}): Langevin dynamics along the same score, a Gaussian jitter of the same size and no move are all worse, with every interval of the ratio above one.
Without repulsion the cloud collapses onto few distinct particles on MDP5 and MDP8.

\begin{table}[t]
  \centering
  \caption{\emph{Stein move ablation.} The Stein move against other moves on MDP1, MDP5 and MDP8, whose selected configurations use it: Langevin dynamics along the same score, a Gaussian jitter of matched size, the Stein move without repulsion and no move. Within-mode shape error (mean $\pm$ standard error over the final seeds), its ratio to the Stein move with the Welch $t$, the fraction of distinct particles, and the nearest-neighbor distance over the cloud's standard deviation.}
  \label{tab:rj1}
  \adjustbox{max width=\linewidth}{\input{tables/tabA_rj1.tex}}
\end{table}

Removing the consistency loss raises the error significantly on MDP1, MDP3, MDP4 and MDP8 and never lowers it significantly (\cref{tab:cons}), in line with the lower variance of the consistency target at small noise levels (\cref{rem:consvar}).

\begin{table}[t]
  \centering
  \caption{\emph{Consistency ablation.} Full $\W2^2$ of SVDRL with the fully annealed transport, with ($\lambda_{\mathrm{cons}}=1$) and without ($\lambda_{\mathrm{cons}}=0$) the consistency loss, mean over ten seeds, and the paired difference with its $95\,\%$ Student-$t$ interval; exact $\W2^2$ for $d\leq2$ and CDF-based sliced $\W2^2$ for MDP4 and MDP7. A negative difference favors the consistency loss.}
  \label{tab:cons}
  \small
  \begin{tabular}{lrrr}
    \toprule
    MDP & $\lambda_{\mathrm{cons}}=1$ ($\downarrow$) & $\lambda_{\mathrm{cons}}=0$ ($\downarrow$) & difference [$95\,\%$ interval] \\
    \midrule
    MDP1 & 0.136 & 0.162 & \num{-0.027} $\sd{\text{[\num{-0.053}, \num{-0.0009}]}}$ \\
    MDP2 & 10.1 & 9.4 & 0.67 $\sd{\text{[\num{-1.65}, 2.99]}}$ \\
    MDP3 & 67.4 & 101.2 & \num{-33.8} $\sd{\text{[\num{-57.9}, \num{-9.6}]}}$ \\
    MDP4 & 230.8 & 270.5 & \num{-39.7} $\sd{\text{[\num{-70.4}, \num{-9.0}]}}$ \\
    MDP5 & 0.0421 & 0.0428 & \num{-0.0006} $\sd{\text{[\num{-0.0036}, 0.0023]}}$ \\
    MDP6 & 71.8 & 103.5 & \num{-31.7} $\sd{\text{[\num{-102.8}, 39.3]}}$ \\
    MDP7 & 0.0241 & 0.0242 & \num{-0.0001} $\sd{\text{[\num{-0.0003}, 0.0001]}}$ \\
    MDP8 & 1.21 & 2.12 & \num{-0.91} $\sd{\text{[\num{-1.05}, \num{-0.78}]}}$ \\
    \bottomrule
  \end{tabular}
\end{table}

\subsection{Chain}\label{app:res-chain}

On the atomic chain, $\sigma_r=0$, the persistent particles of \citet{debes2026multivariate} are the most accurate in $\W2$ (\cref{tab:chain}), as \cref{lem:scorelip} predicts, and the sampler's Euler steps escape near the atoms, where the score is stiff (\cref{proof-anatomy}).
With reward noise the moment-matched sampler with $4096$ draws has the lowest energy distance and mean error at every $\sigma_r\geq0.05$.
The stratified target set is the first component that transfers: on it, every variant of \citet{debes2026multivariate} has a lower energy distance and mean error than on its own target set at every noise level.

\begin{table*}[t]
  \centering
  \caption{\emph{Chain with reward noise.} The chain of \citet{debes2026multivariate} with $K=10$ reward dimensions and reward noise $\sigma_r$, near-exact TD, $M=128$ particles, final seeds: $\W2$ at the start state (exact assignment; not defined for $4096$ draws against the $128$-draw reference), squared energy distance (ED$^2$) in units of $10^{-3}$ and mean error in units of $10^{-2}$; bold: the lowest error of each column among the methods. SVDRL's particles, its moment-matched sampler with $4096$ and $128$ draws and its fully annealed transport; the six variants of \citet{debes2026multivariate} on their own target sets and on ours, and i.i.d.\ samples of the true law as references. All methods have a higher $\W2$ here than \citet{debes2026multivariate} publish at $\sigma_r=0$, because they use $M=500$: with $128$ particles against $128$ reference draws, even an i.i.d.\ sample of the true law lies above the published values.}
  \label{tab:chain}
  \small
  \adjustbox{max width=\textwidth}{\input{tables/tab2_chain_compact.tex}}
\end{table*}

With more particles the competitors win shape at matched compute (\cref{fig:nc1vsn}): with $M=4096$ their energy distances and $\W2^2$ errors lie below those of the sampler with $4096$ draws on either target set, while the sampler keeps the lower mean error.

Running-mean moment matching is the second component that transfers.
Given to the variants of \citet{debes2026multivariate} at evaluation, the shift of \cref{app:mm} to the running target mean lowers their mean error at every particle count and both noise levels, e.g., for sliced $\W2$ on its own target set with $M=2048$ from 0.020 to 0.0012 at $\sigma_r=0.05$ and from 0.025 to 0.0018 at $\sigma_r=0.2$, below the sampler's 0.0050 and 0.010.

\begin{figure*}[t]
  \centering
  \includegraphics{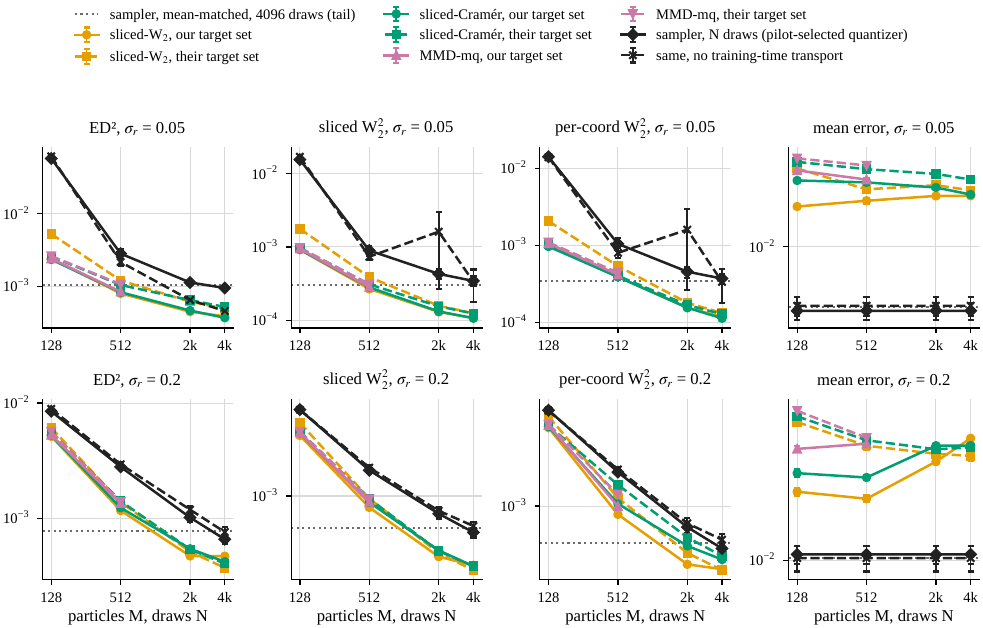}
  \caption{\emph{More particles for the competitors.} The chain with reward noise $\sigma_r=0.05$ (top) and $0.2$ (bottom): the two strongest variants of \citet{debes2026multivariate} and multiquadric MMD (up to $M=512$) with more particles, on their own and on our target set, at their specified learning rate; squared energy distance, sliced and per-coordinate $\W2^2$, and mean error against the particle count $M$ of the competitors and the number of draws $N$ of SVDRL's sampler. Dotted: SVDRL's moment-matched sampler with $4096$ draws; black: SVDRL's sampler with $N$ draws. Mean $\pm$ standard error over the final seeds.}
  \label{fig:nc1vsn}
\end{figure*}

\subsection{Resource Gathering}\label{app:res-rg}

\Cref{tab:morlrgfull} gives every method, particle count and estimator behind \cref{sec:experiments} at the three attack probabilities.
The returns are atomic, so SVDRL's sampler fails here, as \cref{lem:scorelip} predicts, and the hybrid rule reports the particles, whose stratified target set holds each atom at its mass up to $1/M$.
At $p=0.005<1/M$ no equal-weight set holds the death probability at a single evaluation, but averaged over the evaluations and seeds SVDRL's particle estimate is 0.0053 for $P=0.005$, since full replacement by the stratified target set rounds the mass without bias.

{\setlength{\LTcapwidth}{\textwidth}\setlength{\tabcolsep}{3pt}\renewcommand{\arraystretch}{0.95}%
\begin{longtable}{>{\small}l>{\small}c*{10}{>{\small}c}}
  \caption{\emph{Resource gathering.} Attack probability $p\in\{0.1,0.02,0.005\}$ (top to bottom), every method at each particle count $M$, median over the ten final seeds. Death probability: $|\log\hat P/P|$ of the final and of the time-averaged $\hat P$ (avg.); safe: the absolute error under the safe policy. CVaR: errors at the preferences $(1,1,1)$ and $(0,1,1)$. Decision: error rates of the chance-constrained and the CVaR decision. Treasure: error of $\Ex{\text{treasure}\mid\text{alive}}$. Law: squared energy distance (ED$^2$) and sliced $\W2^2$. The last three columns are in units of $10^{-3}$. SVDRL's rows give the hybrid and its two objects, the particles and the sampler ($N=4096$ draws); $r>6$ only: the hybrid rule without its sub-resolution clause; superseded${}^\dagger$: an earlier rule whose sub-resolution clause also counted single particles, kept for reference; rare event: the rare-event configuration, which bootstraps from its sampler. Particle critics answer from their atoms. QR: the quantile critic on the oracle scalarization, with the settings of the rare-event study or with its learning rate (lr) tuned on the tuning seeds.}\label{tab:morlrgfull}\\
  \toprule
  \multicolumn{2}{l}{\small Errors ($\downarrow$)} & \multicolumn{3}{c}{\small Death probability} & \multicolumn{2}{c}{\small CVaR} & \multicolumn{2}{c}{\small Decision} & \multicolumn{1}{c}{\small Treasure} & \multicolumn{2}{c}{\small Law} \\
  \cmidrule(lr){3-5}\cmidrule(lr){6-7}\cmidrule(lr){8-9}\cmidrule(lr){10-10}\cmidrule(lr){11-12}
  Method & $M$ & $|\log\frac{\hat P}{P}|$ & avg. & safe & $(1,1,1)$ & $(0,1,1)$ & chance & CVaR & if alive & ED$^2$ & $\W2^2$ \\
  \midrule
  \endfirsthead
  \caption[]{\emph{Resource gathering}, continued.}\\
  \toprule
  \multicolumn{2}{l}{\small Errors ($\downarrow$)} & \multicolumn{3}{c}{\small Death probability} & \multicolumn{2}{c}{\small CVaR} & \multicolumn{2}{c}{\small Decision} & \multicolumn{1}{c}{\small Treasure} & \multicolumn{2}{c}{\small Law} \\
  \cmidrule(lr){3-5}\cmidrule(lr){6-7}\cmidrule(lr){8-9}\cmidrule(lr){10-10}\cmidrule(lr){11-12}
  Method & $M$ & $|\log\frac{\hat P}{P}|$ & avg. & safe & $(1,1,1)$ & $(0,1,1)$ & chance & CVaR & if alive & ED$^2$ & $\W2^2$ \\
  \midrule
  \endhead
  \midrule
  \multicolumn{12}{r}{\small\emph{Continued on the next page}} \\
  \endfoot
  \bottomrule
  \endlastfoot
  \input{tables/tabA_morlrg_compact.tex}
\end{longtable}}

\subsection{MuJoCo}\label{app:res-mujoco}

\Cref{tab:mujoco_returns} gives the returns with their intervals, as the mean over $800$k to $1$M steps of \cref{tab:mujoco} and as the mean of the last five evaluations, and \cref{fig:mujoco} the learning curves.
Each curve point is a run's mean over its last five evaluations, because single evaluations on Hopper are bimodal for every head.
The divergence guard stopped one Ant seed of SVDRL when the linear variance of its moment head reached its floor; a log-variance head cannot reach the floor, and on HalfCheetah it trains without late collapses (the rows marked $^\ast$).

\begin{figure*}[t]
  \centering
  \includegraphics{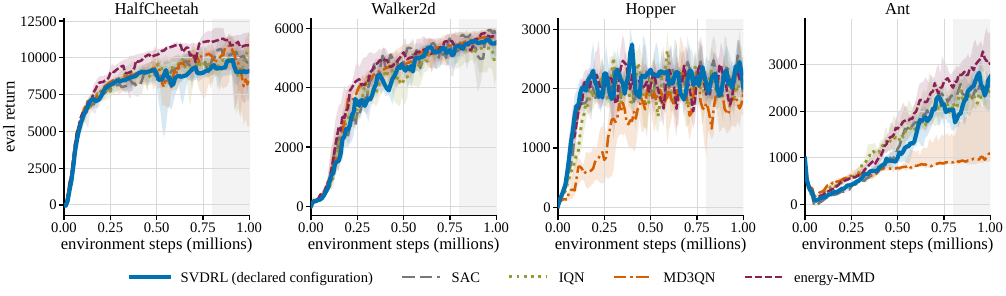}
  \caption{\emph{MuJoCo learning curves.} Evaluation return against environment steps, per run the mean of its last five evaluations, median over seeds with the interquartile band, every head on the same SAC trunk. SVDRL is the declared configuration per environment; the shaded strip marks the steps from $800$k to $1$M over which \cref{tab:mujoco} averages the return. Seeds per head as in \cref{tab:mujoco}; a run stopped by the divergence guard keeps its last evaluation.}
  \label{fig:mujoco}
\end{figure*}
\input{tables/tabA_mujoco_returns.tex}

\Cref{tab:mc_full} compares every head's critic with Monte Carlo returns of the frozen policy beyond the centered shape of \cref{tab:mujoco}: the ratio of standard deviations, the mean error, the full $\W2$ and the coverage of the central $80\,\%$ interval.
Every head's mean is biased by the twin-critic minimum, so only the centered metrics separate the heads, and the sampler row is a diagnostic, not a claim.
\begin{table*}[t]
  \centering
  \caption{\emph{MuJoCo critic accuracy.} Critic laws against frozen-policy Monte Carlo returns at $10^6$ steps, all metrics in raw return units relative to the run's own MC draws: shape = $\W2$ of the centered laws / MC std; std ratio (ideal 1); mean error (critic $-$ MC); full $\W2$ (CDF estimator, cross-checked by the quantile estimator) / MC std; coverage: the fraction of MC draws inside the critic's central 80\,\% interval, raw (80\,\% cov.) and with both means removed (centered), ideal 0.8. IQM over seeds, with the 95\,\% bootstrap interval of shape and full $\W2$ below it. Every head's mean is biased (twin-min targets), so full $\W2$ and raw coverage do not separate the heads. SVDRL (target mean): the variant that takes the moment-matching mean from the bootstrap targets. $^\ast$ The other declared SVDRL variant (appendix row). The sampler row (reverse diffusion from the saved score nets) is shown only for the log-variance head, the only configuration whose nets pass the label-free gate; it is a diagnostic, not a sampler claim. SAC: mean error only (scalar critic).}
  \label{tab:mc_full}
  \small
  \adjustbox{max width=\textwidth}{\input{tables/tabA_mc_compact.tex}}
\end{table*}

\subsection{Checks of the Theory}\label{app:res-theory}

We evaluate each bound of \cref{sec:theory} on the trained networks of MDP1--3, with the ground truth only to measure the left sides (\cref{tab:theoryarms,tab:consTDcheck}).
The checks of the consistency loss add the 3-cycle, a deterministic cycle of three states with rewards $1$, $-1$ and $0.5$ and $\gamma=0.9$, whose return laws are point masses.

The bound of \cref{thm:gen} holds in 80 of 80 (seed, state--action pair) combinations, but loosely: it exceeds the measured total variation 4.4 to 34 times and the measured $\W1$ by a median factor of 177, the gap between the two being the factor $D$ of $\W1\leq D\,\tv$ on the return box.
The discretization term is at least 67\,\% of the bound in every combination, as \cref{rem:laddernecessity} predicts for the point-mass targets that the particles produce.

Several noise levels are necessary (\cref{rem:laddernecessity}).
A network trained at the bottom level alone misallocates the mode masses, with errors of 0.32 on MDP2 and 0.46 on MDP3, against 0.056 and 0.099 for the multi-level network trained by DSM alone, and a single level of the order of the mode separation recovers the masses on MDP3 but not on MDP2.

Sampler bootstrapping without moment matching accumulates the smoothing of the sampler over the horizon (\cref{rem:e2ereadings}), and moment matching removes the accumulation: with it, sampler bootstrapping and sampler-seeded particles match the Stein particles on MDP1--3 in $\Wbar1$ and in the mode masses, all differences within their intervals (\cref{tab:theoryarms}).

The level-by-level bound of \cref{cor:consTD} holds at every level in 87 of 87 seeds (\cref{tab:consTDcheck}).
On the stochastic MDP1, consistency-only training with a lagged target reaches a lower multi-level risk than DSM alone, as \cref{lem:consU,rem:consvar} predict, and with the online network as its own target it diverges on the self-looping state--action pair, which is the lag condition after \cref{cor:scoreTDconv}.

\input{tables/tabA_theoryarms.tex}

\begin{table}[t]
  \centering
  \caption{\emph{Score-space temporal differences.} Level-by-level check of \cref{cor:consTD} on the trained networks: per variant, the seeds in which the one-step inequality \eqref{eq:consTD} and its iterated form hold at every level, and the ratio of the iterated bound to the measured normalized error at the highest, a middle and the lowest noise level. Normalized units, grid ratio $\kappa_{\mathrm g}=\gamma$.}
  \label{tab:consTDcheck}
  \adjustbox{max width=\linewidth}{\input{tables/tabA_consTD_check.tex}}
\end{table}

\section{Limitations in Detail}\label{app:limits}

SVDRL is not the most accurate critic on every return law, and this section collects the evidence for the limits of \cref{sec:conclusion}.
The sampler fails on atomic and widely separated laws, on the atomic chain (\cref{app:res-chain}) and on resource gathering (\cref{app:res-rg}), as \cref{lem:scorelip} predicts, and the hybrid rule reports the particles there.
At equal size or matched compute, fitted particle clouds beat SVDRL's sampler on the unimodal MDPs and on the chain with reward noise (\cref{tab:cmtab,fig:nc1vsn}).
The measured gap $\veps_{\mathrm{St}}$ is undefined for the rare-event configuration, which keeps no particles, and on separated modes it flags a disagreement but can exceed both errors (\cref{tab:epsst}).
In SAC-SVDRL the linear variance of the moment head reached its floor and made one Ant seed diverge, a floor that a log-variance head cannot reach (\cref{app:res-mujoco}), and we make no claim about a sampler on MuJoCo.

A light rare mode far from the bulk comes out too wide, an extrapolation limit of the score network where the data are sparse.
A static fit of the selected score network to fixed draws of the true root law of the rare-event chain, without bootstrapping, consistency loss or EMA, reproduces the width, with a standard deviation of the rare component 13 to 15 times the true one (\cref{tab:re1static}).

\begin{table}[t]
  \centering
  \caption{\emph{Width of the far rare mode.} Static fit of the score network to fixed draws of the root law of the rare-event chain ($p=0.005$, $D=2$), without bootstrapping, consistency loss or EMA, three seeds: mass, median offset and standard deviation of the rare component in return units, and the ratio of the standard deviation to the true one.}
  \label{tab:re1static}
  \input{tables/tabA_re1_static_fit.tex}
\end{table}

\subsection{Pixel Maze}\label{app:maze}

On the maze of \citet{zhang2021distributional} under the pixel protocol of \citet{debes2026multivariate}, the competitors' benchmark from pixels, SVDRL is third of four heads in every variant (\cref{tab:mz}).
The benchmark evaluates the uniform random policy from pixels with a PQN backbone in three variants, and every head is compared with Monte Carlo returns at the start state; SVDRL runs there without the consistency loss (\cref{tab:hparams}).
In SVDRL's first configuration nothing that trained the trunk depended on the transition, so the trunk was blind to the state and the action.
Taking the moment-matching mean from each transition's bootstrap targets, as on Walker2d, Hopper and Ant, together with the log-variance moment head, removes this blindness and improves SVDRL in two of three variants, but it stays behind the sliced Cram\'er and multiquadric MMD critics.

\begin{table}[t]
  \centering
  \caption{\emph{The competitors' pixel benchmark.} Pixel maze of \citet{debes2026multivariate}, policy evaluation of the uniform random policy: per variant and head, the number of seeds $n$ and the median over seeds of each distance to the Monte Carlo returns at the start state, with its interval (the estimate $\hat w_2$ of \citet{debes2026multivariate}, $\hat w_2$ against $200$ Monte Carlo returns, and CDF-based sliced $\W2$). SVDRL with the log-variance moment head and with its earlier configuration.}
  \label{tab:mz}
  \adjustbox{max width=\linewidth}{\input{tables/tabA_mz.tex}}
\end{table}

\subsection{Multi-Objective MuJoCo}\label{app:mo}

The standardization by the moment head limits SAC-SVDRL in distributional policy evaluation of fixed policies on mo-halfcheetah and mo-hopper (\cref{tab:mo1}).
With its selected configuration, SVDRL's particles lie between MD3QN and energy MMD in energy distance after $10^6$ steps on both environments, and behind both after $5\cdot10^5$ steps.
The remaining gap is under-dispersion: the EMA target cloud is narrower than the online one, which the bootstrap compounds by up to the factor $1/(1-\gamma^2)$, about $50$ at $\gamma=0.99$.
The sampler of the selected configuration fails its ground-truth-free gate on three of four policies.

A float32 cancellation in PopArt's variance accumulator affected the low-coefficient-of-variation targets of the first multi-objective runs.
We reran the HalfCheetah baselines with a centered accumulator, and these rows, together with SVDRL's selected configuration, which uses it, are the comparison there; the uncentered HalfCheetah rows are kept for reference.
A check of the logged PopArt statistics of every other reported run finds no comparable effect.

\begin{table*}[t]
  \centering
  \caption{\emph{Multi-objective MuJoCo.} Distributional policy evaluation of fixed policies after $5\cdot10^5$ and $10^6$ steps: energy distance and sliced $\W2^2$ to the Monte Carlo law and between the centered laws, and the absolute mean error, each the median over seeds with its $95\,\%$ bootstrap interval, a run's value its mean over the last $20\,\%$ of the evaluations. SVDRL's initial and selected configurations as particles and sampler; ${}^{\dagger}$: the configuration's sampler fails its ground-truth-free gate. IQN models per-dimension marginals, so its joint metrics are empty. On mo-halfcheetah the heads with the centered PopArt accumulator and SVDRL's selected configuration, which uses it, are compared, and the uncentered rows are kept for reference. Bold: inside the interval of the best compared median of its column.}
  \label{tab:mo1}
  \small
  \adjustbox{max width=\textwidth}{\input{tables/tab_mo1_compact.tex}}
\end{table*}

%% file: tables/tab_re1_appendix.tex
\begin{tabular}{l*{9}{c}}
\toprule
 & \multicolumn{3}{c}{CVaR error ($\downarrow$)} & Event ($\downarrow$) & \multicolumn{2}{c}{Conditional mean error ($\downarrow$)} & Entropic ($\downarrow$) & \multicolumn{2}{c}{Decision accuracy ($\uparrow$)} \\
\cmidrule(lr){2-4}\cmidrule(lr){5-5}\cmidrule(lr){6-7}\cmidrule(lr){8-8}\cmidrule(lr){9-10}
Estimator & $\alpha=0.001$ & $\alpha=0.01$ & $\alpha=0.001$, $(\frac12,\frac12)$ & $|\log\hat P/P|$ & $Z_1\leq\tau$ & $Z_1=z_{\mathrm{rare}}$ & $\beta=0.1$ & chance & ESR \\
\midrule
\multicolumn{10}{l}{\emph{SVDRL sampler, bootstrapping from}} \\
\quad its draws, EDM$^{\ast}$ & $\mathbf{6.9}$ & $9.7$ & $\mathbf{2.9}$ & $\mathbf{0.16}$ & $2.6$ & $2.9$ & $13$ & $\mathbf{1}$ & $0.57$ \\[-1pt]
 & $\sd{[5.8,\,8]}$ & $\sd{[6.3,\,15]}$ & $\sd{[2.4,\,3.4]}$ & $\sd{[0.11,\,0.28]}$ & $\sd{[2.3,\,2.8]}$ & $\sd{[2.5,\,3.6]}$ & $\sd{[11,\,15]}$ & $\sd{[0.99,\,1]}$ & $\sd{[0.57,\,0.57]}$ \\[2pt]
\quad its draws & $9.2$ & $8$ & $4.8$ & $0.32$ & $\mathbf{1.6}$ & -- & $\mathbf{7.7}$ & $0.95$ & $0.28$ \\[-1pt]
 & $\sd{[6.4,\,12]}$ & $\sd{[5.7,\,13]}$ & $\sd{[3.2,\,5.9]}$ & $\sd{[0.19,\,0.49]}$ & $\sd{[1.3,\,1.8]}$ &  & $\sd{[5.6,\,9.5]}$ & $\sd{[0.8,\,1]}$ & $\sd{[0.28,\,0.28]}$ \\[2pt]
\quad its draws, moment-matched & $11$ & $18$ & $4.8$ & $0.48$ & $2.6$ & -- & $17$ & $0.89$ & $0.27$ \\[-1pt]
 & $\sd{[8,\,16]}$ & $\sd{[11,\,27]}$ & $\sd{[3.6,\,7]}$ & $\sd{[0.28,\,0.87]}$ & $\sd{[2,\,3.4]}$ &  & $\sd{[10,\,25]}$ & $\sd{[0.62,\,1]}$ & $\sd{[0.27,\,0.27]}$ \\[2pt]
\quad its draws, running moments & $36$ & $\mathbf{3.7}$ & $7.1$ & $0.29$ & $22$ & -- & $35$ & $0.94$ & $\mathbf{0.88}$ \\[-1pt]
 & $\sd{[34,\,38]}$ & $\sd{[1.9,\,5.5]}$ & $\sd{[6.1,\,8.2]}$ & $\sd{[0.17,\,0.53]}$ & $\sd{[21,\,22]}$ &  & $\sd{[34,\,36]}$ & $\sd{[0.75,\,1]}$ & $\sd{[0.88,\,0.88]}$ \\[2pt]
\quad particles & $15$ & $16$ & $6.6$ & $0.45$ & $\mathbf{1.8}$ & $\mathbf{1.6}$ & $14$ & $0.91$ & $0.35$ \\[-1pt]
 & $\sd{[11,\,17]}$ & $\sd{[10,\,22]}$ & $\sd{[4.8,\,7.4]}$ & $\sd{[0.26,\,0.78]}$ & $\sd{[1.4,\,2.5]}$ & $\sd{[1.2,\,2]}$ & $\sd{[10,\,17]}$ & $\sd{[0.72,\,1]}$ & $\sd{[0.35,\,0.35]}$ \\[2pt]
\quad particles, resample-move & $31$ & $32$ & $13$ & $1.6$ & $2.7$ & -- & $27$ & $0.64$ & $0$ \\[-1pt]
 & $\sd{[16,\,56]}$ & $\sd{[27,\,38]}$ & $\sd{[6.6,\,23]}$ & $\sd{[1.1,\,3.1]}$ & $\sd{[2.2,\,6.1]}$ &  & $\sd{[20,\,41]}$ & $\sd{[0.04,\,1]}$ & $\sd{[0,\,0]}$ \\[2pt]
\midrule
\multicolumn{10}{l}{\emph{Particles and kernel density estimates}} \\
\quad SVDRL Stein particles & $83$ & $52$ & $35$ & $6.5$ & $13$ & -- & $57$ & $0.56$ & $0.28$ \\[-1pt]
 & $\sd{[71,\,94]}$ & $\sd{[49,\,54]}$ & $\sd{[30,\,39]}$ & $\sd{[5.6,\,7.3]}$ & $\sd{[11,\,15]}$ &  & $\sd{[50,\,63]}$ & $\sd{[0.28,\,0.92]}$ & $\sd{[0.28,\,0.28]}$ \\[2pt]
\quad MMD-RBF, atoms & $120$ & $56$ & $51$ & $8.5$ & $18$ & -- & $69$ & $0.5$ & $0$ \\[-1pt]
 & $\sd{[120,\,120]}$ & $\sd{[54,\,57]}$ & $\sd{[50,\,51]}$ & $\sd{[8.5,\,8.5]}$ & $\sd{[18,\,18]}$ &  & $\sd{[69,\,70]}$ & $\sd{[0,\,1]}$ & $\sd{[0,\,0]}$ \\[2pt]
\quad MMD-RBF, KDE & $110$ & $49$ & $46$ & $8.5$ & $33$ & $60$ & $69$ & $0.5$ & $0$ \\[-1pt]
 & $\sd{[110,\,110]}$ & $\sd{[47,\,51]}$ & $\sd{[45,\,47]}$ & $\sd{[8.5,\,8.5]}$ & $\sd{[29,\,38]}$ & $\sd{[52,\,67]}$ & $\sd{[68,\,69]}$ & $\sd{[0,\,1]}$ & $\sd{[0,\,0]}$ \\[2pt]
\quad MD3QN, atoms & $100$ & $41$ & $42$ & $7.9$ & $18$ & -- & $66$ & $0.54$ & $0.09$ \\[-1pt]
 & $\sd{[97,\,110]}$ & $\sd{[37,\,44]}$ & $\sd{[39,\,44]}$ & $\sd{[7.2,\,8.4]}$ & $\sd{[18,\,18]}$ &  & $\sd{[64,\,67]}$ & $\sd{[0.09,\,1]}$ & $\sd{[0.09,\,0.09]}$ \\[2pt]
\quad MD3QN, KDE & $98$ & $40$ & $40$ & $6.7$ & $18$ & $20$ & $66$ & $0.56$ & $0.1$ \\[-1pt]
 & $\sd{[92,\,100]}$ & $\sd{[35,\,44]}$ & $\sd{[36,\,42]}$ & $\sd{[6.2,\,7.5]}$ & $\sd{[15,\,19]}$ & $\sd{[18,\,22]}$ & $\sd{[63,\,67]}$ & $\sd{[0.1,\,1]}$ & $\sd{[0.1,\,0.1]}$ \\[2pt]
\quad Energy MMD, atoms & $71$ & $53$ & $30$ & $4.8$ & $11$ & -- & $44$ & $0.71$ & $0.48$ \\[-1pt]
 & $\sd{[58,\,80]}$ & $\sd{[50,\,56]}$ & $\sd{[25,\,34]}$ & $\sd{[4,\,5.4]}$ & $\sd{[9.3,\,12]}$ &  & $\sd{[36,\,50]}$ & $\sd{[0.48,\,1]}$ & $\sd{[0.48,\,0.48]}$ \\[2pt]
\quad Energy MMD, KDE & $72$ & $54$ & $31$ & $4.9$ & $11$ & $11$ & $44$ & $0.7$ & $0.48$ \\[-1pt]
 & $\sd{[60,\,82]}$ & $\sd{[51,\,57]}$ & $\sd{[26,\,36]}$ & $\sd{[4.1,\,5.4]}$ & $\sd{[8.6,\,14]}$ & $\sd{[8.5,\,15]}$ & $\sd{[37,\,50]}$ & $\sd{[0.48,\,1]}$ & $\sd{[0.48,\,0.48]}$ \\[2pt]
\quad Sliced Cram\'er, atoms & $70$ & $53$ & $30$ & $4.8$ & $11$ & -- & $43$ & $0.71$ & $0.48$ \\[-1pt]
 & $\sd{[58,\,80]}$ & $\sd{[50,\,56]}$ & $\sd{[25,\,34]}$ & $\sd{[4,\,5.4]}$ & $\sd{[9.3,\,12]}$ &  & $\sd{[36,\,49]}$ & $\sd{[0.48,\,0.99]}$ & $\sd{[0.48,\,0.48]}$ \\[2pt]
\quad Sliced Cram\'er, KDE & $72$ & $54$ & $31$ & $4.8$ & $11$ & $11$ & $44$ & $0.71$ & $0.48$ \\[-1pt]
 & $\sd{[60,\,82]}$ & $\sd{[51,\,56]}$ & $\sd{[26,\,35]}$ & $\sd{[4,\,5.4]}$ & $\sd{[9.2,\,13]}$ & $\sd{[8.8,\,14]}$ & $\sd{[37,\,50]}$ & $\sd{[0.48,\,0.99]}$ & $\sd{[0.48,\,0.48]}$ \\[2pt]
\quad Sliced $\W2$, atoms & $67$ & $53$ & $29$ & $4.8$ & $9.8$ & -- & $42$ & $0.7$ & $0.48$ \\[-1pt]
 & $\sd{[55,\,76]}$ & $\sd{[50,\,56]}$ & $\sd{[24,\,32]}$ & $\sd{[4,\,5.4]}$ & $\sd{[8.1,\,11]}$ &  & $\sd{[35,\,47]}$ & $\sd{[0.48,\,1]}$ & $\sd{[0.48,\,0.48]}$ \\[2pt]
\quad Sliced $\W2$, KDE & $71$ & $55$ & $30$ & $4.8$ & $9.9$ & $9.9$ & $43$ & $0.7$ & $0.48$ \\[-1pt]
 & $\sd{[59,\,79]}$ & $\sd{[52,\,58]}$ & $\sd{[25,\,34]}$ & $\sd{[4,\,5.4]}$ & $\sd{[8.4,\,11]}$ & $\sd{[8.1,\,12]}$ & $\sd{[36,\,48]}$ & $\sd{[0.48,\,1]}$ & $\sd{[0.48,\,0.48]}$ \\[2pt]
\quad Sinkhorn, atoms & $88$ & $47$ & $38$ & $5.6$ & $12$ & -- & $53$ & $0.59$ & $0.38$ \\[-1pt]
 & $\sd{[80,\,98]}$ & $\sd{[39,\,53]}$ & $\sd{[35,\,43]}$ & $\sd{[4.9,\,6.5]}$ & $\sd{[12,\,14]}$ &  & $\sd{[50,\,59]}$ & $\sd{[0.38,\,0.89]}$ & $\sd{[0.38,\,0.38]}$ \\[2pt]
\quad Sinkhorn, KDE & $84$ & $47$ & $36$ & $5.5$ & $17$ & $30$ & $53$ & $0.6$ & $0.38$ \\[-1pt]
 & $\sd{[77,\,93]}$ & $\sd{[40,\,52]}$ & $\sd{[33,\,40]}$ & $\sd{[4.8,\,6.5]}$ & $\sd{[14,\,22]}$ & $\sd{[24,\,38]}$ & $\sd{[50,\,58]}$ & $\sd{[0.38,\,0.95]}$ & $\sd{[0.38,\,0.38]}$ \\[2pt]
\midrule
\multicolumn{10}{l}{\emph{Quantile critic}} \\
\quad oracle $\lambda$, $1000$ quantiles & $62$ & $10$ & -- & $1.8$ & -- & -- & $59$ & $0.71$ & $0.02$ \\[-1pt]
 & $\sd{[59,\,64]}$ & $\sd{[6.6,\,13]}$ &  & $\sd{[0.49,\,4.6]}$ &  &  & $\sd{[56,\,61]}$ & $\sd{[0.26,\,1]}$ & $\sd{[0.02,\,0.02]}$ \\
\bottomrule
\end{tabular}

%% file: tables/tabA_sel1_compact.tex
\begin{tabular}{lclrrrrrrrrrrl}
\toprule
 & & & \multicolumn{2}{c}{Hybrid} & \multicolumn{2}{c}{Sampler} & \multicolumn{2}{c}{Particles} & \multicolumn{2}{c}{Tuning choice} & \multicolumn{2}{c}{Test $r>6$ only} & \\
\cmidrule(lr){4-5}\cmidrule(lr){6-7}\cmidrule(lr){8-9}\cmidrule(lr){10-11}\cmidrule(lr){12-13}
MDP & $d$ & Error ($\downarrow$) & $10^{-3}$ & ratio & $10^{-3}$ & ratio & $10^{-3}$ & ratio & $10^{-3}$ & ratio & $10^{-3}$ & ratio & Reports \\
\midrule
MDP1 & 1 & shape & 5.7 & \textbf{0.27} & 5.7 & \textbf{0.27} & 64 & 3.1 & 5.7 & \textbf{0.27} & 5.7 & \textbf{0.27} & sampler \\
\addlinespace[2pt]
MDP2 & 1 & mass & 4.7 & 1.5 & 7.7 & 2.5 & 4.8 & 1.5 & 7.7 & 2.5 & 4.7 & 1.5 & sampler \\
 &  & shape & 81 & 1.6 & 13 & \textbf{0.25} & 83 & 1.7 & 13 & \textbf{0.25} & 81 & 1.6 & sampler \\
\addlinespace[2pt]
MDP3 & 2 & mass & 2.3 & 1.5 & 12 & 8.0 & 2.3 & 1.5 & 12 & 8.0 & 2.3 & 1.5 & sampler \\
 &  & shape & 389 & \textbf{0.80} & 125 & \textbf{0.26} & 389 & \textbf{0.80} & 125 & \textbf{0.26} & 389 & \textbf{0.80} & sampler \\
\addlinespace[2pt]
MDP4 & 5 & mass & 0.065 & \underline{--} & 13 & -- & 0 & \underline{1.0} & 0 & \underline{1.0} & 0.065 & \underline{--} & particles \\
 &  & shape & 62 & \textbf{0.097} & 808 & \underline{1.3} & 55 & \textbf{0.086} & 55 & \textbf{0.086} & 62 & \textbf{0.097} & particles \\
\addlinespace[2pt]
MDP5 & 2 & shape & 6.1 & \textbf{0.20} & 6.1 & \textbf{0.20} & 37 & 1.2 & 6.1 & \textbf{0.20} & 6.1 & \textbf{0.20} & sampler \\
\addlinespace[2pt]
MDP6 & 2 & mass & 0.060 & \underline{--} & 11 & -- & 0 & \underline{1.0} & 0 & \underline{1.0} & 0.060 & \underline{--} & particles \\
 &  & shape & 50 & \textbf{0.12} & 154 & \textbf{0.37} & 73 & \textbf{0.17} & 73 & \textbf{0.17} & 50 & \textbf{0.12} & particles \\
\addlinespace[2pt]
MDP7 & 3 & shape & 1.5 & \textbf{0.27} & 1.5 & \textbf{0.27} & 5.5 & \textbf{0.97} & 1.5 & \textbf{0.27} & 1.5 & \textbf{0.27} & sampler \\
\addlinespace[2pt]
MDP8 & 2 & shape & 38 & \textbf{0.23} & 38 & \textbf{0.23} & 160 & \underline{0.98} & 38 & \textbf{0.23} & 38 & \textbf{0.23} & sampler \\
\midrule
\multicolumn{14}{l}{\emph{Wins/ties/losses, selected configuration}} \\
\multicolumn{3}{l}{\quad mass, own} & \multicolumn{2}{c}{0/2/2} & \multicolumn{2}{c}{0/0/4} & \multicolumn{2}{c}{0/2/2} & \multicolumn{2}{c}{0/2/2} & \multicolumn{2}{c}{0/2/2} & \\
\multicolumn{3}{l}{\quad mass, all} & \multicolumn{2}{c}{0/2/2} & \multicolumn{2}{c}{0/0/4} & \multicolumn{2}{c}{0/2/2} & \multicolumn{2}{c}{0/2/2} & \multicolumn{2}{c}{0/2/2} & \\
\multicolumn{3}{l}{\quad shape, own} & \multicolumn{2}{c}{8/0/0} & \multicolumn{2}{c}{8/0/0} & \multicolumn{2}{c}{6/1/1} & \multicolumn{2}{c}{8/0/0} & \multicolumn{2}{c}{8/0/0} & \\
\multicolumn{3}{l}{\quad shape, all} & \multicolumn{2}{c}{7/0/1} & \multicolumn{2}{c}{7/1/0} & \multicolumn{2}{c}{4/1/3} & \multicolumn{2}{c}{8/0/0} & \multicolumn{2}{c}{7/0/1} & \\
\midrule
\multicolumn{14}{l}{\emph{Wins/ties/losses, single configuration}} \\
\multicolumn{3}{l}{\quad mass, own} & \multicolumn{2}{c}{0/2/2} & \multicolumn{2}{c}{0/0/4} & \multicolumn{2}{c}{0/2/2} & \multicolumn{2}{c}{0/2/2} & \multicolumn{2}{c}{0/2/2} & \\
\multicolumn{3}{l}{\quad mass, all} & \multicolumn{2}{c}{0/2/2} & \multicolumn{2}{c}{0/0/4} & \multicolumn{2}{c}{0/2/2} & \multicolumn{2}{c}{0/2/2} & \multicolumn{2}{c}{0/2/2} & \\
\multicolumn{3}{l}{\quad shape, own} & \multicolumn{2}{c}{8/0/0} & \multicolumn{2}{c}{8/0/0} & \multicolumn{2}{c}{5/1/2} & \multicolumn{2}{c}{8/0/0} & \multicolumn{2}{c}{8/0/0} & \\
\multicolumn{3}{l}{\quad shape, all} & \multicolumn{2}{c}{7/0/1} & \multicolumn{2}{c}{7/1/0} & \multicolumn{2}{c}{4/0/4} & \multicolumn{2}{c}{8/0/0} & \multicolumn{2}{c}{7/0/1} & \\
\bottomrule
\end{tabular}

%% file: tables/tab1_mass.tex
\begin{tabular}{lcrrrr}
\toprule
 & & \multicolumn{4}{c}{Mode masses, total variation ($\downarrow$)} \\
\cmidrule(lr){3-6}
 & & & & \multicolumn{2}{c}{SVDRL / best} \\
\cmidrule(lr){5-6}
MDP & $d$ & SVDRL & floor & own & all \\
\midrule
MDP2 & 1 & 4.7 & 3.1 & 1.3 & 1.5 \\
MDP3 & 2 & 2.3 & 1.6 & 1.4 & 1.5 \\
MDP4 & 5 & 0.065 & 0 & \underline{0.11} & \underline{--} \\
MDP6 & 2 & 0.060 & 0 & \underline{--} & \underline{--} \\
\midrule
\rlap{Wins/ties/losses} & & & & 0/2/2 & 0/2/2 \\
\rlap{\quad its particles alone} & & & & 0/2/2 & 0/2/2 \\
\bottomrule
\end{tabular}

%% file: tables/tab1_competitors.tex
\begin{tabular}{llllllr}
\toprule
MDP & Error & Comparison & Best competitor & Target set & Learning rate & Its error ($10^{-3}$) ($\downarrow$) \\
\midrule
MDP1 & shape & own & multiquadric MMD$^\dagger$ & own & base & 24.3 \\
 &  & all & multiquadric MMD$^\dagger$ & ours & base & 20.8 \\
 &  & $M{=}4096$ & multiquadric MMD$^\dagger$ & ours & base & 0.563 \\
\addlinespace[2pt]
MDP2 & shape & own & sliced $\W2$ & own & base & 98.4 \\
 &  & all & sliced Cram\'er$^\dagger$ & ours & $3\cdot10^{-2}$ & 50.1 \\
 & mass & own & multiquadric MMD$^\dagger$ & own & $3\cdot10^{-2}$ & 3.73 \\
 &  & all & sliced Cram\'er$^\dagger$ & ours & $3\cdot10^{-2}$ & 3.13 \\
\addlinespace[2pt]
MDP3 & shape & own & sliced Cram\'er$^\dagger$ & own & $3\cdot10^{-2}$ & 1950 \\
 &  & all & sliced Cram\'er$^\dagger$ & ours & $3\cdot10^{-2}$ & 488 \\
 & mass & own & sliced Cram\'er$^\dagger$ & own & $3\cdot10^{-2}$ & 1.68 \\
 &  & all & sliced Cram\'er$^\dagger$ & ours & $3\cdot10^{-2}$ & 1.56 \\
\addlinespace[2pt]
MDP4 & shape & own & multiquadric MMD$^\dagger$ & own & $3\cdot10^{-2}$ & 3080 \\
 &  & all & multiquadric MMD$^\dagger$ & ours & $3\cdot10^{-2}$ & 641 \\
 & mass & own & multiquadric MMD$^\dagger$ & own & $3\cdot10^{-2}$ & 0.597 \\
 &  & all & multiquadric MMD$^\dagger$ & ours & $3\cdot10^{-2}$ & 0 \\
\addlinespace[2pt]
MDP5 & shape & own & sliced $\W2$$^\dagger$ & own & $10^{-1}$ & 36.7 \\
 &  & all & sliced $\W2$$^\dagger$ & ours & $10^{-1}$ & 30.0 \\
 &  & $M{=}4096$ & sliced $\W2$$^\dagger$ & ours & base & 5.17 \\
\addlinespace[2pt]
MDP6 & shape & own & multiquadric MMD$^\dagger$ & own & $3\cdot10^{-2}$ & 925 \\
 &  & all & multiquadric MMD$^\dagger$ & ours & $3\cdot10^{-2}$ & 419 \\
 & mass & own & multiquadric MMD$^\dagger$ & own & $3\cdot10^{-2}$ & 0 \\
 &  & all & multiquadric MMD$^\dagger$ & own & $3\cdot10^{-2}$ & 0 \\
\addlinespace[2pt]
MDP7 & shape & own & sliced $\W2$ & own & base & 8.17 \\
 &  & all & sliced $\W2$ & ours & base & 5.60 \\
 &  & $M{=}4096$ & sliced $\W2$$^\dagger$ & ours & base & 0.375 \\
\addlinespace[2pt]
MDP8 & shape & own & sliced $\W2$$^\dagger$ & own & $10^{-1}$ & 205 \\
 &  & all & sliced $\W2$$^\dagger$ & ours & $10^{-1}$ & 163 \\
 &  & $M{=}4096$ & sliced $\W2$$^\dagger$ & ours & base & 16.9 \\
\bottomrule
\end{tabular}

%% file: tables/tabA_cmtab.tex
\begin{tabular}{llrrr}
\toprule
MDP & row & $M$ & error (mean $\pm$ SE) ($\downarrow$) & hybrid $\div$ row (Welch $t$) \\
\midrule
MDP1 & SVDRL hybrid (sampler) & 4096 & $0.00559\sd{\pm 0.00052}$ & \\
MDP1 & multiquadric MMD$^\dagger$ & 64 & $0.0208\sd{\pm 0.00056}$ & 0.269 (-19.9) \\
MDP1 & multiquadric MMD$^\dagger$ & 512 & $0.00286\sd{\pm 0.000070}$ & 1.95 (+5.2) \\
MDP1 & multiquadric MMD$^\dagger$ & 4096 & $\mathbf{0.000563}\sd{\pm 0.000061}$ & 9.94 (+9.7) \\
MDP1 & SVDRL particles, $\alpha=0$ & 512 & $0.0118\sd{\pm 0.00028}$ & 0.474 (-10.5) \\
MDP1 & SVDRL particles, $\alpha=0$ & 4096 & $0.00157\sd{\pm 0.000069}$ & 3.56 (+7.7) \\
\midrule
MDP5 & SVDRL hybrid (sampler) & 4096 & $0.00609\sd{\pm 0.000068}$ & \\
MDP5 & sliced $\W2$$^\dagger$ & 128 & $0.0300\sd{\pm 0.00022}$ & 0.203 (-103.5) \\
MDP5 & sliced $\W2$$^\dagger$ & 512 & $0.0112\sd{\pm 0.00027}$ & 0.544 (-18.2) \\
MDP5 & sliced $\W2$$^\dagger$ & 4096 & $0.00517\sd{\pm 0.00024}$ & 1.18 (+3.6) \\
MDP5 & SVDRL particles, $\alpha=0$ & 512 & $0.0166\sd{\pm 0.00011}$ & 0.366 (-82.3) \\
MDP5 & SVDRL particles, $\alpha=0$ & 4096 & $\mathbf{0.00489}\sd{\pm 0.000022}$ & 1.24 (+16.6) \\
\midrule
MDP7 & SVDRL hybrid (sampler) & 4096 & $0.00155\sd{\pm 0.000030}$ & \\
MDP7 & sliced $\W2$$^\dagger$ & 192 & $0.00730\sd{\pm 0.00010}$ & 0.213 (-53.6) \\
MDP7 & sliced $\W2$$^\dagger$ & 512 & $0.00284\sd{\pm 0.000041}$ & 0.548 (-25.2) \\
MDP7 & sliced $\W2$$^\dagger$ & 4096 & $0.000375\sd{\pm 0.0000059}$ & 4.14 (+38.1) \\
MDP7 & SVDRL particles, $\alpha=0$ & 512 & $0.00219\sd{\pm 0.000012}$ & 0.709 (-19.6) \\
MDP7 & SVDRL particles, $\alpha=0$ & 4096 & $\mathbf{0.000305}\sd{\pm 0.0000055}$ & 5.10 (+40.5) \\
\midrule
MDP8 & SVDRL hybrid (sampler) & 4096 & $0.0378\sd{\pm 0.0011}$ & \\
MDP8 & sliced $\W2$$^\dagger$ & 128 & $0.163\sd{\pm 0.0031}$ & 0.232 (-37.8) \\
MDP8 & sliced $\W2$$^\dagger$ & 512 & $0.0469\sd{\pm 0.00061}$ & 0.806 (-7.1) \\
MDP8 & sliced $\W2$$^\dagger$ & 4096 & $\mathbf{0.0169}\sd{\pm 0.00028}$ & 2.24 (+18.2) \\
MDP8 & SVDRL particles, $\alpha=0$ & 512 & $0.0787\sd{\pm 0.0015}$ & 0.481 (-21.4) \\
MDP8 & SVDRL particles, $\alpha=0$ & 4096 & $0.0228\sd{\pm 0.00031}$ & 1.66 (+12.9) \\
\bottomrule
\end{tabular}

%% file: tables/tabA_fullw2.tex
\begin{tabular}{l rr rl rl}
\toprule
 & \multicolumn{2}{c}{SVDRL} & \multicolumn{2}{c}{best competitor, own target set} & \multicolumn{2}{c}{best competitor, every row} \\
\cmidrule(lr){2-3}\cmidrule(lr){4-5}\cmidrule(lr){6-7}
MDP & selected ($\downarrow$) & single ($\downarrow$) & error (row) ($\downarrow$) & selected $\div$ & error (row) ($\downarrow$) & selected $\div$ \\
\midrule
MDP1 ($d{=}1$) & 0.0635 & 0.0768 & 0.0243 (multiquadric MMD$^\dagger$) & $\times$2.6 L & 0.0208 (multiquadric MMD$^{\dagger*}$) & $\times$3.1 L \\
MDP2 ($d{=}1$) & 1.05 & 1.05 & 0.832 (multiquadric MMD$^\dagger$, $3\cdot10^{-2}$) & $\times$1.3 L & 0.504 (sliced $\W2$$^{\dagger*}$, $3\cdot10^{-2}$) & $\times$2.1 L \\
MDP3 ($d{=}2$) & 12.9 & 12.9 & 11.8 (sliced Cram\'er$^\dagger$, $3\cdot10^{-2}$) & $\times$1.1 T & 11.8 (sliced Cram\'er$^\dagger$, $3\cdot10^{-2}$) & $\times$1.1 T \\
MDP4 ($d{=}5$) & 0.0526 & 0.0526 & 3.72 (multiquadric MMD$^\dagger$, $3\cdot10^{-2}$) & $\times$0.014 W & 0.613 (multiquadric MMD$^{\dagger*}$, $3\cdot10^{-2}$) & $\times$0.086 W \\
MDP5 ($d{=}2$) & 0.0368 & 0.0459 & 0.0367 (sliced $\W2$$^\dagger$, $0.1$) & $\times$1.0 T & 0.0300 (sliced $\W2$$^{\dagger*}$, $0.1$) & $\times$1.2 L \\
MDP6 ($d{=}2$) & 9.68 & 9.68 & 9.77 (multiquadric MMD$^\dagger$, $3\cdot10^{-2}$) & $\times$0.99 T & 9.73 (multiquadric MMD$^{\dagger*}$, $3\cdot10^{-2}$) & $\times$0.99 T \\
MDP7 ($d{=}3$) & 0.00545 & 0.00545 & 0.00817 (sliced $\W2$) & $\times$0.67 W & 0.00560 (sliced $\W2$$^{*}$) & $\times$0.97 W \\
MDP8 ($d{=}2$) & 0.171 & 0.219 & 0.205 (sliced $\W2$$^\dagger$, $0.1$) & $\times$0.83 W & 0.163 (sliced $\W2$$^{\dagger*}$, $0.1$) & $\times$1.1 T \\
\midrule
\multicolumn{7}{l}{W/T/L, selected configuration per MDP: 3/3/2 (own target set) / 2/3/3 (every row)} \\
\multicolumn{7}{l}{W/T/L, single configuration: 2/3/3 (own target set) / 2/2/4 (every row)} \\
\bottomrule
\end{tabular}

%% file: tables/tabA_alpha.tex
\begin{table*}[t]
\centering
\caption{\emph{Relaxation and target set.} Ablation of the relaxation $\alpha$ of the Stein move and the stratified target set on the synthetic MDPs: within-mode shape error (every MDP; the full error where every evaluated state--action pair is unimodal) and total variation of the mode masses (mass TV; MDPs with separated modes). Fully annealed: SVDRL with fully annealed transport; + strat.: the stratified target set (systematic successors + Latin-hypercube rewards); + full-cov. matching: full-covariance moment matching; $\rho=1$: the whole cloud replaced by the stratified target set, with relaxation $\alpha$ (no running-mean moment matching). IQM over the ten final seeds; $^\dagger$: run on the five tuning seeds only, IQM over those; --: not run. Time-averaged over the last 20\,\% of the steps.}
\label{tab:alpha}
\footnotesize
\setlength{\tabcolsep}{4pt}
\adjustbox{max width=\textwidth}{\begin{tabular}{l l r r r r r r }
\toprule
 & & & & & \multicolumn{3}{c}{$\rho=1$} \\
\cmidrule(lr){6-8}
MDP & block & fully annealed & fully annealed + strat. & + full-cov. matching & $\alpha=10^{-3}$ & $\alpha=0.01$ & $\alpha=0.1$ \\
\midrule
MDP1 & shape ($\downarrow$) & $0.127$ & $0.109$$^\dagger$ & -- & $0.0770$ & -- & $0.0635$ \\
MDP2 & shape ($\downarrow$) & $2.52$ & -- & -- & $0.0815$ & -- & $0.327$$^\dagger$ \\
 & mass TV ($\downarrow$) & $0.0374$ & -- & -- & $0.00487$ & -- & $0.00496$$^\dagger$ \\
MDP3 & shape ($\downarrow$) & $20.0$ & -- & -- & $0.387$ & -- & $2.94$$^\dagger$ \\
 & mass TV ($\downarrow$) & $0.0550$ & -- & -- & $0.00244$ & -- & $0.00251$$^\dagger$ \\
MDP4 ($d{=}5$) & shape ($\downarrow$) & $59.6$ & -- & -- & $0.0549$ & $0.0575$$^\dagger$ & $1.20$$^\dagger$ \\
 & mass TV ($\downarrow$) & $0.0420$ & -- & -- & $0$ & $0$$^\dagger$ & $0$$^\dagger$ \\
MDP5 & shape ($\downarrow$) & $0.0437$ & $0.0369$ & -- & $0.0460$ & -- & $0.115$$^\dagger$ \\
MDP6 & shape ($\downarrow$) & $16.4$ & -- & -- & $0.0727$ & -- & $0.461$ \\
 & mass TV ($\downarrow$) & $0.0440$ & -- & -- & $0$ & -- & $0$ \\
MDP7 & shape ($\downarrow$) & $0.0241$ & $0.0213$$^\dagger$ & $0.00631$$^\dagger$ & $0.00545$ & -- & $0.0380$$^\dagger$ \\
MDP8 & shape ($\downarrow$) & $1.23$ & $1.16$$^\dagger$ & $0.171$ & $0.217$ & -- & $0.263$$^\dagger$ \\
\bottomrule
\end{tabular}}
\end{table*}

%% file: tables/tabA_epsst.tex
\begin{table*}[t]
\centering
\caption{\emph{Ground-truth-free gap.} Measured Stein-particle gap $\varepsilon_{\mathrm{St}}$: the distance between SVDRL's particle cloud ($M{=}64d$) and its own reverse-diffusion sampler ($N{=}4096$ draws, Euler on 512 geometric levels of the state--action pair's EMA score net), computed without ground truth; next to it, the ground-truth error of each object in the same distance. $d{=}1$: exact $W_1$ and $W_2$; $d{=}2$: exact $W_2$ (GT side: 4096 draws); $d{>}2$: sliced $W_2$ (quantile estimator). Per seed: $W_2$ per evaluated state--action pair, mean over pairs, then over the tail evaluations (last 20\,\% of the steps; every 4th at $d{=}2$); IQM over the ten final seeds [bootstrap 95\,\% CI]. Configurations: the selected configuration of each MDP and the single configuration ($\rho=1$, $\alpha=10^{-3}$).}
\label{tab:epsst}
\footnotesize
\setlength{\tabcolsep}{3pt}
\adjustbox{max width=\textwidth}{\begin{tabular}{l c l r r r r r}
\toprule
MDP & $d$ & config & $M$ & $\varepsilon_{\mathrm{St}}$ ($W_2$) & $\varepsilon_{\mathrm{St}}$ ($W_1$) & particles--GT ($\downarrow$) & sampler--GT ($\downarrow$) \\
\midrule
MDP1 & 1 & $\rho{=}1$, $\alpha{=}0.1$ (selected) & 64 & $0.240\sd{[0.23,\,0.25]}$ & $0.168\sd{[0.16,\,0.18]}$ & $0.248\sd{[0.24,\,0.26]}$ & $0.0692\sd{[0.065,\,0.075]}$ \\
 &  & $\rho{=}1$, $\alpha{=}10^{-3}$ (single) & 64 & $0.264\sd{[0.25,\,0.27]}$ & $0.189\sd{[0.18,\,0.20]}$ & $0.271\sd{[0.26,\,0.28]}$ & $0.0695\sd{[0.065,\,0.075]}$ \\
\addlinespace
MDP2 & 1 & $\rho{=}1$, $\alpha{=}10^{-3}$ (selected = single) & 64 & $1.26\sd{[1.2,\,1.4]}$ & $0.342\sd{[0.32,\,0.37]}$ & $0.983\sd{[0.91,\,1.0]}$ & $1.08\sd{[1.0,\,1.2]}$ \\
\addlinespace
MDP3 & 2 & $\rho{=}1$, $\alpha{=}10^{-3}$ (selected = single) & 128 & $4.12\sd{[3.8,\,4.9]}$ & -- & $3.32\sd{[3.0,\,3.7]}$ & $4.44\sd{[3.8,\,5.1]}$ \\
\addlinespace
MDP4 ($d{=}5$) & 5 & $\rho{=}1$, $\alpha{=}10^{-3}$ (selected = single) & 320 & $7.95\sd{[7.4,\,9.0]}$ & -- & $0.229\sd{[0.23,\,0.23]}$ & $7.95\sd{[7.4,\,8.9]}$ \\
\addlinespace
MDP5 & 2 & fully annealed + strat. (selected) & 128 & $0.196\sd{[0.19,\,0.20]}$ & -- & $0.191\sd{[0.19,\,0.19]}$ & $0.0774\sd{[0.077,\,0.079]}$ \\
 &  & $\rho{=}1$, $\alpha{=}10^{-3}$ (single) & 128 & $0.217\sd{[0.22,\,0.22]}$ & -- & $0.214\sd{[0.21,\,0.22]}$ & $0.0777\sd{[0.077,\,0.079]}$ \\
\addlinespace
MDP6 & 2 & $\rho{=}1$, $\alpha{=}10^{-3}$ (selected = single) & 128 & $2.41\sd{[2.1,\,2.7]}$ & -- & $1.77\sd{[1.6,\,2.0]}$ & $2.44\sd{[2.2,\,2.7]}$ \\
\addlinespace
MDP7 & 3 & $\rho{=}1$, $\alpha{=}10^{-3}$ (selected = single) & 192 & $0.0817\sd{[0.081,\,0.083]}$ & -- & $0.0733\sd{[0.073,\,0.074]}$ & $0.0377\sd{[0.037,\,0.039]}$ \\
\addlinespace
MDP8 & 2 & + full-cov. matching (selected) & 128 & $0.374\sd{[0.37,\,0.38]}$ & -- & $0.377\sd{[0.37,\,0.38]}$ & $0.179\sd{[0.18,\,0.19]}$ \\
 &  & $\rho{=}1$, $\alpha{=}10^{-3}$ (single) & 128 & $0.443\sd{[0.43,\,0.45]}$ & -- & $0.436\sd{[0.43,\,0.45]}$ & $0.177\sd{[0.17,\,0.18]}$ \\
\bottomrule
\end{tabular}}
\end{table*}

%% file: tables/tabA_rj1.tex
\begin{tabular}{ll rrr rr}
\toprule
MDP & move & shape error ($\downarrow$) & $\div$ Stein & Welch $t$ & distinct & NN / std \\
\midrule
MDP1 & Stein (selected) & $0.0638\sd{\pm 0.0026}$ & -- & -- & 1.0 & $0.022$ \\
 & Langevin & $0.0762\sd{\pm 0.0034}$ & $\times1.19$ & $+2.9$ & 1.0 & $0.021$ \\
 & jitter & $0.0768\sd{\pm 0.0031}$ & $\times1.20$ & $+3.2$ & 1.0 & $0.021$ \\
 & no repulsion & $0.0771\sd{\pm 0.0031}$ & $\times1.21$ & $+3.2$ & 1.0 & $0.021$ \\
 & no move & $0.0771\sd{\pm 0.0032}$ & $\times1.21$ & $+3.2$ & 1.0 & $0.021$ \\
\midrule
MDP5 & Stein (selected) & $0.0374\sd{\pm 0.0004}$ & -- & -- & 1.0 & $0.17$ \\
 & Langevin & $0.0458\sd{\pm 0.0005}$ & $\times1.23$ & $+13.2$ & 1.0 & $0.16$ \\
 & jitter & $0.0453\sd{\pm 0.0004}$ & $\times1.21$ & $+13.9$ & 1.0 & $0.16$ \\
 & no repulsion & $1.07\sd{\pm 0.01}$ & $\times28.6$ & $+85.7$ & 0.030 & $0.0$ \\
 & no move & $0.0452\sd{\pm 0.0008}$ & $\times1.21$ & $+8.8$ & 1.0 & $0.16$ \\
\midrule
MDP8 & Stein (selected) & $0.164\sd{\pm 0.002}$ & -- & -- & 1.0 & $0.13$ \\
 & Langevin & $0.215\sd{\pm 0.004}$ & $\times1.31$ & $+11.6$ & 1.0 & $0.12$ \\
 & jitter & $0.228\sd{\pm 0.004}$ & $\times1.39$ & $+13.2$ & 1.0 & $0.12$ \\
 & no repulsion & $33.5\sd{\pm 1.6}$ & $\times204$ & $+20.6$ & 0.035 & $0.0$ \\
 & no move & $0.224\sd{\pm 0.008}$ & $\times1.36$ & $+6.8$ & 1.0 & $0.12$ \\
\bottomrule
\end{tabular}

%% file: tables/tab2_chain_compact.tex
\setlength{\tabcolsep}{4pt}
\begin{tabular}{l*{12}{c}}
\toprule
Errors ($\downarrow$) & \multicolumn{3}{c}{$\sigma_r=0$} & \multicolumn{3}{c}{$\sigma_r=0.05$} & \multicolumn{3}{c}{$\sigma_r=0.2$} & \multicolumn{3}{c}{$\sigma_r=0.5$} \\
\cmidrule(lr){2-4}\cmidrule(lr){5-7}\cmidrule(lr){8-10}\cmidrule(lr){11-13}
Method & $\W2$ & ED$^2$ & mean & $\W2$ & ED$^2$ & mean & $\W2$ & ED$^2$ & mean & $\W2$ & ED$^2$ & mean \\
\midrule
SVDRL particles & 0.482 & 19.0 & 4.31 & 0.572 & 6.79 & 5.44 & 1.35 & 12.7 & 9.91 & 2.96 & 25.9 & 19.2 \\
SVDRL sampler, $N=4096$ & -- & 12.5 & \textbf{0.413} & -- & \textbf{1.05} & \textbf{0.497} & -- & \textbf{0.778} & \textbf{1.03} & -- & \textbf{1.27} & \textbf{2.10} \\
SVDRL sampler, $N=128$ & 0.520 & 22.1 & 8.39 & 0.600 & 11.5 & 8.78 & 1.37 & 18.7 & 14.6 & 3.00 & 38.8 & 30.5 \\
SVDRL, fully annealed & 0.661 & 77.7 & 15.3 & 0.775 & 70.1 & 18.9 & 1.45 & 48.2 & 25.5 & 2.95 & 75.0 & 53.7 \\
\midrule
\multicolumn{13}{l}{\emph{Variants of \citet{debes2026multivariate}, as specified}} \\
sliced $\W2$ & \textbf{0.438} & 14.7 & 2.20 & \textbf{0.521} & 5.27 & 2.49 & 1.26 & 6.09 & 3.34 & 2.80 & 11.9 & 4.73 \\
sliced Cram\'er & 0.452 & 5.18 & 2.25 & 0.530 & 2.58 & 2.68 & 1.27 & 5.64 & 3.51 & 2.81 & 11.8 & 4.87 \\
MMD, multiquadric & 0.452 & 5.13 & 2.02 & 0.530 & 2.60 & 2.80 & 1.27 & 5.65 & 3.68 & 2.81 & 11.8 & 4.97 \\
MMD, inv.\ multiquadric & 0.452 & 5.26 & 2.20 & 0.530 & 2.62 & 2.78 & 1.27 & 5.66 & 3.67 & 2.81 & 11.8 & 4.95 \\
max-sliced $\W2$ & 0.456 & 21.2 & 3.68 & 0.539 & 7.56 & 3.96 & \textbf{1.26} & 7.66 & 4.64 & \textbf{2.78} & 14.0 & 5.75 \\
max-sliced Cram\'er & 0.468 & 10.6 & 3.02 & 0.557 & 3.72 & 3.62 & 1.28 & 6.78 & 4.78 & 2.80 & 13.8 & 6.69 \\
\midrule
\multicolumn{13}{l}{\emph{The same variants, on our target set}} \\
sliced $\W2$ & 0.442 & 4.73 & 1.17 & 0.525 & 2.40 & 1.60 & 1.27 & 5.13 & 1.82 & 2.83 & 10.6 & 2.49 \\
sliced Cram\'er & 0.454 & 4.66 & 1.72 & 0.530 & 2.33 & 2.16 & 1.26 & 5.25 & 2.14 & 2.81 & 11.4 & 2.87 \\
MMD, multiquadric & 0.452 & 2.77 & 1.62 & 0.530 & 2.40 & 2.43 & 1.26 & 5.29 & 2.64 & 2.82 & 11.1 & 3.19 \\
MMD, inv.\ multiquadric & 0.452 & \textbf{2.75} & 1.60 & 0.529 & 2.40 & 2.44 & 1.26 & 5.33 & 2.67 & 2.82 & 11.2 & 3.27 \\
max-sliced $\W2$ & 0.446 & 6.96 & 1.59 & 0.534 & 2.62 & 1.70 & 1.27 & 5.81 & 2.19 & 2.80 & 12.4 & 2.48 \\
max-sliced Cram\'er & 0.460 & 9.59 & 2.25 & 0.551 & 2.59 & 1.94 & 1.28 & 5.87 & 2.38 & 2.80 & 12.9 & 3.53 \\
\midrule
i.i.d.\ sample, $M=128$ & 0.536 & 12.9 & 9.71 & 0.599 & 9.64 & 8.78 & 1.36 & 19.2 & 15.4 & 2.97 & 35.9 & 28.0 \\
i.i.d.\ sample, $4096$ & -- & 0.298 & 1.63 & -- & 0.289 & 1.53 & -- & 0.584 & 2.69 & -- & 1.36 & 5.62 \\
\bottomrule
\end{tabular}

%% file: tables/tabA_morlrg_compact.tex
\emph{Attack probability} $p=0.1$ & & & & & & & & & & & \\*
SVDRL, hybrid & 32 & 0.093 & 0.028 & 0 & 0.019 & 0.013 & 0.030 & 0 & 1.0 & 2.8 & 3.1 \\
SVDRL, hybrid, $r>6$ only & 32 & 0.52 & 0.41 & 0 & 0.15 & 0.12 & 0.48 & 0 & 68 & 73 & 23 \\
SVDRL, hybrid, superseded${}^\dagger$ & 32 & 0.093 & 0.028 & 0 & 0.019 & 0.013 & 0.030 & 0 & 1.0 & 2.8 & 3.1 \\
SVDRL, particles & 32 & 0.093 & 0.028 & 0 & 0.011 & 0.0063 & 0.030 & 0 & 0.012 & 1.1 & 3.1 \\
SVDRL, sampler & 32 & 1.0 & 1.0 & 0 & 0.28 & 0.18 & 0.80 & 0 & 103 & 128 & 36 \\
\addlinespace[3pt]
SVDRL, hybrid & 128 & 0.028 & 0.0083 & 0 & 0.014 & 0.0098 & 0 & 0 & 0.97 & 2.4 & 1.1 \\
SVDRL, hybrid, $r>6$ only & 128 & 0.030 & 0.0094 & 0 & 0.015 & 0.011 & 0 & 0 & 1.1 & 2.8 & 1.1 \\
SVDRL, hybrid, superseded${}^\dagger$ & 128 & 0.028 & 0.0083 & 0 & 0.013 & 0.0090 & 0 & 0 & 0.88 & 2.4 & 1.1 \\
SVDRL, particles & 128 & 0.028 & 0.0083 & 0 & 0.0035 & 0.0018 & 0 & 0 & 0.0078 & 0.82 & 1.0 \\
SVDRL, sampler & 128 & 1.0 & 1.0 & 0 & 0.28 & 0.18 & 0.80 & 0 & 102 & 127 & 36 \\
\addlinespace[3pt]
SVDRL, rare event, sampler & 32 & 1.0 & 0.99 & 0 & 0.30 & 0.18 & 0.80 & 0 & 106 & 120 & 33 \\
SVDRL, rare event, EMA sampler & 32 & 0.98 & 0.97 & 0 & 0.29 & 0.18 & 0.80 & 0 & 105 & 127 & 36 \\
SVDRL, rare event, sampler & 128 & 0.96 & 0.96 & 0 & 0.30 & 0.18 & 0.80 & 0 & 104 & 121 & 33 \\
SVDRL, rare event, EMA sampler & 128 & 0.95 & 0.95 & 0 & 0.28 & 0.18 & 0.80 & 0 & 104 & 127 & 36 \\
\addlinespace[3pt]
MMD-RBF & 32 & 2.5 & 0.78 & 0 & 0.45 & 0.24 & 0.66 & 0 & 94 & 100 & 24 \\
MMD-RBF & 128 & 0.59 & 0.44 & 0 & 0.28 & 0.18 & 0.46 & 0 & 73 & 52 & 17 \\
MD3QN & 32 & 1.5 & 0.60 & 0 & 0.34 & 0.20 & 0.61 & 0 & 56 & 73 & 17 \\
MD3QN & 128 & 0.43 & 0.36 & 0 & 0.22 & 0.14 & 0.35 & 0 & 46 & 46 & 11 \\
Energy MMD & 32 & 0.43 & 0.12 & 0 & 0.090 & 0.060 & 0.30 & 0 & 6.7 & 9.0 & 11 \\
Energy MMD & 128 & 0.18 & 0.076 & 0 & 0.065 & 0.055 & 0.14 & 0 & 7.1 & 4.8 & 4.8 \\
Sliced Cram\'er & 32 & 0.43 & 0.16 & 0 & 0.099 & 0.067 & 0.34 & 0 & 9.6 & 11 & 10 \\
Sliced Cram\'er & 128 & 0.19 & 0.086 & 0 & 0.072 & 0.064 & 0.16 & 0 & 7.9 & 4.9 & 4.7 \\
Sliced $\W2$ & 32 & 0.44 & 0.11 & 0 & 0.053 & 0.032 & 0.28 & 0 & 0.034 & 5.7 & 11 \\
Sliced $\W2$ & 128 & 0.19 & 0.082 & 0 & 0.030 & 0.019 & 0.14 & 0 & 0.014 & 1.4 & 5.3 \\
Sinkhorn & 32 & 0.97 & 0.14 & 0 & 0.12 & 0.077 & 0.34 & 0 & 9.8 & 22 & 9.8 \\
\addlinespace[3pt]
QR, rare event & 1000 & 0.59 & 0.61 & 0.27 & 1.1 & 1.5 & 0.58 & 0.52 & -- & -- & -- \\
QR, rare event, $156\,800$ steps & 1000 & 0.75 & 0.59 & 0.35 & 1.7 & 2.0 & 0.65 & 0.96 & -- & -- & -- \\
QR, oracle $\lambda$, lr $0.03$ & 1000 & 0.19 & 0.19 & 0.029 & 0.78 & 1.2 & 0.20 & 0.57 & -- & -- & -- \\
\midrule
\emph{Attack probability} $p=0.02$ & & & & & & & & & & & \\*
SVDRL, hybrid & 32 & 1.8 & 0.21 & 0 & 0.16 & 0.089 & 0.48 & 0.20 & 1.8 & 7.5 & 3.7 \\
SVDRL, hybrid, $r>6$ only & 32 & 1.8 & 0.21 & 0 & 0.14 & 0.074 & 0.48 & 0 & 3.3 & 12 & 4.1 \\
SVDRL, hybrid, superseded${}^\dagger$ & 32 & 1.8 & 0.21 & 0 & 0.16 & 0.089 & 0.48 & 0.20 & 1.8 & 7.5 & 3.7 \\
SVDRL, particles & 32 & 1.8 & 0.21 & 0 & 0.18 & 0.10 & 0.48 & 0.30 & 0.0057 & 0.75 & 3.6 \\
SVDRL, sampler & 32 & 9.1 & 7.8 & 0 & 0.50 & 0.21 & 0.80 & 0 & 21 & 80 & 13 \\
\addlinespace[3pt]
SVDRL, hybrid & 128 & 0.15 & 0.055 & 0 & 0.022 & 0.015 & 0.13 & 0.058 & 0.97 & 2.3 & 0.91 \\
SVDRL, hybrid, $r>6$ only & 128 & 6.8 & 2.8 & 0 & 0.41 & 0.16 & 0.80 & 0 & 17 & 76 & 13 \\
SVDRL, hybrid, superseded${}^\dagger$ & 128 & 0.15 & 0.055 & 0 & 0.022 & 0.015 & 0.13 & 0.058 & 0.93 & 2.2 & 0.91 \\
SVDRL, particles & 128 & 0.15 & 0.055 & 0 & 0.015 & 0.0089 & 0.13 & 0.058 & 0.0062 & 0.60 & 0.91 \\
SVDRL, sampler & 128 & 7.9 & 7.4 & 0 & 0.47 & 0.19 & 0.80 & 0 & 20 & 87 & 15 \\
\addlinespace[3pt]
SVDRL, rare event, sampler & 32 & 7.5 & 7.3 & 0 & 0.51 & 0.22 & 0.80 & 0 & 21 & 74 & 12 \\
SVDRL, rare event, EMA sampler & 32 & 7.4 & 7.3 & 0 & 0.50 & 0.21 & 0.80 & 0 & 21 & 81 & 13 \\
SVDRL, rare event, sampler & 128 & 7.1 & 7.0 & 0 & 0.48 & 0.20 & 0.80 & 0 & 21 & 78 & 13 \\
SVDRL, rare event, EMA sampler & 128 & 7.0 & 6.9 & 0 & 0.47 & 0.19 & 0.80 & 0 & 20 & 86 & 15 \\
\addlinespace[3pt]
MMD-RBF & 32 & 8.8 & 5.5 & 0 & 0.71 & 0.39 & 0.75 & 0.81 & 26 & 52 & 11 \\
MMD-RBF & 128 & 7.9 & 5.8 & 0 & 0.54 & 0.27 & 0.80 & 0.22 & 23 & 39 & 8.1 \\
MD3QN & 32 & 9.5 & 6.0 & 0 & 0.68 & 0.39 & 0.78 & 0.74 & 25 & 53 & 11 \\
MD3QN & 128 & 4.5 & 1.6 & 0 & 0.43 & 0.24 & 0.79 & 0.26 & 20 & 44 & 7.6 \\
Energy MMD & 32 & 4.0 & 0.28 & 0 & 0.45 & 0.27 & 0.57 & 0.70 & 3.6 & 6.5 & 5.4 \\
Energy MMD & 128 & 0.74 & 0.28 & 0 & 0.19 & 0.12 & 0.50 & 0.15 & 4.5 & 4.5 & 2.5 \\
Sliced Cram\'er & 32 & 4.0 & 0.34 & 0 & 0.43 & 0.25 & 0.55 & 0.70 & 4.7 & 7.6 & 5.4 \\
Sliced Cram\'er & 128 & 0.74 & 0.29 & 0 & 0.17 & 0.11 & 0.48 & 0.14 & 4.4 & 4.7 & 2.3 \\
Sliced $\W2$ & 32 & 4.0 & 0.31 & 0 & 0.38 & 0.23 & 0.57 & 0.70 & 0.0077 & 1.6 & 5.8 \\
Sliced $\W2$ & 128 & 0.45 & 0.18 & 0 & 0.060 & 0.038 & 0.41 & 0.17 & 0.0098 & 0.57 & 2.6 \\
Sinkhorn & 32 & 6.0 & 0.61 & 0 & 0.54 & 0.32 & 0.65 & 0.70 & 10 & 25 & 7.3 \\
\addlinespace[3pt]
QR, rare event & 1000 & 1.9 & 2.1 & 0.27 & 1.3 & 1.6 & 0.67 & 0.60 & -- & -- & -- \\
QR, rare event, $156\,800$ steps & 1000 & 1.9 & 2.2 & 0.35 & 1.8 & 2.3 & 0.70 & 0.97 & -- & -- & -- \\
QR, oracle $\lambda$, lr $0.03$ & 1000 & 0.54 & 0.54 & 0.029 & 1.0 & 1.4 & 0.40 & 0.67 & -- & -- & -- \\
\midrule
\emph{Attack probability} $p=0.005$ & & & & & & & & & & & \\*
SVDRL, hybrid & 32 & 6.8 & 0.51 & 0 & 0.51 & 0.29 & 0.72 & 0.67 & 5.0 & 18 & 3.0 \\
SVDRL, hybrid, $r>6$ only & 32 & 6.8 & 0.51 & 0 & 0.51 & 0.29 & 0.72 & 0.67 & 5.0 & 18 & 3.0 \\
SVDRL, hybrid, superseded${}^\dagger$ & 32 & 6.8 & 0.51 & 0 & 0.51 & 0.29 & 0.72 & 0.67 & 5.0 & 18 & 3.0 \\
SVDRL, particles & 32 & 6.8 & 0.51 & 0 & 0.57 & 0.34 & 0.72 & 0.73 & 0.0055 & 0.40 & 2.8 \\
SVDRL, sampler & 32 & 8.9 & 8.9 & 0 & 0.63 & 0.35 & 0.80 & 0.68 & 6.8 & 24 & 2.4 \\
\addlinespace[3pt]
SVDRL, hybrid & 128 & 1.6 & 0.21 & 0 & 0.18 & 0.10 & 0.48 & 0.12 & 1.7 & 7.8 & 1.2 \\
SVDRL, hybrid, $r>6$ only & 128 & 1.6 & 0.21 & 0 & 0.18 & 0.10 & 0.48 & 0.12 & 1.8 & 7.8 & 1.2 \\
SVDRL, hybrid, superseded${}^\dagger$ & 128 & 1.6 & 0.21 & 0 & 0.18 & 0.10 & 0.48 & 0.12 & 1.7 & 7.3 & 1.1 \\
SVDRL, particles & 128 & 1.6 & 0.21 & 0 & 0.19 & 0.11 & 0.48 & 0.31 & 0.0057 & 0.28 & 0.92 \\
SVDRL, sampler & 128 & 8.9 & 8.9 & 0 & 0.56 & 0.30 & 0.80 & 0.14 & 6.1 & 46 & 4.4 \\
\addlinespace[3pt]
SVDRL, rare event, sampler & 32 & 8.9 & 8.9 & 0 & 0.59 & 0.33 & 0.80 & 0.49 & 6.1 & 32 & 3.0 \\
SVDRL, rare event, EMA sampler & 32 & 8.9 & 8.9 & 0 & 0.59 & 0.33 & 0.80 & 0.49 & 6.9 & 36 & 3.3 \\
SVDRL, rare event, sampler & 128 & 8.9 & 8.9 & 0 & 0.56 & 0.30 & 0.80 & 0.050 & 5.6 & 43 & 4.2 \\
SVDRL, rare event, EMA sampler & 128 & 8.9 & 8.9 & 0 & 0.55 & 0.30 & 0.80 & 0.042 & 6.0 & 47 & 4.7 \\
\addlinespace[3pt]
MMD-RBF & 32 & 8.9 & 8.9 & 0 & 0.68 & 0.40 & 0.80 & 0.78 & 6.8 & 16 & 3.6 \\
MMD-RBF & 128 & 8.9 & 8.9 & 0 & 0.61 & 0.35 & 0.80 & 0.68 & 5.5 & 18 & 2.8 \\
MD3QN & 32 & 8.9 & 8.9 & 0 & 0.69 & 0.41 & 0.80 & 0.78 & 6.8 & 21 & 3.8 \\
MD3QN & 128 & 8.9 & 8.9 & 0 & 0.62 & 0.36 & 0.80 & 0.70 & 5.5 & 24 & 3.1 \\
Energy MMD & 32 & 7.1 & 0.51 & 0 & 0.60 & 0.36 & 0.73 & 0.78 & 1.5 & 3.5 & 3.1 \\
Energy MMD & 128 & 5.3 & 0.75 & 0 & 0.47 & 0.29 & 0.67 & 0.63 & 3.4 & 7.8 & 1.9 \\
Sliced Cram\'er & 32 & 7.0 & 0.72 & 0 & 0.60 & 0.36 & 0.73 & 0.78 & 1.8 & 4.6 & 3.1 \\
Sliced Cram\'er & 128 & 5.7 & 0.61 & 0 & 0.47 & 0.28 & 0.68 & 0.63 & 3.4 & 7.7 & 1.9 \\
Sliced $\W2$ & 32 & 6.9 & 0.55 & 0 & 0.57 & 0.34 & 0.72 & 0.78 & 0.0064 & 0.59 & 3.0 \\
Sliced $\W2$ & 128 & 4.2 & 0.29 & 0 & 0.35 & 0.21 & 0.60 & 0.59 & 0.0059 & 0.31 & 1.5 \\
Sinkhorn & 32 & 8.9 & 8.9 & 0 & 0.68 & 0.41 & 0.80 & 0.78 & 6.2 & 16 & 3.7 \\
\addlinespace[3pt]
QR, rare event & 1000 & 2.9 & 3.1 & 0.27 & 1.4 & 1.7 & 0.72 & 0.65 & -- & -- & -- \\
QR, rare event, $156\,800$ steps & 1000 & 3.2 & 3.5 & 0.35 & 2.1 & 2.4 & 0.68 & 0.82 & -- & -- & -- \\
QR, oracle $\lambda$, lr $0.03$ & 1000 & 1.6 & 1.6 & 0.029 & 1.2 & 1.5 & 0.60 & 0.79 & -- & -- & -- \\

%% file: tables/tabA_mujoco_returns.tex
\begin{table*}[t]
\centering\small
\caption{\emph{Returns per head.} MuJoCo returns at 1M environment steps: IQM over seeds of each run's mean eval return over 800k--1M (21 evaluations, as in \cref{tab:mujoco}) and of its last 5 evaluations, with 95\,\% bootstrap CIs over seeds. The SVDRL configuration per environment is declared (HalfCheetah: the default configuration; Walker2d, Hopper, Ant: target mean, the variant that takes the moment-matching mean from the bootstrap targets; a per-environment choice between two declared variants); the rows marked $^\ast$ are the other SVDRL variant run on that environment (HalfCheetah: log-variance moment head). $^\dagger$Includes runs stopped by the divergence guard, counted with the evaluations they logged: Ant SVDRL (target mean), one seed (16 window evaluations).}
\label{tab:mujoco_returns}
\begin{tabular}{llrr}
\toprule
Env & Head & 800k--1M IQM [95\,\% CI] ($\uparrow$) & last-5 IQM [95\,\% CI] ($\uparrow$) \\
\midrule
HalfCheetah & SVDRL (default) & $8720\sd{[7565,\,9889]}$ & $8821\sd{[7158,\,10035]}$ \\
 & SAC & $10014\sd{[9446,\,10512]}$ & $9988\sd{[8673,\,10742]}$ \\
 & IQN & $9219\sd{[8338,\,9879]}$ & $10227\sd{[7198,\,11087]}$ \\
 & MD3QN & $8986\sd{[8449,\,9877]}$ & $8676\sd{[5406,\,11376]}$ \\
 & energy-MMD & $10098\sd{[8885,\,11186]}$ & $10571\sd{[8256,\,11855]}$ \\
 & SVDRL, log-variance head$^\ast$ & $9193\sd{[7393,\,10571]}$ & $8640\sd{[4898,\,11551]}$ \\
\midrule
Walker2d & SVDRL (target mean) & $5514\sd{[5025,\,5780]}$ & $5482\sd{[4883,\,5843]}$ \\
 & SAC & $5372\sd{[4939,\,5744]}$ & $5543\sd{[4925,\,5956]}$ \\
 & IQN & $5083\sd{[4466,\,5496]}$ & $4928\sd{[4289,\,5501]}$ \\
 & MD3QN & $5661\sd{[5129,\,5907]}$ & $5695\sd{[5235,\,5984]}$ \\
 & energy-MMD & $5609\sd{[5361,\,5805]}$ & $5713\sd{[5434,\,5988]}$ \\
 & SVDRL (default)$^\ast$ & $3774\sd{[3560,\,4513]}$ & $3422\sd{[2249,\,4629]}$ \\
\midrule
Hopper & SVDRL (target mean) & $2187\sd{[2047,\,2294]}$ & $2061\sd{[1849,\,2334]}$ \\
 & SAC & $2187\sd{[2083,\,2319]}$ & $1935\sd{[1738,\,2483]}$ \\
 & IQN & $2019\sd{[1841,\,2221]}$ & $2089\sd{[1753,\,2416]}$ \\
 & MD3QN & $1822\sd{[1743,\,1889]}$ & $1943\sd{[1524,\,2425]}$ \\
 & energy-MMD & $2180\sd{[2047,\,2255]}$ & $2214\sd{[1685,\,2529]}$ \\
\midrule
Ant & SVDRL (target mean) & $2429\sd{[2059,\,2654]}$ & $2621\sd{[2124,\,2952]}$ \\
 & SAC & $2426\sd{[2214,\,2759]}$ & $2552\sd{[2305,\,2995]}$ \\
 & IQN & $2411\sd{[2095,\,3029]}$ & $2523\sd{[2242,\,3074]}$ \\
 & MD3QN & $1233\sd{[895,\,2343]}$ & $1316\sd{[917,\,2545]}$ \\
 & energy-MMD & $2928\sd{[2628,\,3598]}$ & $3088\sd{[2617,\,3910]}$ \\
 & SVDRL (default)$^\ast$ & $1052\sd{[956,\,1165]}$ & $1176\sd{[887,\,1602]}$ \\
\bottomrule
\end{tabular}
\end{table*}

%% file: tables/tabA_mc_compact.tex
\begin{tabular}{lcccccc}
\toprule
Head & shape ($\downarrow$) & std ratio & mean error & full $\W2$ ($\downarrow$) & 80\,\% cov. & centered \\
\midrule
\multicolumn{7}{l}{\emph{HalfCheetah}} \\
SVDRL (default) & $0.43$ & 1.36 & -88 & $4.9$ & 0.01 & 0.92 \\[-1pt]
 & $\sd{[0.28,\,0.56]}$ &  &  & $\sd{[3.7,\,7.0]}$ &  &  \\[1pt]
IQN & $1.05$ & 2.01 & -74 & $4.6$ & 0.05 & 0.97 \\[-1pt]
 & $\sd{[0.82,\,1.41]}$ &  &  & $\sd{[3.3,\,5.7]}$ &  &  \\[1pt]
MD3QN & $0.69$ & 1.64 & -50 & $6.8$ & 0.00 & 0.96 \\[-1pt]
 & $\sd{[0.50,\,0.92]}$ &  &  & $\sd{[4.9,\,8.9]}$ &  &  \\[1pt]
energy-MMD & $0.52$ & 1.46 & -76 & $4.1$ & 0.02 & 0.94 \\[-1pt]
 & $\sd{[0.39,\,0.71]}$ &  &  & $\sd{[3.3,\,4.9]}$ &  &  \\[1pt]
SVDRL, log-variance head$^\ast$ & $0.51$ & 1.46 & -10 & $4.2$ & 0.01 & 0.94 \\[-1pt]
 & $\sd{[0.44,\,0.91]}$ &  &  & $\sd{[3.7,\,6.0]}$ &  &  \\[1pt]
\quad its sampler$^\ast$ & $0.54$ & 1.49 & -11 & $4.2$ & 0.01 & 0.95 \\[-1pt]
 & $\sd{[0.43,\,0.98]}$ &  &  & $\sd{[3.7,\,6.2]}$ &  &  \\[1pt]
SAC & -- & -- & -55 & -- & -- & -- \\
\midrule
\multicolumn{7}{l}{\emph{Walker2d}} \\
SVDRL (target mean) & $0.92$ & 1.34 & -68 & $4.3$ & 0.02 & 0.95 \\[-1pt]
 & $\sd{[0.61,\,1.31]}$ &  &  & $\sd{[1.9,\,9.9]}$ &  &  \\[1pt]
IQN & $3.99$ & 4.85 & -44 & $6.3$ & 0.46 & 1.00 \\[-1pt]
 & $\sd{[2.32,\,6.13]}$ &  &  & $\sd{[3.6,\,8.8]}$ &  &  \\[1pt]
MD3QN & $1.19$ & 1.94 & -18 & $3.4$ & 0.10 & 0.99 \\[-1pt]
 & $\sd{[0.90,\,2.75]}$ &  &  & $\sd{[1.9,\,7.9]}$ &  &  \\[1pt]
energy-MMD & $2.08$ & 2.84 & -21 & $3.9$ & 0.51 & 0.98 \\[-1pt]
 & $\sd{[0.82,\,7.51]}$ &  &  & $\sd{[1.6,\,10.2]}$ &  &  \\[1pt]
SVDRL (default)$^\ast$ & $0.70$ & 0.67 & 100 & $3.7$ & 0.02 & 0.68 \\[-1pt]
 & $\sd{[0.42,\,1.10]}$ &  &  & $\sd{[1.8,\,17.1]}$ &  &  \\[1pt]
SAC & -- & -- & -14 & -- & -- & -- \\
\midrule
\multicolumn{7}{l}{\emph{Hopper}} \\
SVDRL (target mean) & $2.25$ & 2.95 & -8 & $5.6$ & 0.20 & 0.99 \\[-1pt]
 & $\sd{[0.87,\,4.56]}$ &  &  & $\sd{[3.5,\,18.2]}$ &  &  \\[1pt]
IQN & $10.4$ & 11.18 & 2 & $14.2$ & 0.24 & 1.00 \\[-1pt]
 & $\sd{[6.7,\,20.6]}$ &  &  & $\sd{[8.9,\,26.2]}$ &  &  \\[1pt]
MD3QN & $6.10$ & 6.69 & -21 & $15.5$ & 0.00 & 0.99 \\[-1pt]
 & $\sd{[1.69,\,14.74]}$ &  &  & $\sd{[4.4,\,33.4]}$ &  &  \\[1pt]
energy-MMD & $21.3$ & 21.97 & -28 & $26.0$ & 0.22 & 1.00 \\[-1pt]
 & $\sd{[10.7,\,44.5]}$ &  &  & $\sd{[11.9,\,62.5]}$ &  &  \\[1pt]
SAC & -- & -- & 2 & -- & -- & -- \\
\midrule
\multicolumn{7}{l}{\emph{Ant}} \\
SVDRL (target mean) & $0.59$ & 0.74 & -27 & $3.1$ & 0.03 & 0.80 \\[-1pt]
 & $\sd{[0.57,\,0.63]}$ &  &  & $\sd{[2.7,\,3.4]}$ &  &  \\[1pt]
IQN & $0.73$ & 1.24 & -37 & $1.7$ & 0.27 & 0.85 \\[-1pt]
 & $\sd{[0.42,\,1.19]}$ &  &  & $\sd{[1.2,\,2.2]}$ &  &  \\[1pt]
MD3QN & $0.72$ & 1.52 & -22 & $3.1$ & 0.16 & 0.93 \\[-1pt]
 & $\sd{[0.56,\,2.95]}$ &  &  & $\sd{[1.8,\,4.3]}$ &  &  \\[1pt]
energy-MMD & $0.49$ & 0.75 & -10 & $2.0$ & 0.08 & 0.82 \\[-1pt]
 & $\sd{[0.45,\,0.56]}$ &  &  & $\sd{[1.5,\,2.4]}$ &  &  \\[1pt]
SVDRL (default)$^\ast$ & $0.57$ & 0.51 & -23 & $1.7$ & 0.11 & 0.53 \\[-1pt]
 & $\sd{[0.48,\,0.63]}$ &  &  & $\sd{[1.1,\,2.7]}$ &  &  \\[1pt]
SAC & -- & -- & -14 & -- & -- & -- \\
\bottomrule
\end{tabular}

%% file: tables/tabA_theoryarms.tex
\begin{table*}[t]
\centering
\small
\caption{\emph{Theory checks.} Final checkpoint, mean $\pm$ 1.96 s.e. over seeds (raw units unless stated). (a, b) bootstrap source and particle set (\cref{thm:e2e}); (c) multi-level vs single-level networks (\cref{rem:laddernecessity}), state--action pair $s_0$; (d) \cref{thm:gen}, normalized units except $W_1$; (e) multi-level risk $\calE_{\mathrm{ml}}$ (normalized, summed over evaluated state--action pairs) under DSM, consistency and both; $n$: seeds with a finite final checkpoint.}
\label{tab:theoryarms}
\adjustbox{max width=\textwidth}{\begin{tabular}{llrrrr}
\toprule
(a, b) & bootstrap / particles & sampler $\bar W_1$ ($\downarrow$) & particles $\bar W_1$ ($\downarrow$) & $\veps_{\mathrm{St}}$ ($s_0$) & sampler mode mass ($s_0$) \\
\midrule
MDP1 & particle (Stein particles) & $0.132\sd{\pm 0.029}$ & $0.247\sd{\pm 0.032}$ & $0.214\sd{\pm 0.011}$ & -- \\
MDP1 & sampler & $0.454\sd{\pm 0.165}$ & $0.558\sd{\pm 0.244}$ & $0.289\sd{\pm 0.071}$ & -- \\
MDP1 & sampler, moment-matched & $0.139\sd{\pm 0.030}$ & $0.209\sd{\pm 0.042}$ & $0.173\sd{\pm 0.031}$ & -- \\
MDP1 & sampler-seeded & $0.372\sd{\pm 0.107}$ & $0.472\sd{\pm 0.156}$ & $0.288\sd{\pm 0.077}$ & -- \\
MDP1 & sampler-seeded, moment-matched & $0.132\sd{\pm 0.025}$ & $0.208\sd{\pm 0.036}$ & $0.182\sd{\pm 0.029}$ & -- \\
MDP2 & particle (Stein particles) & $0.601\sd{\pm 0.135}$ & $1.61\sd{\pm 0.33}$ & $1.40\sd{\pm 0.31}$ & $0.0260\sd{\pm 0.0166}$ \\
MDP2 & sampler & $2.32\sd{\pm 0.22}$ & $2.36\sd{\pm 0.30}$ & $1.42\sd{\pm 0.37}$ & $0.0322\sd{\pm 0.0233}$ \\
MDP2 & sampler, moment-matched & $0.642\sd{\pm 0.103}$ & $1.63\sd{\pm 0.38}$ & $1.33\sd{\pm 0.43}$ & $0.0263\sd{\pm 0.0145}$ \\
MDP2 & sampler-seeded & $2.00\sd{\pm 0.19}$ & $2.14\sd{\pm 0.19}$ & $1.04\sd{\pm 0.16}$ & $0.0300\sd{\pm 0.0125}$ \\
MDP2 & sampler-seeded, moment-matched & $0.668\sd{\pm 0.125}$ & $1.48\sd{\pm 0.46}$ & $1.30\sd{\pm 0.43}$ & $0.0309\sd{\pm 0.0132}$ \\
MDP3 & particle (Stein particles) & $3.34\sd{\pm 0.40}$ & $6.16\sd{\pm 1.28}$ & $4.68\sd{\pm 1.39}$ & $0.0645\sd{\pm 0.0199}$ \\
MDP3 & sampler & $104\sd{\pm 8}$ & $103\sd{\pm 8}$ & $13.0\sd{\pm 2.3}$ & $0.466\sd{\pm 0.074}$ \\
MDP3 & sampler, moment-matched & $3.44\sd{\pm 0.38}$ & $6.36\sd{\pm 2.07}$ & $5.43\sd{\pm 1.95}$ & $0.0616\sd{\pm 0.0237}$ \\
MDP3 & sampler-seeded & $109\sd{\pm 5}$ & $111\sd{\pm 5}$ & $12.7\sd{\pm 1.7}$ & $0.400\sd{\pm 0.052}$ \\
MDP3 & sampler-seeded, moment-matched & $3.35\sd{\pm 0.33}$ & $6.04\sd{\pm 1.87}$ & $4.84\sd{\pm 1.97}$ & $0.0561\sd{\pm 0.0198}$ \\
\bottomrule
\end{tabular}}

\medskip
\adjustbox{max width=\textwidth}{\begin{tabular}{llrrrr}
\toprule
(c) & network & mode mass & $W_1$ & multi-level risk (norm.) & own-level $\calE_\sigma$ (unnorm.) \\
\midrule
MDP2 & main (multi-level, DSM + cons.) & $0.0260\sd{\pm 0.0166}$ & $0.601\sd{\pm 0.135}$ & $0.0944\sd{\pm 0.0229}$ & -- \\
MDP2 & multi-level, DSM only & $0.0563\sd{\pm 0.0346}$ & $0.909\sd{\pm 0.275}$ & $0.182\sd{\pm 0.050}$ & -- \\
MDP2 & single level $\sigma_{\min}$ & $0.320\sd{\pm 0.010}$ & $5.96\sd{\pm 2.30}$ & $251\sd{\pm 155}$ & $3.20\sd{\pm 0.96}$ \\
MDP2 & single level $\Delta/2$ & $0.181\sd{\pm 0.022}$ & $6.09\sd{\pm 0.10}$ & $2.56\sd{\pm 0.09}$ & $0.00342\sd{\pm 0.00143}$ \\
MDP3 & main (multi-level, DSM + cons.) & $0.0645\sd{\pm 0.0199}$ & $3.34\sd{\pm 0.40}$ & $0.364\sd{\pm 0.100}$ & -- \\
MDP3 & multi-level, DSM only & $0.0986\sd{\pm 0.0440}$ & $3.50\sd{\pm 0.76}$ & $0.438\sd{\pm 0.135}$ & -- \\
MDP3 & single level $\sigma_{\min}$ & $0.461\sd{\pm 0.059}$ & $31.7\sd{\pm 6.2}$ & $236\sd{\pm 261}$ & $18.3\sd{\pm 4.7}$ \\
MDP3 & single level $\Delta/2$ & $0.0875\sd{\pm 0.0358}$ & $16.6\sd{\pm 0.3}$ & $30.0\sd{\pm 9.9}$ & $0.0330\sd{\pm 0.0088}$ \\
\bottomrule
\end{tabular}}

\medskip
\adjustbox{max width=\textwidth}{\begin{tabular}{llrrrrrrr}
\toprule
(d) & state--action pair & $\calE_{\mathrm{ml}}$ & $\veps_{\mathrm{disc}}$ & $\mathrm{KL}_0$ & $\veps_{\mathrm{TV}}$ & TV meas. & $W_1$ bound & $W_1$ meas. \\
\midrule
MDP1 & $s_0$ & $0.0285\sd{\pm 0.0119}$ & $0.2490\sd{\pm 0.0005}$ & $0.0175\sd{\pm 0.0001}$ & $0.384\sd{\pm 0.008}$ & $0.0461\sd{\pm 0.0072}$ & $53.9\sd{\pm 1.1}$ & $0.136\sd{\pm 0.030}$ \\
MDP2 & $s_0$ & $0.0944\sd{\pm 0.0229}$ & $0.5505\sd{\pm 0.0002}$ & $0.01469\sd{\pm 0.00005}$ & $0.574\sd{\pm 0.010}$ & $0.0839\sd{\pm 0.0177}$ & $92.8\sd{\pm 1.6}$ & $0.575\sd{\pm 0.135}$ \\
MDP2 & $s_1$ & $0.00449\sd{\pm 0.00199}$ & $0.6954\sd{\pm 0.0002}$ & $1.76 \times 10^{-4}\sd{\pm 4.8 \times 10^{-7}}$ & $0.5916\sd{\pm 0.0009}$ & $0.0252\sd{\pm 0.0098}$ & $95.6\sd{\pm 0.1}$ & $0.346\sd{\pm 0.030}$ \\
MDP2 & $s_2$ & $0.0102\sd{\pm 0.0044}$ & $0.6468\sd{\pm 0.0002}$ & $3.15 \times 10^{-4}\sd{\pm 2.1 \times 10^{-6}}$ & $0.573\sd{\pm 0.002}$ & $0.0370\sd{\pm 0.0103}$ & $92.7\sd{\pm 0.3}$ & $0.306\sd{\pm 0.036}$ \\
MDP3 & $s_0$ & $0.364\sd{\pm 0.100}$ & $1.328\sd{\pm 0.002}$ & $0.0312\sd{\pm 0.0002}$ & $0.927\sd{\pm 0.027}$ & $0.199\sd{\pm 0.023}$ & $296\sd{\pm 8}$ & $3.10\sd{\pm 0.34}$ \\
MDP3 & $s_1$ & $0.0988\sd{\pm 0.0264}$ & $1.524\sd{\pm 0.001}$ & $1.38 \times 10^{-4}\sd{\pm 1.3 \times 10^{-6}}$ & $0.901\sd{\pm 0.007}$ & $0.179\sd{\pm 0.007}$ & $287\sd{\pm 2}$ & $1.76\sd{\pm 0.04}$ \\
MDP3 & $s_2$ & $0.115\sd{\pm 0.071}$ & $1.544\sd{\pm 0.001}$ & $1.02 \times 10^{-4}\sd{\pm 1.3 \times 10^{-6}}$ & $0.910\sd{\pm 0.019}$ & $0.169\sd{\pm 0.010}$ & $290\sd{\pm 6}$ & $1.87\sd{\pm 0.05}$ \\
MDP3 & $s_3$ & $0.0733\sd{\pm 0.0365}$ & $1.507\sd{\pm 0.001}$ & $1.46 \times 10^{-4}\sd{\pm 8.9 \times 10^{-7}}$ & $0.889\sd{\pm 0.010}$ & $0.200\sd{\pm 0.006}$ & $284\sd{\pm 3}$ & $1.60\sd{\pm 0.05}$ \\
\bottomrule
\end{tabular}}

\medskip
\adjustbox{max width=\textwidth}{\begin{tabular}{llrrr}
\toprule
(e) & training signal & $\calE_{\mathrm{ml}}$ & sampler $\bar W_1$ & $n$ \\
\midrule
3-cycle & DSM only & $5.08 \times 10^{-5}\sd{\pm 1.6 \times 10^{-5}}$ & $0.0420\sd{\pm 0.0002}$ & 10/10 \\
3-cycle & consistency only & $6.27 \times 10^{-4}\sd{\pm 5.2 \times 10^{-4}}$ & $0.0420\sd{\pm 0.0002}$ & 10/10 \\
3-cycle & DSM + consistency & $2.48 \times 10^{-5}\sd{\pm 1.4 \times 10^{-5}}$ & $0.0420\sd{\pm 0.0002}$ & 10/10 \\
3-cycle & consistency only, grid pairing & $4.51 \times 10^{-4}\sd{\pm 3.3 \times 10^{-4}}$ & $0.0420\sd{\pm 0.0002}$ & 10/10 \\
3-cycle & DSM + consistency, grid pairing & $1.67 \times 10^{-5}\sd{\pm 5.0 \times 10^{-6}}$ & $0.0420\sd{\pm 0.0002}$ & 10/10 \\
MDP1 & DSM only & $0.0493\sd{\pm 0.0185}$ & $0.172\sd{\pm 0.031}$ & 10/10 \\
MDP1 & DSM + consistency (deployed) & $0.0285\sd{\pm 0.0119}$ & $0.132\sd{\pm 0.029}$ & 10/10 \\
MDP1 & consistency only & $5.21 \times 10^{15}\sd{\pm 1.0 \times 10^{16}}$ & $79.0\sd{\pm 21.6}$ & 2/10 \\
MDP1 & consistency only, grid pairing & $5.51 \times 10^{16}\sd{\pm 8.8 \times 10^{16}}$ & $74.0\sd{\pm 19.2}$ & 5/10 \\
MDP1 & consistency only, grid pairing, EMA target & $0.0143\sd{\pm 0.0037}$ & $0.144\sd{\pm 0.022}$ & 10/10 \\
\bottomrule
\end{tabular}}
\end{table*}

%% file: tables/tabA_consTD_check.tex
\begin{tabular}{llrrrrr}
\toprule
MDP & variant & one-step holds & iterated holds & bound/error top & middle & bottom \\
\midrule
3-cycle & DSM only & 10/10 & 10/10 & $2.54\sd{\pm 0.24}$ & $53.7\sd{\pm 15.7}$ & $109\sd{\pm 37}$ \\
3-cycle & consistency only & 10/10 & 10/10 & $1.00010\sd{\pm 0.00005}$ & $6.47\sd{\pm 1.74}$ & $11.4\sd{\pm 3.8}$ \\
3-cycle & DSM + consistency & 10/10 & 10/10 & $1.00012\sd{\pm 0.00006}$ & $32.8\sd{\pm 6.0}$ & $74.5\sd{\pm 19.4}$ \\
3-cycle & consistency only, grid pairing & 10/10 & 10/10 & $1.00006\sd{\pm 0.00002}$ & $6.56\sd{\pm 2.13}$ & $12.1\sd{\pm 5.2}$ \\
3-cycle & DSM + consistency, grid pairing & 10/10 & 10/10 & $1.00012\sd{\pm 0.00004}$ & $35.4\sd{\pm 7.9}$ & $80.2\sd{\pm 18.0}$ \\
MDP1 & DSM only & 10/10 & 10/10 & $1.31\sd{\pm 0.18}$ & $3.02\sd{\pm 0.83}$ & $30.5\sd{\pm 11.0}$ \\
MDP1 & DSM + consistency (deployed) & 10/10 & 10/10 & $1.32\sd{\pm 0.15}$ & $4.79\sd{\pm 2.19}$ & $39.9\sd{\pm 15.2}$ \\
MDP1 & consistency only & 2/2 & 2/2 & $1.000000004\sd{\pm 0.000000006}$ & $124\sd{\pm 152}$ & $9.94\sd{\pm 0.02}$ \\
MDP1 & consistency only, grid pairing & 5/5 & 5/5 & $1.0000001\sd{\pm 0.0000001}$ & $62.7\sd{\pm 114.8}$ & $98.7\sd{\pm 80.9}$ \\
MDP1 & consistency only, grid pairing, EMA target & 10/10 & 10/10 & $3.43\sd{\pm 1.74}$ & $2.36\sd{\pm 0.30}$ & $15.8\sd{\pm 1.3}$ \\
\bottomrule
\end{tabular}

%% file: tables/tabA_re1_static_fit.tex
\begin{tabular}{lrrrr}
\toprule
seed & rare mass & median offset & sd & sd $\div$ true sd \\
\midrule
0 & 0.00667 & +1.7 & 7.42 & 12.9 \\
1 & 0.00958 & +4.7 & 8.22 & 14.2 \\
2 & 0.0110 & +8.1 & 8.48 & 14.7 \\
\midrule
truth & 0.005 & 0 & 0.578 & 1 \\
\bottomrule
\end{tabular}

%% file: tables/tabA_mz.tex
\begin{tabular}{@{}llrrrr@{}}
\toprule
variant & head & $n$ & $\hat W_2$ (Debes) ($\downarrow$) & $\hat W_2$ (200 MC) ($\downarrow$) & sliced $W_2$ (CDF) ($\downarrow$) \\
\midrule
exclusive & MMD-mq & 5 & $\mathbf{0.309}\sd{[0.30,\,0.31]}$ & $0.121\sd{[0.12,\,0.13]}$ & $0.0385\sd{[0.033,\,0.040]}$ \\
 & sliced-Cram\'er & 5 & $0.309\sd{[0.30,\,0.31]}$ & $\mathbf{0.116}\sd{[0.11,\,0.12]}$ & $\mathbf{0.0249}\sd{[0.022,\,0.031]}$ \\
 & SVDRL (log-var.\ moment head) & 5 & $0.422\sd{[0.42,\,0.44]}$ & $0.367\sd{[0.34,\,0.39]}$ & $0.224\sd{[0.19,\,0.25]}$ \\
 & sliced-W$_2$ & 5 & $0.512\sd{[0.49,\,0.54]}$ & $0.506\sd{[0.50,\,0.51]}$ & $0.352\sd{[0.35,\,0.35]}$ \\
 & SVDRL (earlier selected configuration) & 5 & $0.761\sd{[0.43,\,0.81]}$ & $0.702\sd{[0.40,\,0.71]}$ & $0.378\sd{[0.26,\,0.43]}$ \\
\midrule
identical & sliced-Cram\'er & 5 & $\mathbf{0.253}\sd{[0.25,\,0.26]}$ & $\mathbf{0.0960}\sd{[0.093,\,0.10]}$ & $\mathbf{0.0313}\sd{[0.027,\,0.040]}$ \\
 & MMD-mq & 5 & $0.263\sd{[0.25,\,0.27]}$ & $0.115\sd{[0.11,\,0.12]}$ & $0.0532\sd{[0.050,\,0.061]}$ \\
 & SVDRL (log-var.\ moment head) & 5 & $0.326\sd{[0.30,\,0.34]}$ & $0.272\sd{[0.26,\,0.28]}$ & $0.182\sd{[0.17,\,0.19]}$ \\
 & SVDRL (earlier selected configuration) & 5 & $0.331\sd{[0.30,\,0.55]}$ & $0.238\sd{[0.19,\,0.46]}$ & $0.139\sd{[0.11,\,0.26]}$ \\
 & sliced-W$_2$ & 5 & $0.387\sd{[0.38,\,0.40]}$ & $0.377\sd{[0.37,\,0.38]}$ & $0.260\sd{[0.26,\,0.27]}$ \\
\midrule
multireward & MMD-mq & 5 & $\mathbf{0.271}\sd{[0.26,\,0.31]}$ & $0.134\sd{[0.13,\,0.14]}$ & $0.0322\sd{[0.031,\,0.038]}$ \\
 & sliced-Cram\'er & 5 & $0.273\sd{[0.26,\,0.31]}$ & $\mathbf{0.129}\sd{[0.12,\,0.14]}$ & $\mathbf{0.0250}\sd{[0.024,\,0.026]}$ \\
 & SVDRL (log-var.\ moment head) & 5 & $0.398\sd{[0.32,\,0.43]}$ & $0.383\sd{[0.29,\,0.40]}$ & $0.145\sd{[0.12,\,0.16]}$ \\
 & SVDRL (earlier selected configuration) & 5 & $0.433\sd{[0.42,\,0.48]}$ & $0.412\sd{[0.40,\,0.45]}$ & $0.144\sd{[0.14,\,0.16]}$ \\
 & sliced-W$_2$ & 5 & $0.485\sd{[0.48,\,0.49]}$ & $0.487\sd{[0.49,\,0.49]}$ & $0.242\sd{[0.24,\,0.24]}$ \\
\bottomrule
\end{tabular}

%% file: tables/tab_mo1_compact.tex
\setlength{\tabcolsep}{4pt}
\begin{tabular}{lccccc}
\toprule
 & \multicolumn{2}{c}{Law ($\downarrow$)} & Mean ($\downarrow$) & \multicolumn{2}{c}{Centered law ($\downarrow$)} \\
\cmidrule(lr){2-3}\cmidrule(lr){4-4}\cmidrule(lr){5-6}
Head & energy & sliced $\W2^2$ & abs.\ error & energy & sliced $\W2^2$ \\
\midrule
\multicolumn{6}{l}{\emph{mo-halfcheetah, $5\cdot10^5$ steps}} \\
\quad SVDRL selected, particles$^{\dagger}$ & $0.1\sd{[0.098,\,0.14]}$ & $0.17\sd{[0.16,\,0.22]}$ & $0.17\sd{[0.15,\,0.18]}$ & $0.065\sd{[0.055,\,0.087]}$ & $0.14\sd{[0.12,\,0.18]}$ \\
\quad MD3QN & $\mathbf{0.038}\sd{[0.037,\,0.039]}$ & $\mathbf{0.062}\sd{[0.061,\,0.063]}$ & $\mathbf{0.15}\sd{[0.15,\,0.16]}$ & $\mathbf{0.0073}\sd{[0.0066,\,0.0073]}$ & $\mathbf{0.027}\sd{[0.026,\,0.028]}$ \\
\quad Energy MMD & $0.05\sd{[0.047,\,0.053]}$ & $0.09\sd{[0.086,\,0.096]}$ & $0.17\sd{[0.16,\,0.18]}$ & $0.014\sd{[0.014,\,0.015]}$ & $0.048\sd{[0.046,\,0.05]}$ \\
\quad IQN & -- & -- & $\mathbf{0.15}\sd{[0.15,\,0.16]}$ & -- & -- \\
\multicolumn{6}{l}{\quad\emph{uncentered accumulator, for reference}} \\
\quad SVDRL initial, particles & $0.33\sd{[0.16,\,0.8]}$ & $0.76\sd{[0.27,\,2.2]}$ & $0.29\sd{[0.21,\,0.66]}$ & $0.11\sd{[0.09,\,0.3]}$ & $0.37\sd{[0.21,\,1.5]}$ \\
\quad SVDRL initial, sampler & $13\sd{[0.35,\,110]}$ & $280\sd{[1.4,\,4\,100]}$ & $4.1\sd{[0.29,\,33]}$ & $2.1\sd{[0.25,\,21]}$ & $84\sd{[1.2,\,1\,200]}$ \\
\quad MD3QN & $0.038\sd{[0.034,\,0.04]}$ & $0.064\sd{[0.058,\,0.065]}$ & $0.15\sd{[0.14,\,0.16]}$ & $0.0077\sd{[0.0067,\,0.0091]}$ & $0.028\sd{[0.027,\,0.03]}$ \\
\quad Energy MMD & $0.081\sd{[0.053,\,1.2]}$ & $0.15\sd{[0.12,\,26]}$ & $0.16\sd{[0.15,\,0.67]}$ & $0.033\sd{[0.02,\,1.2]}$ & $0.11\sd{[0.084,\,25]}$ \\
\quad IQN & -- & -- & $0.15\sd{[0.15,\,0.16]}$ & -- & -- \\
\midrule
\multicolumn{6}{l}{\emph{mo-halfcheetah, $10^6$ steps}} \\
\quad SVDRL selected, particles$^{\dagger}$ & $0.094\sd{[0.063,\,0.14]}$ & $0.22\sd{[0.15,\,0.28]}$ & $\mathbf{0.2}\sd{[0.18,\,0.27]}$ & $0.05\sd{[0.027,\,0.069]}$ & $0.14\sd{[0.097,\,0.17]}$ \\
\quad MD3QN & $\mathbf{0.076}\sd{[0.063,\,0.093]}$ & $\mathbf{0.17}\sd{[0.15,\,0.19]}$ & $\mathbf{0.22}\sd{[0.19,\,0.25]}$ & $\mathbf{0.025}\sd{[0.022,\,0.027]}$ & $\mathbf{0.093}\sd{[0.083,\,0.095]}$ \\
\quad Energy MMD & $0.14\sd{[0.12,\,0.19]}$ & $0.29\sd{[0.26,\,0.34]}$ & $0.29\sd{[0.28,\,0.36]}$ & $0.048\sd{[0.043,\,0.052]}$ & $0.15\sd{[0.14,\,0.16]}$ \\
\quad IQN & -- & -- & $\mathbf{0.24}\sd{[0.23,\,0.28]}$ & -- & -- \\
\multicolumn{6}{l}{\quad\emph{uncentered accumulator, for reference}} \\
\quad SVDRL initial, particles & $0.68\sd{[0.26,\,0.75]}$ & $1.6\sd{[0.68,\,4.4]}$ & $0.55\sd{[0.33,\,0.67]}$ & $0.23\sd{[0.084,\,0.47]}$ & $0.91\sd{[0.5,\,3.9]}$ \\
\quad SVDRL initial, sampler & $15\sd{[0.64,\,260]}$ & $1\,300\sd{[5.6,\,17\,000]}$ & $6.6\sd{[0.43,\,85]}$ & $11\sd{[0.38,\,52]}$ & $840\sd{[4,\,3\,900]}$ \\
\quad MD3QN & $0.087\sd{[0.073,\,0.25]}$ & $0.18\sd{[0.16,\,0.42]}$ & $0.19\sd{[0.19,\,0.34]}$ & $0.028\sd{[0.026,\,0.033]}$ & $0.098\sd{[0.093,\,0.11]}$ \\
\quad Energy MMD & $0.61\sd{[0.23,\,1.6]}$ & $21\sd{[1.5,\,54]}$ & $1\sd{[0.29,\,1.7]}$ & $1.2\sd{[0.18,\,1.8]}$ & $16\sd{[1.4,\,47]}$ \\
\quad IQN & -- & -- & $0.24\sd{[0.19,\,0.72]}$ & -- & -- \\
\midrule
\multicolumn{6}{l}{\emph{mo-hopper, $5\cdot10^5$ steps}} \\
\quad SVDRL selected, particles & $1.1\sd{[0.97,\,1.1]}$ & $1.3\sd{[1.1,\,1.3]}$ & $0.51\sd{[0.49,\,0.52]}$ & $0.23\sd{[0.2,\,0.23]}$ & $0.3\sd{[0.28,\,0.32]}$ \\
\quad SVDRL selected, sampler & $1.1\sd{[0.98,\,1.1]}$ & $1.2\sd{[1.2,\,1.3]}$ & $0.51\sd{[0.5,\,0.54]}$ & $0.22\sd{[0.21,\,0.23]}$ & $0.3\sd{[0.3,\,0.33]}$ \\
\quad SVDRL initial, particles & $0.92\sd{[0.88,\,0.97]}$ & $0.97\sd{[0.93,\,1]}$ & $0.45\sd{[0.44,\,0.46]}$ & $0.15\sd{[0.14,\,0.16]}$ & $0.2\sd{[0.18,\,0.21]}$ \\
\quad SVDRL initial, sampler & $1.4\sd{[1,\,1.8]}$ & $3.6\sd{[2.9,\,18]}$ & $0.73\sd{[0.55,\,0.85]}$ & $0.4\sd{[0.26,\,0.48]}$ & $2.5\sd{[1,\,17]}$ \\
\quad MD3QN & $0.86\sd{[0.68,\,1.6]}$ & $1.6\sd{[0.81,\,9.2]}$ & $0.56\sd{[0.42,\,1.2]}$ & $0.37\sd{[0.21,\,1.9]}$ & $0.74\sd{[0.26,\,6]}$ \\
\quad Energy MMD & $\mathbf{0.74}\sd{[0.71,\,0.79]}$ & $\mathbf{0.65}\sd{[0.6,\,0.68]}$ & $\mathbf{0.38}\sd{[0.37,\,0.41]}$ & $\mathbf{0.07}\sd{[0.065,\,0.073]}$ & $\mathbf{0.064}\sd{[0.063,\,0.066]}$ \\
\quad IQN & -- & -- & $0.48\sd{[0.46,\,0.52]}$ & -- & -- \\
\midrule
\multicolumn{6}{l}{\emph{mo-hopper, $10^6$ steps}} \\
\quad SVDRL selected, particles$^{\dagger}$ & $2.9\sd{[2.7,\,3.1]}$ & $8.4\sd{[6.5,\,8.9]}$ & $\mathbf{1.5}\sd{[1.5,\,1.6]}$ & $0.76\sd{[0.36,\,0.88]}$ & $3\sd{[1.7,\,3.7]}$ \\
\quad SVDRL initial, particles & $2.8\sd{[2.7,\,3.1]}$ & $\mathbf{6.8}\sd{[6.1,\,7.2]}$ & $\mathbf{1.4}\sd{[1.4,\,1.5]}$ & $\mathbf{0.26}\sd{[0.24,\,0.44]}$ & $\mathbf{1.3}\sd{[1.3,\,1.8]}$ \\
\quad SVDRL initial, sampler & $2.7\sd{[2.7,\,5.2]}$ & $12\sd{[7.1,\,99]}$ & $\mathbf{1.4}\sd{[1.4,\,2.4]}$ & $\mathbf{0.41}\sd{[0.27,\,1.9]}$ & $7.3\sd{[1.8,\,63]}$ \\
\quad MD3QN & $6\sd{[4.1,\,7.4]}$ & $110\sd{[69,\,120]}$ & $2.6\sd{[2,\,3.3]}$ & $3\sd{[2.7,\,3.6]}$ & $56\sd{[53,\,70]}$ \\
\quad Energy MMD & $\mathbf{2.5}\sd{[2.3,\,2.6]}$ & $8.1\sd{[6.8,\,11]}$ & $\mathbf{1.4}\sd{[1.3,\,1.5]}$ & $\mathbf{0.41}\sd{[0.36,\,0.42]}$ & $3.6\sd{[1.8,\,5.5]}$ \\
\quad IQN & -- & -- & $\mathbf{1.5}\sd{[1.5,\,1.6]}$ & -- & -- \\
\bottomrule
\end{tabular}